\documentclass[journal,twoside]{cls/IEEEtran}

\usepackage[table,usenames,dvipsnames]{xcolor}
\usepackage{cite}
\usepackage{amsmath,amssymb,amsfonts,amsthm,amscd,dsfont, bbm}
\usepackage{accents} 
\usepackage[ruled]{algorithm}
\usepackage{listings}
\usepackage[noend]{algpseudocode}
\usepackage{graphicx,tabularx,adjustbox, multirow}
\usepackage{mathtools,nccmath}
\usepackage{subcaption}
\usepackage{paralist}
\usepackage[inline]{enumitem}

\let\appendices\relax
\usepackage[titletoc,title]{appendix}
\usepackage{balance}
\usepackage[breaklinks=true, colorlinks, bookmarks=true]{hyperref}

\newcommand{\calC}{{\cal C}}

\newcommand{\calE}{{\cal E}}
\newcommand{\calF}{{\cal F}}
\newcommand{\calG}{{\cal G}}

\newcommand{\calL}{{\cal L}}
\newcommand{\calM}{{\cal M}}
\newcommand{\calN}{{\cal N}}

\newcommand{\calP}{{\cal P}}

\newcommand{\calS}{{\cal S}}
\newcommand{\calT}{{\cal T}}

\newcommand{\calV}{{\cal V}}
\newcommand{\calW}{{\cal W}}
\newcommand{\calX}{{\cal X}}

\newcommand{\bfc}{\mathbf{c}}

\newcommand{\bfe}{\mathbf{e}}

\newcommand{\bfg}{\mathbf{g}}
\newcommand{\bfh}{\mathbf{h}}

\newcommand{\bfm}{\mathbf{m}}

\newcommand{\bfp}{\mathbf{p}}
\newcommand{\bfq}{\mathbf{q}}

\newcommand{\bfs}{\mathbf{s}}

\newcommand{\bfx}{\mathbf{x}}
\newcommand{\bfy}{\mathbf{y}}
\newcommand{\bfz}{\mathbf{z}}

\newcommand{\bfbeta}{\boldsymbol{\beta}}
\newcommand{\bfgamma}{\boldsymbol{\gamma}}

\newcommand{\bftheta}{\boldsymbol{\theta}}

\newcommand{\bfrho}{\boldsymbol{\rho}}

\newcommand{\bfxi}{\boldsymbol{\xi}}

\newcommand{\bfE}{\mathbf{E}}

\newcommand{\bfI}{\mathbf{I}}

\newcommand{\bfQ}{\mathbf{Q}}
\newcommand{\bfR}{\mathbf{R}}

\newcommand{\bfU}{\mathbf{U}}
\newcommand{\bfV}{\mathbf{V}}

\newcommand{\bfX}{\mathbf{X}}

\newcommand{\bbE}{\mathbb{E}}

\newcommand{\bbR}{\mathbb{R}}
\newcommand{\bbS}{\mathbb{S}}

\newcommand{\bbZ}{\mathbb{Z}}

\newcommand{\crl}[1]{\left\{#1\right\}}

\newcommand{\SE}{\textit{SE}(3)}
\newcommand{\se}{\mathfrak{se}(3)}
\DeclareMathOperator{\grad}{grad}

\algnewcommand{\LineComment}[1]{\State \(\triangleright\) #1}
\algrenewcommand\algorithmicindent{0.75em}%

\DeclareMathOperator{\Exp}{Exp}

\newtheorem{theorem}{Theorem}

\newtheorem{lemma}{Lemma}

\theoremstyle{definition}
\newtheorem{definition}{Definition}

\newtheorem{assumption}{Assumption}
\newtheorem{step}{Step}
\newtheorem{problem}{Problem}
\theoremstyle{remark}

\newtheorem*{example}{Example}

\title{\LARGE \bf Riemannian Optimization for Active Mapping with Robot Teams%
}

\author{Arash~Asgharivaskasi$^{1}$,~\IEEEmembership{Student Member,~IEEE,} Fritz Girke$^{2}$, Nikolay~Atanasov$^{1}$,~\IEEEmembership{Senior Member,~IEEE}
\thanks{The authors gratefully acknowledge support from NSF FRR CAREER 2045945, ONR N00014-23-1-2353, and ARL DCIST W911NF-17-2-0181.}%
\thanks{$^{1}$A. Asgharivaskasi and N. Atanasov are with the Department of Electrical and Computer Engineering, University of California San Diego, La Jolla, CA 92093, USA (e-mails: {\tt\small \{aasghari,natanasov\}@ucsd.edu}).}
\thanks{$^{2}$F. Girke is with the School of Computation, Information and Technology, Technical University of Munich, Munich 80333, Germany (e-mail: {\tt\small fritz.girke@tum.de}).}}

\begin{document}

\markboth{IEEE Transactions on Robotics,~Vol.~X, No.~X, Month~20XX}%
{Asgharivaskasi, Girke, and Atanasov: Riemannian Optimization for Active Mapping with Robot Teams}
%

\maketitle


\begin{abstract}
Autonomous exploration of unknown environments using a team of mobile robots demands distributed perception and planning strategies to enable efficient and scalable performance. Ideally, each robot should update its map and plan its motion not only relying on its own observations, but also considering the observations of its peers. Centralized solutions to multi-robot coordination are susceptible to central node failure and require a sophisticated communication infrastructure for reliable operation. Current decentralized active mapping methods consider simplistic robot models with linear-Gaussian observations and Euclidean robot states. In this work, we present a distributed multi-robot mapping and planning method, called Riemannian Optimization for Active Mapping (ROAM). We formulate an optimization problem over a graph with node variables belonging to a Riemannian manifold and a consensus constraint requiring feasible solutions to agree on the node variables. We develop a distributed Riemannian optimization algorithm that relies only on one-hop communication to solve the problem with consensus and optimality guarantees. We show that multi-robot active mapping can be achieved via two applications of our distributed Riemannian optimization over different manifolds: distributed estimation of a 3-D semantic map and distributed planning of $\SE$ trajectories that minimize map uncertainty. We demonstrate the performance of ROAM in simulation and real-world experiments using a team of robots with RGB-D cameras.
\end{abstract}


\IEEEpeerreviewmaketitle

\section{Introduction}
\label{sec:intro}

\begin{figure}[t]
    \begin{subfigure}[t]{\linewidth}
    \centering
    \includegraphics[width=\linewidth]{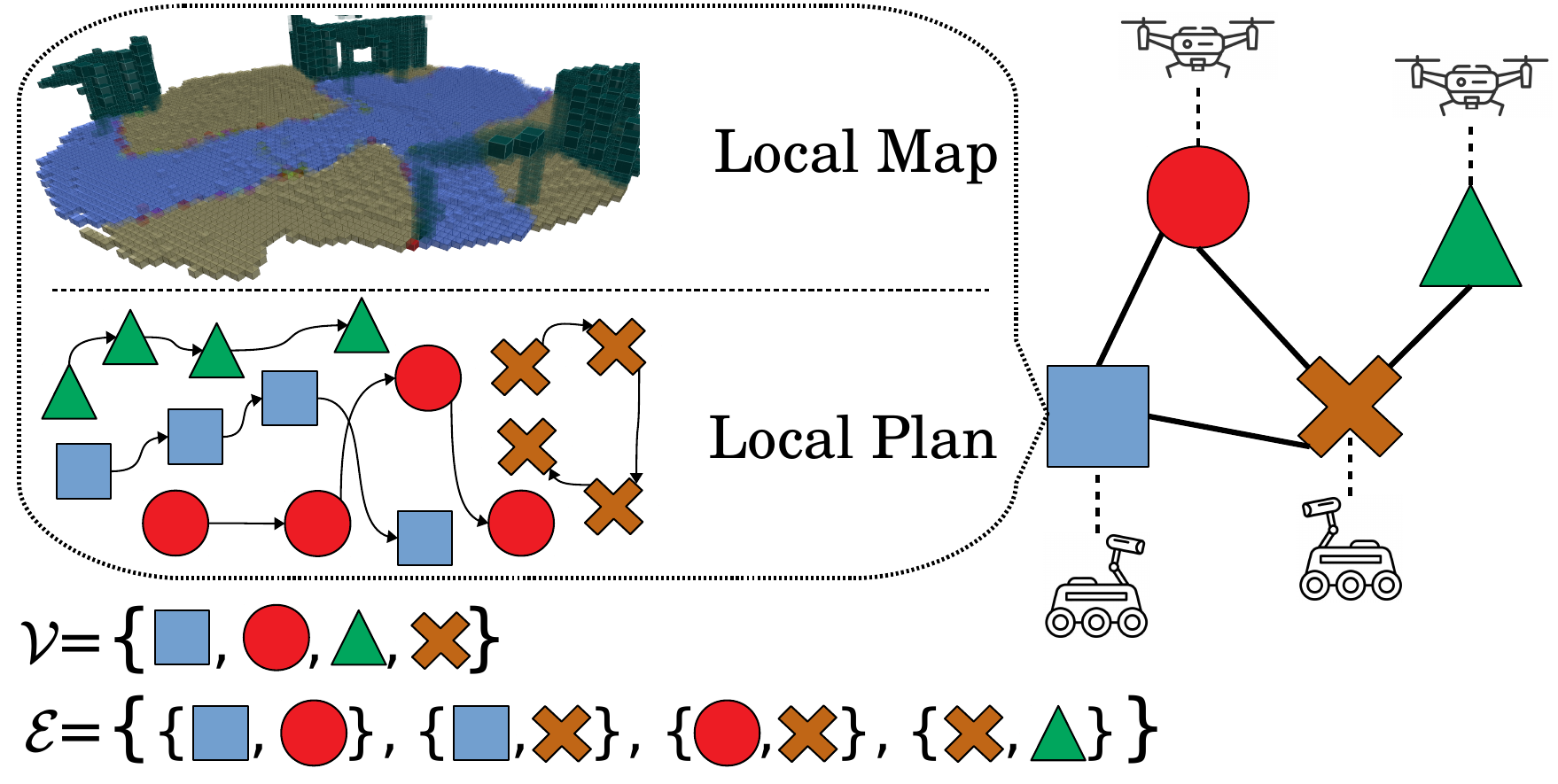}
    \captionsetup{justification=centering}
    \caption{Robot network and the communicated messages}
    \label{subfig:intro_a}
    \end{subfigure}\\
    \begin{subfigure}[t]{\linewidth}
    \centering
    \includegraphics[width=0.9\linewidth]{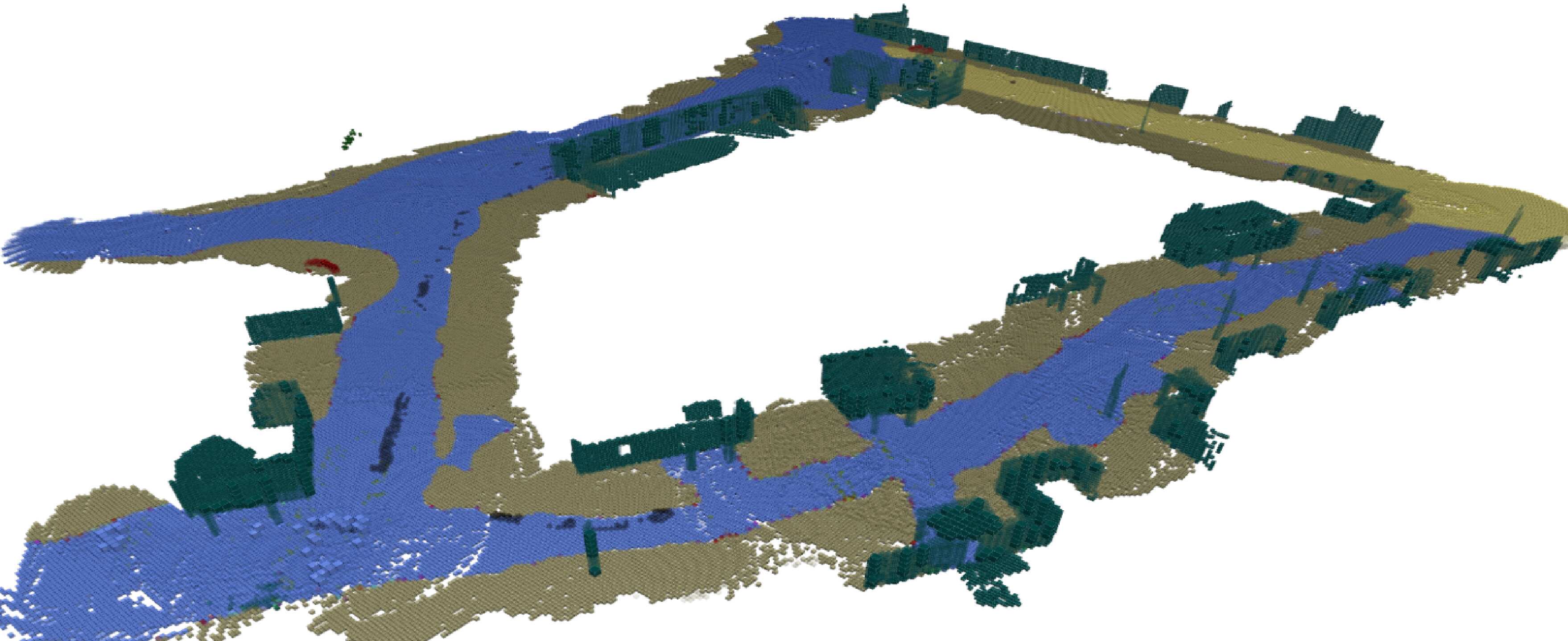}
    \captionsetup{justification=centering}
    \caption{Local maps become globally consistent over time}
    \label{subfig:intro_b}
    \end{subfigure}
    \caption{Overview of our distributed multi-robot active mapping approach. (a) A team of robots, denoted by vertex set $\calV$, collaboratively explores an unknown environment. Each robot builds a local map using onboard sensor measurements and computes a local plan for the team, with the goal of maximizing the collective information gathered by the team. The local maps and plans are communicated over a peer-to-peer network whose connectivity is represented by the edge set $\calE$. (b) As the robot team continues communication, the local maps stored by different robots become globally consistent in that they store similar information about the environment.}
    \label{fig:intro}
\end{figure}

The ability to explore unknown environments and discover objects of interest is a prerequisite for autonomous execution of complex tasks by mobile robots. Active mapping methods \cite{placed2023survey} consider joint optimization of the motion of a robot team and the fidelity of the map constructed by the team. The goal is to compute maximally informative robot trajectories under a limited exploration budget (e.g., time, energy, etc.).


Time-critical applications, such as search and rescue \cite{sar_survey, darpa_1} and security and surveillance \cite{coopertaive_surveillance}, as well as large-scale operations, such as environmental monitoring~\cite{multi_central_1}, substantially benefit if exploration is carried out by a team of coordinating robots. This is traditionally done via multi-robot systems relying on centralized estimation and control \cite{multi_central_1, multi_central_10, multi_central_11, multi_central_12, multi_central_coverage_survey}, where each robot receives local sensor observations, builds its own map, and sends it to a central node for map aggregation and team trajectory computation. The availability of powerful computation onboard small robot platforms makes it possible to develop autonomous exploration algorithms without the need for a central processing node \cite{liu2023active}. Removing the central unit improves the resilience of a multi-robot system with respect to communication-based faults and central node failures \cite{resilience} but brings up new challenges related to distributed storage, computation, and communication. How can one guarantee that the performance of decentralized active mapping would be on par with a centralized architecture in terms of global map consistency and team trajectory optimality?

To address this question, we propose ROAM: \textbf{R}iemannian \textbf{O}ptimization for \textbf{A}ctive \textbf{M}apping with robot teams. ROAM is a decentralized Riemannian optimization algorithm that operates on a communication graph with node variables belonging to a Riemannian manifold and ensures consensus among the node variables. The graph nodes correspond to different robots, while the graph edges model the communication among the robots. In the context of mapping, the node variables are categorical probability mass functions representing probabilistic maps with different semantic classes (e.g., building, vegetation, terrain) at each robot. The consensus constraint requires that the local maps of different robots agree with each other. In the context of planning, the node variables are trajectories of $\SE$ robot poses. Each robot plans trajectories for the whole team using its local information, while the consensus constraint requires that the team trajectories computed by different robots agree. See Fig.~\ref{fig:intro} for an overview of ROAM.



We demonstrate the performance of ROAM in a variety of simulation and real-world experiments using a team of wheeled robots with on-board sensing and processing hardware. Specifically, each robot gathers range and semantic segmentation measurements using an RGBD sensor, and incrementally builds a local 3-D semantic grid map of the environment, where each map cell maintains a probability distribution over object classes. To achieve memory and communication efficiency, an octree data structure is employed to represent the 3-D semantic maps \cite{vaskasi_TRO}. The robots cooperatively find the most informative set of $\SE$ paths for the team to efficiently improve the map and explore the unknown areas while avoiding obstacle collisions. Both multi-robot mapping and planning are performed in the absence of a central estimation and control node and only involve peer-to-peer communication among neighboring robots.


\subsection{Related Work}

Multi-robot active mapping is in essence an optimization problem, with the goal of finding maximally informative robot trajectories, while simultaneously maintaining globally consistent map estimates. Thus, we begin our literature review with identifying relevant works in distributed optimization. The algorithms introduced in \cite{Chen, Gharesifard, Nedic_1, Ram, Tsianos} provide a class of approaches for decentralized gradient-based optimization in the Euclidean space under a variety of constraints such as time variation or communication asymmetry between agents in the network. The survey by Halsted et al.~\cite{multi-robot_opt_survey} provides a comprehensive study of distributed optimization methods for multi-robot applications. In this work, we decompose the task of multi-robot active mapping to two consensus-constrained Riemannian optimization problems, i.e. distributed mapping and distributed path planning. However, naive utilization of the Euclidean optimization techniques in Riemannian manifolds might violate the structure of the optimization domain, leading to infeasible solutions. Therefore, it is required to employ a special family of distributed optimization methods specific to Riemannian manifolds.

Absil et al. \cite{absil2009optimization} presents the foundations of optimization over matrix manifolds, giving rise to many centralized and distributed algorithms in subsequent works. As examples, Chen et al. \cite{chen2021decentralized} and Wang et al. \cite{wang} devise decentralized optimization algorithms for Stiefel manifolds where a Lagrangian function is used to enforce consensus and maintain the manifold structure. Manifold optimization also allows designing efficient learning algorithms where model parameters can be learned using unconstrained manifold optimization as opposed to Euclidean space optimization with projection to the parameter manifold. Zhang et al.~\cite{zhang2016riemannian} and Li et al.~\cite{li2022federated} introduce stochastic learning algorithms for Riemannian manifolds in centralized and federated formats, respectively. Related to our work, Tian et al.~\cite{tian2020asynchronous} present a multi-robot pose-graph simultaneous localization and mapping (SLAM) algorithm which employs gradient-descent local to each robot directly over the $\SE$ space of poses. Our work is inspired by the distributed Riemannian gradient optimization method introduced by Shah \cite{shah2017distributed}. We develop a distributed gradient-descent optimization method for general Riemannian manifolds, and derive specific instantiations for two particular cases, namely the space of probability distributions over semantic maps and the space of $\SE$ robot pose trajectories.

Distributed mapping is a special case of distributed estimation, where a model of the environment is estimated via sensor measurements. Distributed estimation techniques are used in multi-robot localization \cite{atanasov2014joint}, multi-robot mapping \cite{zobeidi_gp}, or multi-robot (SLAM) \cite{kimera_multi}. Paritosh et al. \cite{Paritosh_dist_mapping} define Bayesian distributed estimation as maximizing sensor data likelihood from all agents, while enforcing consensus in the estimates. The present work follows a similar methodology in that we achieve multi-robot Bayesian semantic mapping via distributed maximization of local sensor observation log-likelihood with a consensus constraint on the estimated maps.

Regarding collaborative mapping, an important consideration is the communication of local map estimates among the robots. Corah et al. \cite{corah2019communication} propose distributed Gaussian mixture model (GMM) mapping, where a GMM map is globally estimated, and each robot uses this global map to extract occupancy maps for planning. The use of GMM environment representation for multi-robot exploration is motivated by its lower communication overhead compared to uniform resolution occupancy grid maps. Subsequent works in \cite{dong2022mr} and \cite{wu2022mr} have similarly used distributed GMM mapping for place recognition and relative localization alongside exploration. Alternative techniques for communication-efficient multi-robot mapping include distributed topological mapping \cite{zhang2022mr}, sub-map-based grid mapping \cite{yu2021smmr}, and distributed truncated signed distance field (TSDF) estimation \cite{duhautbout2019distributed}. More recently, the work in \cite{macim} extends neural implicit signed distance mapping to a distributed setting via formulating multi-robot map learning as a consensus-constrained minimization of the loss function. In this case, the robots need to share the neural network parameters to achieve consensus. In our work, we use a semantic octree data structure introduced in our prior work \cite{vaskasi_TRO} to alleviate the communication burden by using a lossless octree compression. Relevant to our work, the authors in \cite{octomap_multi} propose merging of two binary octree maps via summing the occupancy log-odds of corresponding octree leaves. Our work distinguishes itself from \cite{octomap_multi} through a different formulation of multi-robot mapping as a consensus-based Riemannian optimization problem, which enables \begin{enumerate*}[label=\itshape\alph*\upshape)] \item extension to multi-class octree representations, and \item combination of map merging with online map updates from local observations\end{enumerate*}.

Similar to multi-robot mapping, many multi-robot planning methods utilize distributed optimization techniques. The work in \cite{multi-robot_traj_opt_survey} outlines various trajectory planning methods used in multi-robot systems, including graph-based, sampling-based, model-based, and bio-inspired approaches. In particular, graph neural networks (GNNs) have been utilized in \cite{gosrich2022coverage, zhou2022graph} for learning to extract, communicate, and accumulate features from local observations in the context of collaborative multi-robot planning in a distributed way. Coordination and plan deconfliction for multi-robot cooperative tasks is discussed in \cite{deconf}, where robots are assigned priorities in a decentralized manner in order to reach a Pareto equilibrium. In our work, we introduce a decentralized gradient-based negotiation mechanism to resolve $\SE$ path conflicts.

Path planning for autonomous exploration has been extensively studied in the field of active SLAM. Atanasov et al. \cite{atanasov2015decentralized} propose a distributed active SLAM method for robots with linear-Gaussian observation models and a finite set of admissible controls. The authors exploit the conditional entropy formula for the Gaussian noise model to derive an open-loop control policy, called reduced value iteration (RVI), with the same performance guarantees as a closed-loop policy. An anytime version of RVI is proposed in \cite{schlotfeldt2018anytime} using a tree search that progressively reduces the suboptimality of the plan. In contrast to \cite{atanasov2015decentralized, schlotfeldt2018anytime}, we use a probabilistic range-category observation model that accounts for occlusion in sensing. Sampling-based solutions to multi-robot active SLAM have been presented in \cite{kantaros2019asymptotically} and \cite{tzes2021distributed}, with asymptotic optimality guarantees. Cai et al.~\cite{cai2023energy} consider collision safety and energy as additional factors in the cost function for active SLAM using a heterogeneous team of robots. Zhou and Kumar~\cite{zhou2023robust} propose robust multi-robot active target tracking with performance guarantees in regard to sensing and communication attacks, however, the estimation and control are carried out centrally. Tzes et al.~\cite{tzes2023graph} develop a learning-based approach for multi-robot target estimation and tracking, used a GNN to accumulate and process information communicated among one-hop neighbors. The works in \cite{liu2022decentralized, zahroof2023multi} aim to maintain multi-robot network connectivity and collision avoidance via control barrier functions. Another line of research \cite{corah2019distributed, best2019dec} uses decentralized Monte-Carlo tree search for multi-robot path planning for exploration. The interested reader is encouraged to refer to \cite{placed2023survey} for a comprehensive survey of active SLAM methods. Our work distinguishes itself by considering continuous-space planning on a Riemannian manifold, generalizing the previous works in terms of the finite number of controls and the Euclidean robot states.

Related to active SLAM with continuous-space planning, Koga et al.~\cite{icr, icr-lqr} introduce iterative covariance regulation, an $\SE$ trajectory optimization algorithm for single-robot active SLAM with a Gaussian observation model. Model-based \cite{yang2023policy} and model-free \cite{yang2023learning} deep reinforcement learning techniques have been applied to similar single-robot active SLAM problems. Extending to a team of robots, Hu et al.~\cite{hu2020voronoi} propose Voronoi-based decentralized exploration using reinforcement learning, where coordination among the robots takes place via distributed assignment of each Voronoi region to a robot, and the policy generates a 2-D vector of linear and angular velocities. In our work, we formulate multi-robot planning for exploration as a distributed optimization problem in $\SE$ space with a consensus constraint to enforce agreement among the robot plans.

\subsection{Contributions}


Our gradient-based distributed Riemannian optimization approach extends the scope of previous works in multi-robots estimation and planning to enable continuous non-Euclidean state and control spaces, as well as non-linear and non-Gaussian perception models. Our contributions include:
\begin{enumerate}
    \item a distributed Riemannian optimization algorithm for multi-robot systems using only one-hop communication, with consensus and optimality guarantees,
    
    \item a distributed semantic octree mapping approach utilizing local semantic point cloud observations at each robot, 
    
    \item a distributed collaborative planning algorithm for robot exploration, where the search domain is defined as the continuous space of $\SE$ robot pose trajectories,
    
    \item an open-source implementation,\footnote{Open-source software and videos supplementing this paper are available at \url{https://existentialrobotics.org/ROAM_webpage/}.} achieving real-time performance onboard resource-constrained robots in simulation and real-world experiments.
\end{enumerate}

We begin by formulating consensus-constrained Riemannian optimization for multi-agent systems in Sec.~\ref{sec:prob_state}. Next, in Sec.~\ref{sec:dist_Riemannian_opt}, we introduce a distributed Riemannian optimization algorithm with consensus and optimality guarantees. In Sec.~\ref{sec:dist_mapping}, we formulate distributed semantic octree mapping as a special case, where the optimization variables are probability mass functions over the set of possible semantic maps. Sec.~\ref{sec:dist_planning} formulates distributed collaborative planning for robot exploration as another application of distributed Riemannian optimization, where robot trajectories in the $\SE$ manifold are the optimization variables. Lastly, in Sec.~\ref{sec:exp} we evaluate the performance of our proposed distributed multi-robot exploration in several simulation and real-world experiments.

\section{Problem Statement}
\label{sec:prob_state}

Consider a network of agents represented by an undirected connected graph $\calG(\calV, \calE)$, where $\calV$ denotes the set of agents and $\calE \subseteq \calV \times \calV$ encodes the existence of communication links between pairs of agents. Each agent $i \in \calV$ has state $x^i$ which belongs to a compact Riemannian manifold $\calM$. Let $T_{x^i}\calM$ denote the tangent space of $\calM$ at $x^i$ and let $\langle v,u \rangle_{x^i} \in \bbR$ with $u,v \in T_{x^i}\calM$ be a Riemannian metric on $\calM$~\cite[Ch.3]{boumal2022intromanifolds}. The norm of a tangent vector $v \in T_{x^i}\calM$ is defined by the Riemannian metric as $\|v\|_{x^i} = \sqrt{\langle v, v \rangle_{x^i}}$. Additionally, let $\Exp_{x^i}{(\cdot)}: T_{x^i}\calM \rightarrow \calM$ denote the exponential map on $\calM$ at $x^i$, and denote its inverse as $\Exp^{-1}_{x^i}{(\cdot)}: \calM \rightarrow T_{x^i}\calM$. 


We associate a local objective function $f^i(\cdot): \calM \rightarrow \bbR$ with each agent $i \in \calV$. Our goal is to maximize the cumulative objective function over the joint agent state $\bfx = (x^1, \ldots, x^{|\calV|})$:
\begin{equation}\label{eq:gen_opt}
    F(\bfx) = \frac{1}{|\calV|} \sum_{i\in \calV} f^i(x^i).
\end{equation}
The global objective can be maximized using $|\calV|$ independent local optimizations. However, in many applications it is necessary to find a common solution among all agents. For example, in multi-robot mapping, the robots need to ensure that their local maps are consistent and take into account the observations from other robots. Therefore, the global optimization problem needs to be constrained such that the agents reach consensus on $\bfx$ during optimization. For this aim, we define an aggregate distance function $\phi(\bfx): \calM^{|\calV|} \rightarrow \bbR_{\geq 0}$:
\begin{equation}\label{eq:consensus_constraint}
    \phi(\bfx) = \sum_{\{i, j\} \in \calE} A_{ij} d^2(x^i, x^j),
\end{equation}
where $A$ is a symmetric weighted adjacency matrix corresponding to the graph $\calG$, and $d(\cdot): \calM \times \calM \rightarrow \bbR_{\geq 0}$ is a distance function on the Riemannian manifold $\calM$, i.e., computes the length of the geodesic (shortest path) between pairs of elements in $\calM$. The definition of the aggregate distance function in \eqref{eq:consensus_constraint} implies that consensus will be reached if and only if $\phi(\bfx) = 0$. Hence, adding $\phi(\bfx) = 0$ as a constraint to \eqref{eq:gen_opt} would require feasible joint states $\bfx = ({x^1}, \ldots, {x^{|\calV|}})$ to satisfy ${x^i} = {x^j}$ for all $i, j \in \calV$.

\begin{problem}\label{prob:consensus_constrained_optimization}
    Consider a connected graph $\calG = (\calV, \calE)$ where each node $i \in \calV$ represents an agent with state $x^i \in \calM$ and local objective function $f^i(x^i)$. Find a joint state $\bfx$ that maximizes the following objective function:
    \begin{equation}
    \begin{aligned}
        \max_{\bfx}& \quad F(\bfx) = \frac{1}{|\calV|} \sum_{i \in \calV} f^i(x^i),\\
        \text{s.t.}& \quad x^i \in \calM, \;\; \forall i \in \calV, \;\; \text{and} \;\; \phi(\bfx) = 0,
    \label{eq:dist_opt}
    \end{aligned}
    \end{equation}
    where $\phi(\bfx) = 0$ is the consensus constraint defined in \eqref{eq:consensus_constraint}.
\end{problem}

As we discuss in Sec.~\ref{sec:dist_mapping} and Sec.~\ref{sec:dist_planning}, both multi-robot mapping and multi-robot trajectory optimization can be formulated as consensus-constrained optimization problems as in \eqref{eq:dist_opt}. In the case of mapping, the manifold $\calM$ is the probability simplex capturing map density functions while the local objective $f^i(x^i)$ is the log-likelihood of the observations made by robot $i$. In the case of trajectory optimization, $\calM$ represents the space of 3-D pose (rotation and translation) trajectories in $\SE$, and $f^i(x^i)$ is a collision and perception-aware objective for the robot pose trajectories. In the next section, we develop a distributed gradient-based optimization algorithm to solve \eqref{eq:dist_opt} using only local computation and single-hop communication.



\section{Distributed Riemannian Optimization}
\label{sec:dist_Riemannian_opt}


\begin{algorithm}[t]
\caption{Distributed Riemannian Optimization}
\begin{algorithmic}[1]
\renewcommand{\algorithmicrequire}{\textbf{Input:}}
\renewcommand{\algorithmicensure}{\textbf{Output:}}
\Require Network $\calG(\calV, \calE)$ and initial state ${x^i}^{(0)}$
\Ensure Consensus optimal solution to \eqref{eq:dist_opt}
\For{$k \in \bbZ_{\geq 0}$}
    \For{each agent $i \in \calV$}
        \LineComment{Promote consensus with step size $\epsilon$:}
        \State ${\Tilde{x}^i}^{(k)} = \Exp_{{x^i}^{(k)}}{(- \epsilon \grad_{x^i}{\phi(\bfx)}|_{\bfx = \bfx^{(k)}})}$
        \label{alg_line:consensus}
        \LineComment{Optimize local objective with step size $\alpha^{(k)}$:}
        \State ${x^i}^{(k+1)} = \Exp_{{\Tilde{x}^i}^{(k)}}{(\alpha^{(k)} \grad{f^i(x^i)}|_{x^i = {\Tilde{x}^i}^{(k)}})}$
        \label{alg_line:optimize_local}
    \EndFor
\EndFor
\State \Return ${x^i}^{(k)}$
\end{algorithmic}
\label{alg:dist_opt}
\end{algorithm}

The problem in \eqref{eq:dist_opt} has a specific structure, maximizing a sum of local objectives subject to a consensus constraint among all $x^i$, $i \in \calV$. We develop a distributed gradient-based algorithm to solve \eqref{eq:dist_opt}. The idea is to interleave gradient updates for the local objectives with gradient updates for the consensus constraint at each agent. Alg.~\ref{alg:dist_opt} formalizes this idea. The update step in line~\ref{alg_line:consensus} guides the local state $x^i$ towards satisfaction of the consensus constraint, with a step size of $\epsilon$. The gradient of $\phi(\bfx)$ with respect to $x^i$, denoted as $\grad_{x^i}{\phi(\bfx)}$, lies in the tangent space $T_{x^i}\calM$. Hence, the exponential map is used to retract the gradient update $-\epsilon \grad_{x^i}{\phi(\bfx)}|_{\bfx = \bfx^{(k)}}$ to the manifold $\calM$. The gradient $\grad_{x^i}{\phi(\bfx)}$ can be expressed as a sum of gradients with respect to the neighbors $\calN_i = \{j | A_{ij} > 0\}$ of agent $i$:
\begin{equation}
\begin{aligned}
    \grad&_{x^i}{\phi(\bfx)} =\\
    &\sum_{j \in \calN_i} A_{ij} \grad_{x^i}{d^2(x^i, x^j)} = - 2 \sum_{j \in \calN_i} A_{ij} \Exp^{-1}_{x^i}{(x^j)}.\nonumber
\end{aligned}
\end{equation}
Therefore, line~\ref{alg_line:consensus} requires \emph{only single-hop communication} between agent $i$ and its neighbors $\calN_i$. Line~\ref{alg_line:optimize_local} carries out an update with step size $\alpha^{(k)}$ in the direction of the gradient of the local objective $f^i(\cdot)$, computed at the updated state ${\Tilde{x}^i}^{(k)}$. Similar to the consensus update step, the exponential map is used in to retract $\grad{f^i(x^i)}$ and apply it to the point ${\Tilde{x}^i}^{(k)}$. Line~\ref{alg_line:optimize_local} is local to each agent $i$ and does not require communication. The two update steps are continuously applied until a maximum number of iterations is reached or the update norm is smaller than a threshold.

\begin{example}
    Consider a sensor network where several agents gather data that is not supposed to be shared over the network, due to either privacy reasons or bandwidth limitations. Our Riemannian optimization algorithm enables distributed processing of the global data, accumulated over all agents, without actually sharing the data. As an example, Fig.~\ref{fig:roam_example} illustrates applying Alg.~\ref{alg:dist_opt} to compute the leading eigenvector of the covariance of data. Fig.~\ref{subfig:roam_example_a} depicts the global data distribution $Z = [Z_1^\top Z_2^\top]^\top$, such that different segments of the data $Z_1$ and $Z_2$ are known to agent $1$ and agent $2$, separately. This problem can be formulated as:
    \begin{equation}\label{eq:roam_example}
    \begin{aligned}
        \max_{x^1, x^2}& \quad \sum_{i \in \{1, 2\}} (Z_i x^i)^\top Z_i {x^i},\\
        \text{s.t.}& \quad x^1, x^2 \in \bbS^1 \;\; \text{and} \;\; \arccos({x^1}^\top x^2) = 0,
    \end{aligned}
    \end{equation}
    where the domain manifold is the unit circle $\bbS^1$, and cosine distance is used as the distance function. Note that for all $x^1$ and $x^2$ that satisfy the consensus constraint $\arccos({x^1}^\top x^2) = 0$, the objective function is equivalent to the one for the centralized leading eigenvector problem. Hence, we expect to find the eigenvector for the covariance of the global data matrix $Z$ by employing Alg.~\ref{alg:dist_opt} to \eqref{eq:roam_example}. Fig.~\ref{subfig:roam_example_b} shows an initialization of $x^1$ and $x^2$ over the unit circle $\bbS^1$. While the Riemannian gradients of $\phi(\cdot)$ and $f^i(\cdot)$ are tangent vectors to $\bbS^1$, the consensus and local objective function update steps keep the state on the $\bbS^1$ manifold thanks to the exponential map (see the circular arcs in Fig.~\ref{subfig:roam_example_c}):
    \begin{equation}
        \Exp_{x^i}(v) = \cos(\sqrt{v^\top v}) x^i + \sin(\sqrt{v^\top v}) \frac{v}{\sqrt{v^\top v}}.
    \end{equation}
    Note that the consensus update (green arc) acts in the direction of agreement between $x^1$ and $x^2$, whereas the local objective function gradient tries to steer the states $x^i$ towards the leading eigenvector of their respective data $Z_i$. Although each agent has only partial access to $Z$, both $x^1$ and $x^2$ eventually converge to $x^*$, namely the leading eigenvector of the covariance for the global data matrix $Z$, as Fig.~\ref{subfig:roam_example_d} suggests.\hfill$\bullet$
\end{example}

\begin{figure}[t]
    \begin{subfigure}[t]{0.5\linewidth}
    \centering
    \includegraphics[width=\linewidth]{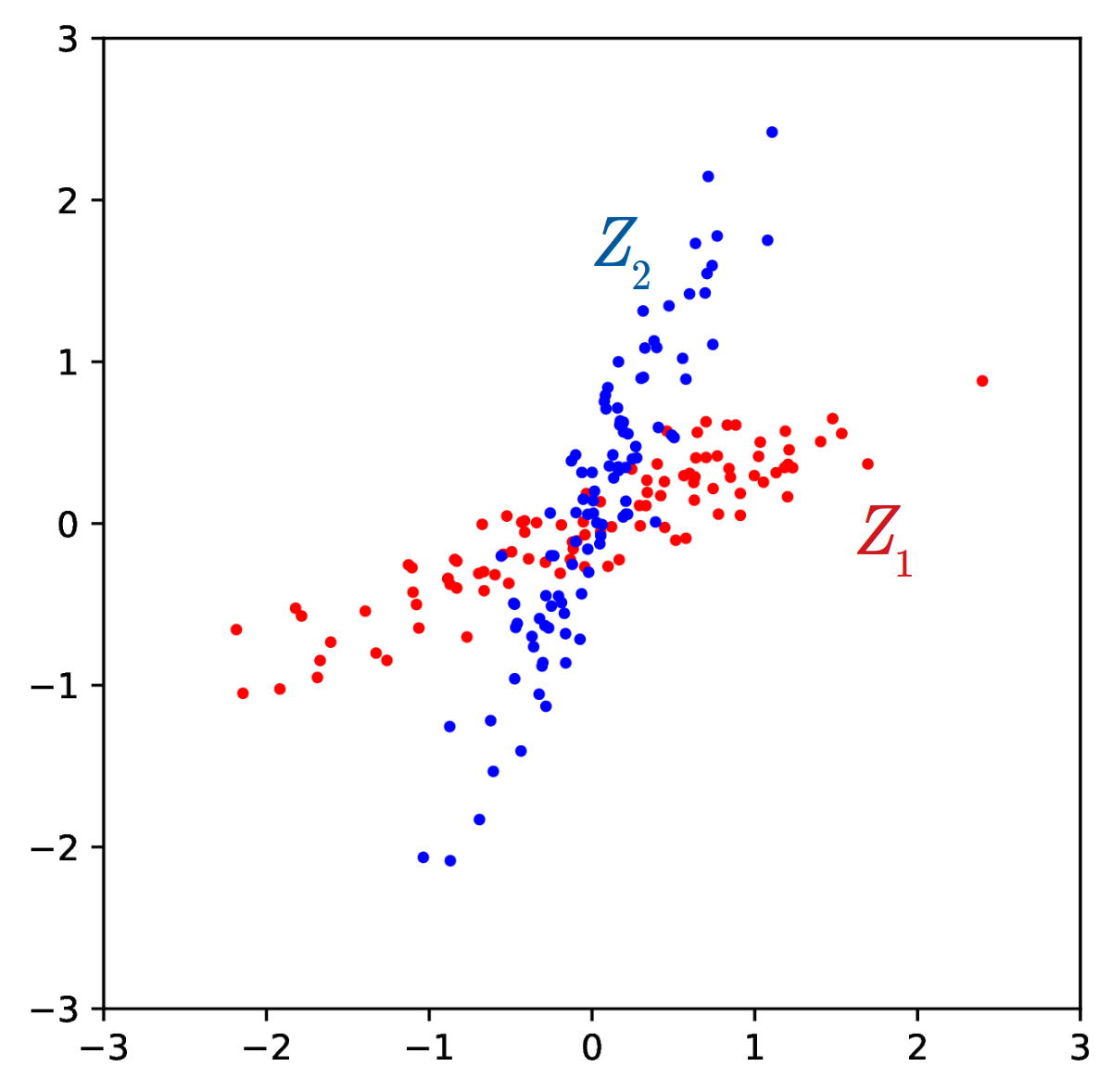}
    \captionsetup{justification=centering}
    \caption{Data distribution}
    \label{subfig:roam_example_a}
    \end{subfigure}%
    \hfill%
    \begin{subfigure}[t]{0.5\linewidth}
    \centering
    \includegraphics[width=\linewidth]{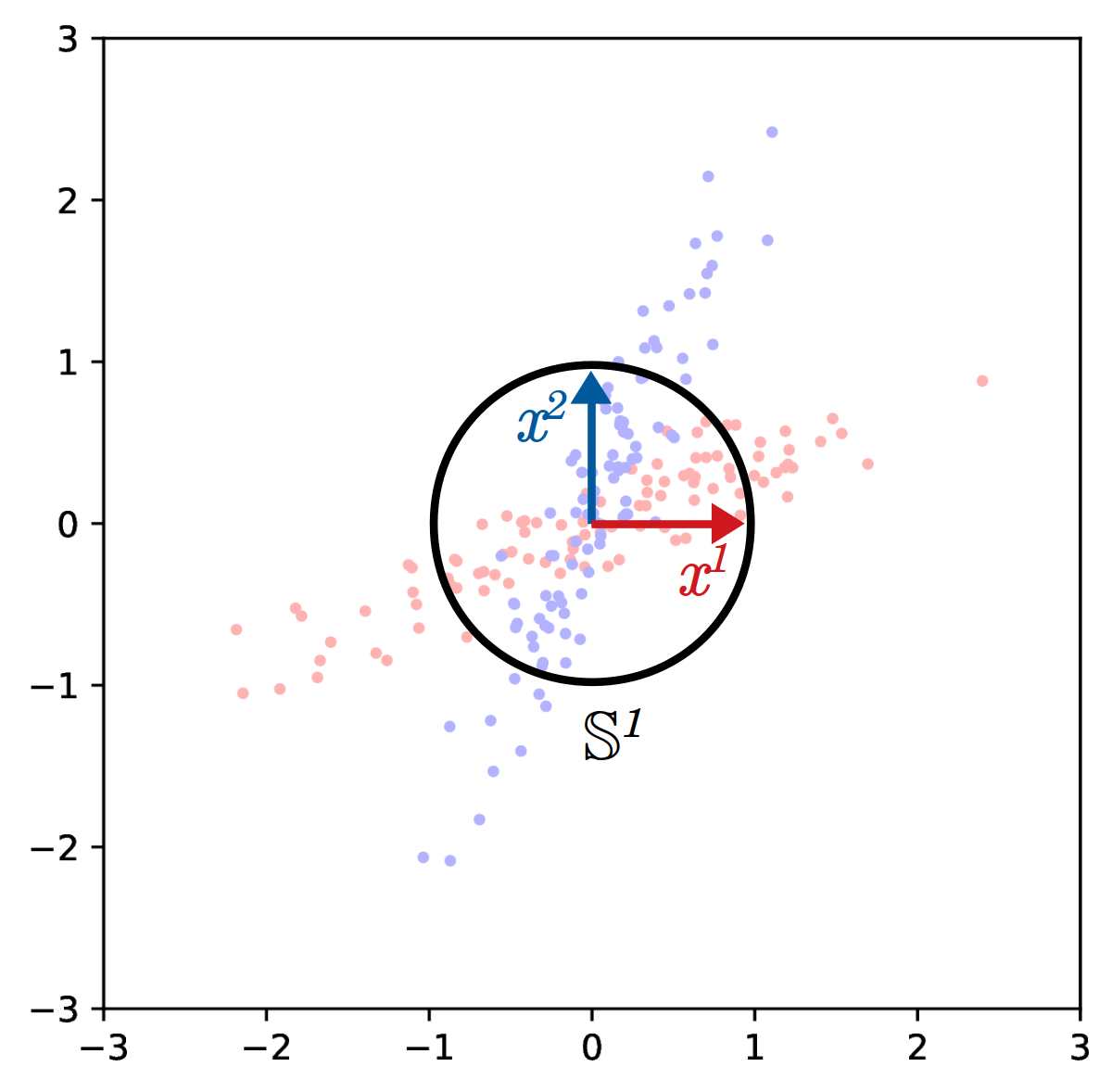}
    \captionsetup{justification=centering}
    \caption{State initialization}
    \label{subfig:roam_example_b}
    \end{subfigure}\\
    \begin{subfigure}[t]{0.5\linewidth}
    \centering
    \includegraphics[width=\linewidth]{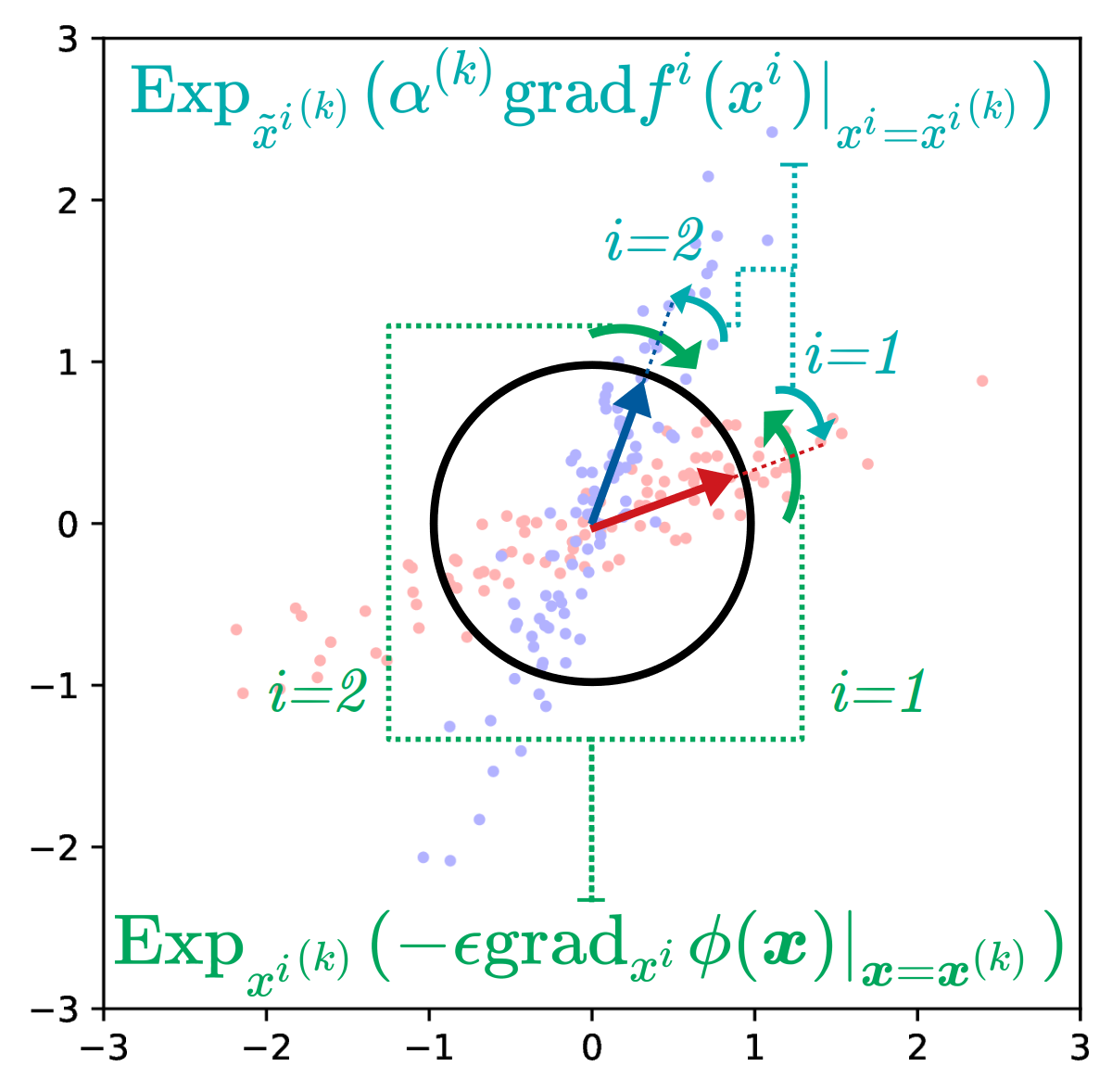}
    \captionsetup{justification=centering}
    \caption{Update steps}
    \label{subfig:roam_example_c}
    \end{subfigure}%
    \hfill%
    \begin{subfigure}[t]{0.5\linewidth}
    \centering
    \includegraphics[width=\linewidth]{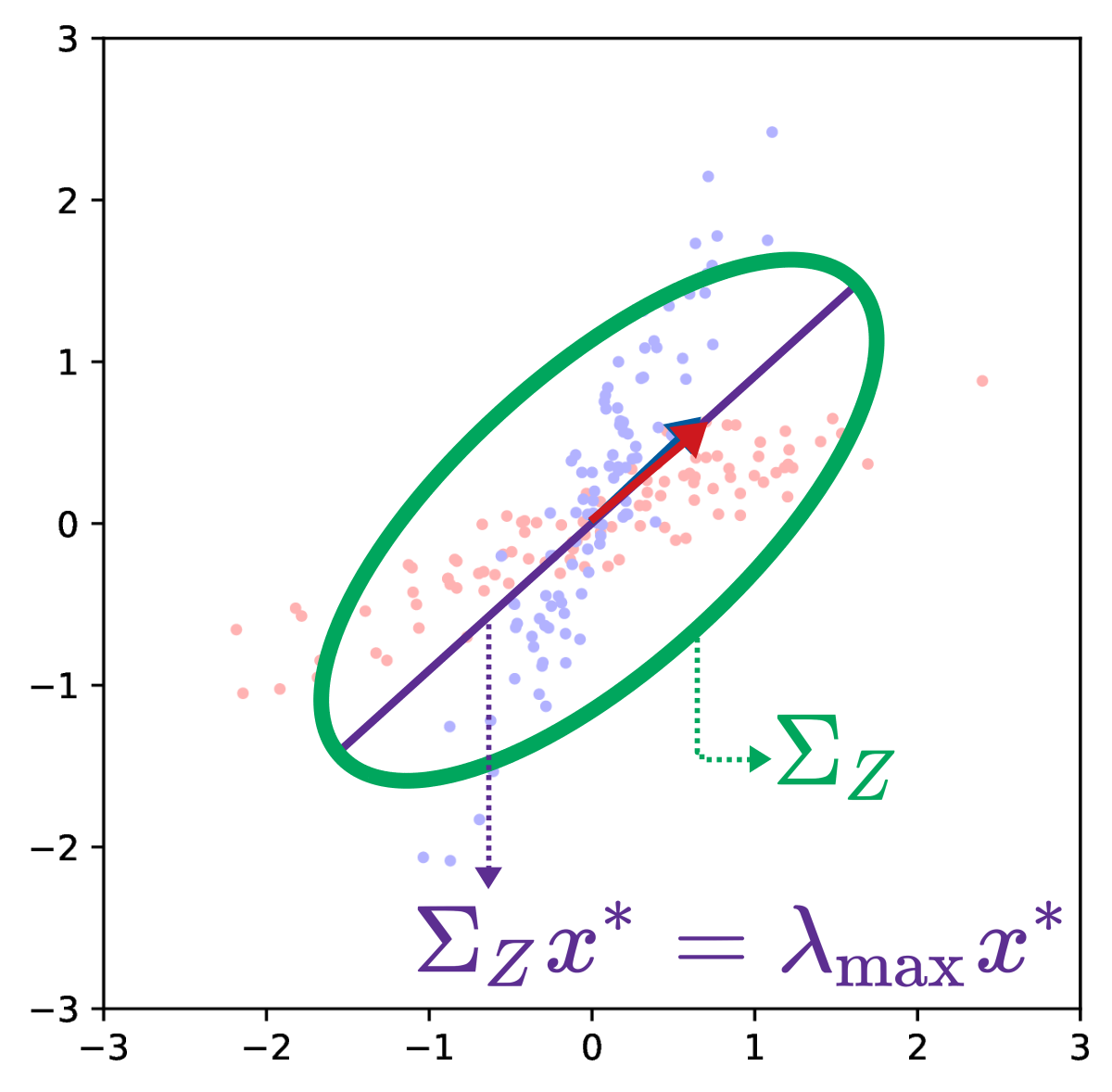}
    \captionsetup{justification=centering}
    \caption{Convergence to the global leading eigenvector}
    \label{subfig:roam_example_d}
    \end{subfigure}
    \caption{Application of Alg.~\ref{alg:dist_opt} to the leading eigenvector problem. (a) Data distribution $Z = [Z_1^\top Z_2^\top]^\top$, where $Z_1$ and $Z_2$ are separately known to agent $1$ and agent $2$. (b) State initialization, limited to the unit circle $\bbS^1$ since we are only interested in eigenvector directions. (c) Consensus and local objective function updates, shown in green and teal arcs, respectively. The exponential mapping of $\bbS^1$ maintains the manifold structure of the states throughout the update steps. (d) Level-set of the covariance matrix of the global data $Z$ is shown, alongside its leading eigenvector $x^*$.}
    \label{fig:roam_example}
\end{figure}

Next, we study whether Alg.~\ref{alg:dist_opt} achieves consensus and optimality. We make several assumptions to ensure that the problem is well-posed in accordence with prior work on distributed optimization \cite{Chen, Gharesifard, Li, Nedic_1, Nedic_2, Ram, Srivastava, Tsianos, SEPULCHRE201156}.

\begin{definition}\label{def:convexity}
    A differentiable function $f: \calM \rightarrow \bbR$ is geodesically convex if and only if for any $x, y \in \calM$:
    \begin{equation}
        f(x) \geq f(y) + \langle \grad{f(y)}, \Exp^{-1}_{y}{(x)} \rangle_{y}.\nonumber
    \end{equation}
    The function $f$ is geodesically concave if the above inequality is flipped.
\end{definition}

%

%
%
\begin{assumption}\label{assumption:general}
    Assume the following statements hold for Problem~\ref{prob:consensus_constrained_optimization} and Alg.~\ref{alg:dist_opt}.
    \begin{itemize}
        \item The Riemannian manifold $\calM$ is compact.

        \item The local objective functions $f^i$, $\forall i \in \calV$, are smooth, geodesically concave, and their Riemannian gradients are bounded by some constant $C$:
        \begin{equation}
            \|\grad f^i(x^i)\|_{x^i} \leq C, \; \forall x^i \in \calM, \; \forall i \in \calV.\nonumber
        \end{equation}
        
        \item The weighted adjacency matrix $A$ of the graph $\calG$ is row-stochastic, i.e., $\sum_{j \in \calV} A_{ij} = 1$.
                
        \item The squared distance function $d^2: \calM \times \calM \rightarrow \bbR_{\geq 0}$ is geodesically convex.

        \item The step sizes $\alpha^{(k)} > 0$ for the update step in line~\ref{alg_line:optimize_local} satisfy the Robbins-Monro conditions:
        \begin{equation}
            \sum_{k=0}^{\infty} \alpha^{(k)} = \infty,\qquad \sum_{k=0}^{\infty} {\alpha^{(k)}}^2 < \infty, \qquad \forall k \geq 0.\nonumber
        \end{equation}        
    \end{itemize}
\end{assumption}

In addition to the assumptions above, we require an additional condition to prove that Alg.~\ref{alg:dist_opt} achieves consensus, i.e., $\phi(\bfx) = 0$. Let $\calT_{y^i}^{x^i}: T_{y^i}\calM \rightarrow T_{x^i}\calM$ denote parallel transport \cite[Ch.10]{boumal2022intromanifolds} from the tangent space at $y^i$ to the tangent space at $x^i$. For points $x^i, x^j, y^j, y^i \in \calM$, consider the geodesic loop $x^i \rightarrow x^j \rightarrow y^j \rightarrow y^i \rightarrow x^i$ with corresponding tangent vectors $v_{x}^{ij}, v_{y}^{ij}, v_{xy}^{i}, v_{xy}^{j}$ defined as:
\begin{equation}
    v_{x}^{ij} = \Exp^{-1}_{x^i}{(x^j)},\; v_{xy}^{i} = \Exp^{-1}_{x^i}{(y^i)},\nonumber
\end{equation}
with similar definitions for $v_{y}^{ij}$ and $v_{xy}^{j}$. Let $v_{xy}^{ij} \in T_{x^i}\calM$ be the net tangent vector transported to $T_{x^i}\calM$:
\begin{equation}\label{eq:net_vec}
    v_{xy}^{ij} = v_{x}^{ij} + \calT_{x^j}^{x^i}v_{xy}^{j} - v_{xy}^{i} - \calT_{y^i}^{x^i}v_{y}^{ij}.
\end{equation}
\begin{assumption}\label{assumption:net_curve_bound}
    For a $\rho > 0$ and any $4$-tuple $(x^i, x^j, y^j, y^i) \in \calM$, assume the norm of the net tangent vector $v_{xy}^{ij}$ is bounded by the lengths of the opposite geodesics along the loop:
    \begin{equation}\label{eq:assumption_2}
        \|v_{xy}^{ij}\|_{x^i} \leq \rho \min\{\|v_{xy}^{i}\|_{x^i} + \|v_{xy}^{j}\|_{x^j}, \|v_{x}^{ij}\|_{x^i} + \|v_{y}^{ij}\|_{y^i}\}.
    \end{equation}
\end{assumption}

In Euclidean space, the net tangent vector $v_{xy}^{ij}$ is equivalent to zero linear displacement; hence, the assumption holds for any $\rho \geq 0$. Similarly, the manifold $\bbS^1$ of example \eqref{eq:roam_example} satisfies \eqref{eq:assumption_2} due to zero angular displacement. For a general case, $v_{xy}^{ij}$ can be non-zero due to the curvature of the manifold. This is dual to the fact that, for a zero net tangent vector $v_{xy}^{ij}$, the corresponding geodesics might not form a loop. The assumption in \eqref{eq:assumption_2} essentially imposes a condition over curvature of the manifold so that the norm of $v_{xy}^{ij}$ is limited by the length of the geodesic loop. Based on the above assumptions, we show consensus and optimality for Alg.~\ref{alg:dist_opt}.

\begin{theorem}
\label{thm:dist_opt_cons_opt}
Consider the consensus-constrained Riemannian optimization problem in \eqref{eq:dist_opt} and the distributed Riemannian optimization algorithm in Alg.~\ref{alg:dist_opt}. Suppose that Assumptions \ref{assumption:general} and \ref{assumption:net_curve_bound} hold and step size $\epsilon$ is chosen such that $\epsilon \in (0, 2/L)$ with $L = 4 (1 + \rho)$. Then, Alg.~\ref{alg:dist_opt} provides a solution to \eqref{eq:dist_opt} with the following properties.
    \begin{enumerate}
        \item The joint state $\bfx^{(k)}$ converges to $\bfx^{(\infty)} \in \calM^{|\calV|}$, where $\bfx^{(\infty)}$ is a consensus configuration, i.e., ${x^i}^{(\infty)} = {x^j}^{(\infty)}$ for all $i, j \in \calV$.
        
        \item Let $\bfx^*$ be an optimal solution to \eqref{eq:dist_opt}. The optimal value $F(\bfx^*)$ is a lower-bound for the maximum of $F(\bfx^{(k)})$ across all iterations:
        \begin{equation}\label{eq:opt_bound}
            F(\bfx^*) \leq \lim_{k_{\text{max}} \rightarrow \infty} \; \max_{0 \leq k \leq k_{\text{max}}} F(\bfx^{(k)}).
        \end{equation}
    \end{enumerate}
\end{theorem}

\begin{proof}
See Appendix~\ref{app:dist_opt_cons_opt}.
\end{proof}

We stress that, while the optimal solution $\bfx^*$ and the convergence point $\bfx^{(\infty)}$ of Alg.~\ref{alg:dist_opt} are both consensus configurations, the optimality bound of \eqref{eq:opt_bound} can potentially admit a solution $\bfx^{(k)}$ that does not satisfy the consensus constraint. For the Euclidean case, Nedi{\'c}~\cite[Ch.5]{Nedic_3} shows that $d^2(\bfx^*,\bfx^{(k)})$ is a Lyapunov function, and subsequently, $F(\bfx^*) = F(\bfx^{(\infty)})$ holds. However, a similar derivation for $d^2(\bfx^*,\bfx^{(k)})$ has not been found for a general Riemannian manifold, due to the complexity added by the curvature.

Alg.~\ref{alg:dist_opt} establishes consensus and an optimality bound without requiring identical initial states ${x^i}^{(0)} = x_0$ for all $i \in \calV$ or parallel transport of the gradients between neighboring agents. Hence, our distributed Riemannian optimization provides an approach to solve multi-robot problems with communication constraints. The main requirement to use Alg.~\ref{alg:dist_opt} is to express a multi-robot optimization problem in the form of \eqref{eq:dist_opt}, with local objectives $f^i(\cdot)$ and distance measure $\phi(\cdot)$ defined as smooth concave and convex functions in $\calM$ and $\calM^{|\calV|}$, respectively. In the absence of concavity for the objective functions or convexity for the consensus constraint, Alg.~\ref{alg:dist_opt} can still be utilized to obtain a solution with local consensus and optimality guarantees.

In the next two sections, we apply Alg.~\ref{alg:dist_opt} to achieve simultaneous multi-robot mapping and planning. We refer to our approach as \emph{Riemannian Optimization for Active Mapping (ROAM)}. In Sec.~\ref{sec:dist_mapping}, we apply Alg.~\ref{alg:dist_opt} to multi-robot estimation of semantic octree maps, while in Sec.~\ref{sec:dist_planning} we use Alg.~\ref{alg:dist_opt} to achieve multi-robot motion planning for exploration and active estimation of semantic octree maps.

\section{Multi-Robot Semantic Octree Mapping}
\label{sec:dist_mapping}

In this section, we design a decentralized multi-robot octree mapping algorithm using the results from Sec.~\ref{sec:dist_Riemannian_opt}. We consider a team of robots gathering local sensor measurements and communicating map estimates with one-hop neighbors in order to build a globally consistent common map. The robots are navigating in an environment consisting of disjoint sets $\calS_c \subset \mathbb{R}^3$, each associated with a semantic category $c \in \calC := \crl{0,1,\ldots,C}$. Let $\calS_0$ represent the free space and let each $\calS_c$ for $c >0$ represent a different category, such as building, vegetation, terrain. Each robot $i \in \calV$ is equipped with a mounted sensor that provides a stream of semantically-annotated point clouds in the sensor frame. Such information may be obtained by processing the measurements of an RGBD camera \cite{bonnet} or a LiDAR with a semantic segmentation algorithm \cite{rangenet}. We model a point cloud as a set $\bfz^i_t = \{(r^i_{t,b}, y^i_{t,b})\}_{b=1}^B$ of $B$ rays at time $t$, containing the distance $r^i_{t,b} \in \bbR_{\geq 0}$ from the sensor's position to the closest obstacle along the ray in addition to object category $y^i_{t,b} \in \calC$ of the obstacle (see Fig.~\ref{fig:mapping_setup}).

\begin{figure}[t]
    \centering
    \includegraphics[width=0.9\linewidth]{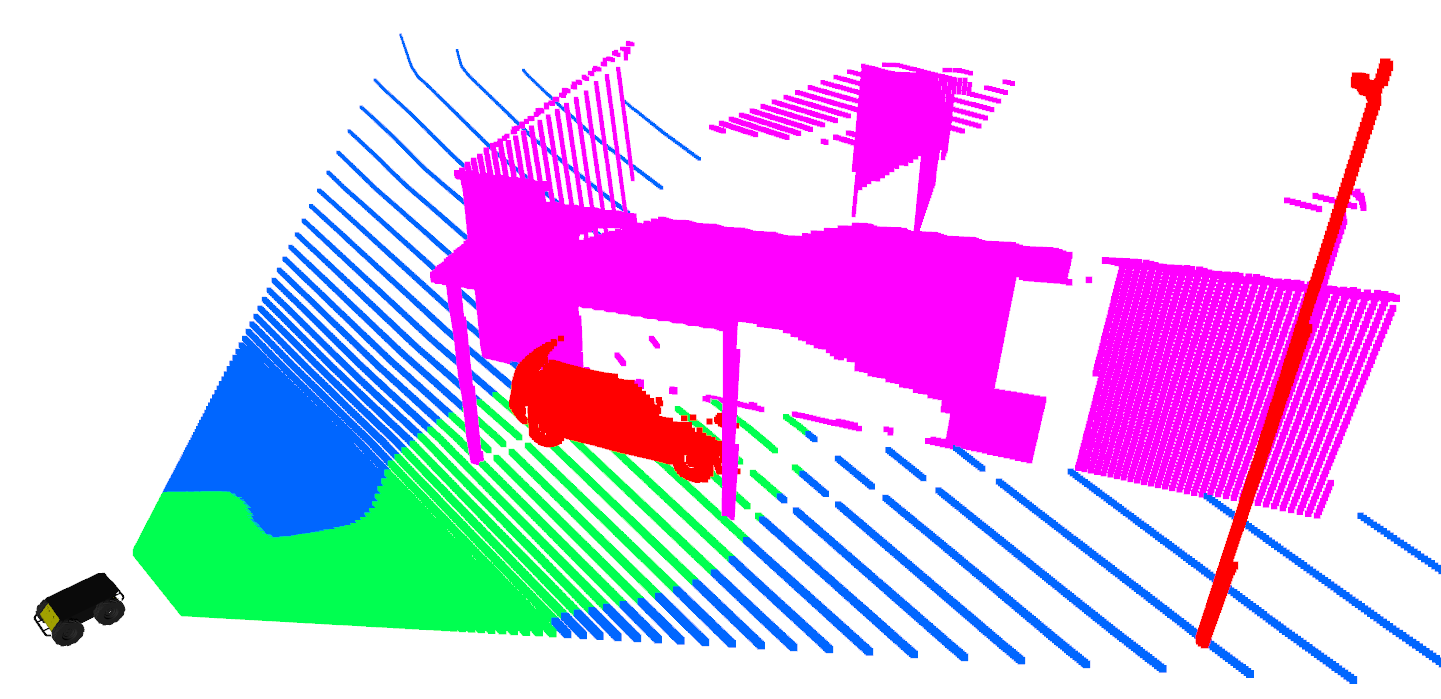}\\
    \includegraphics[width=0.9\linewidth]{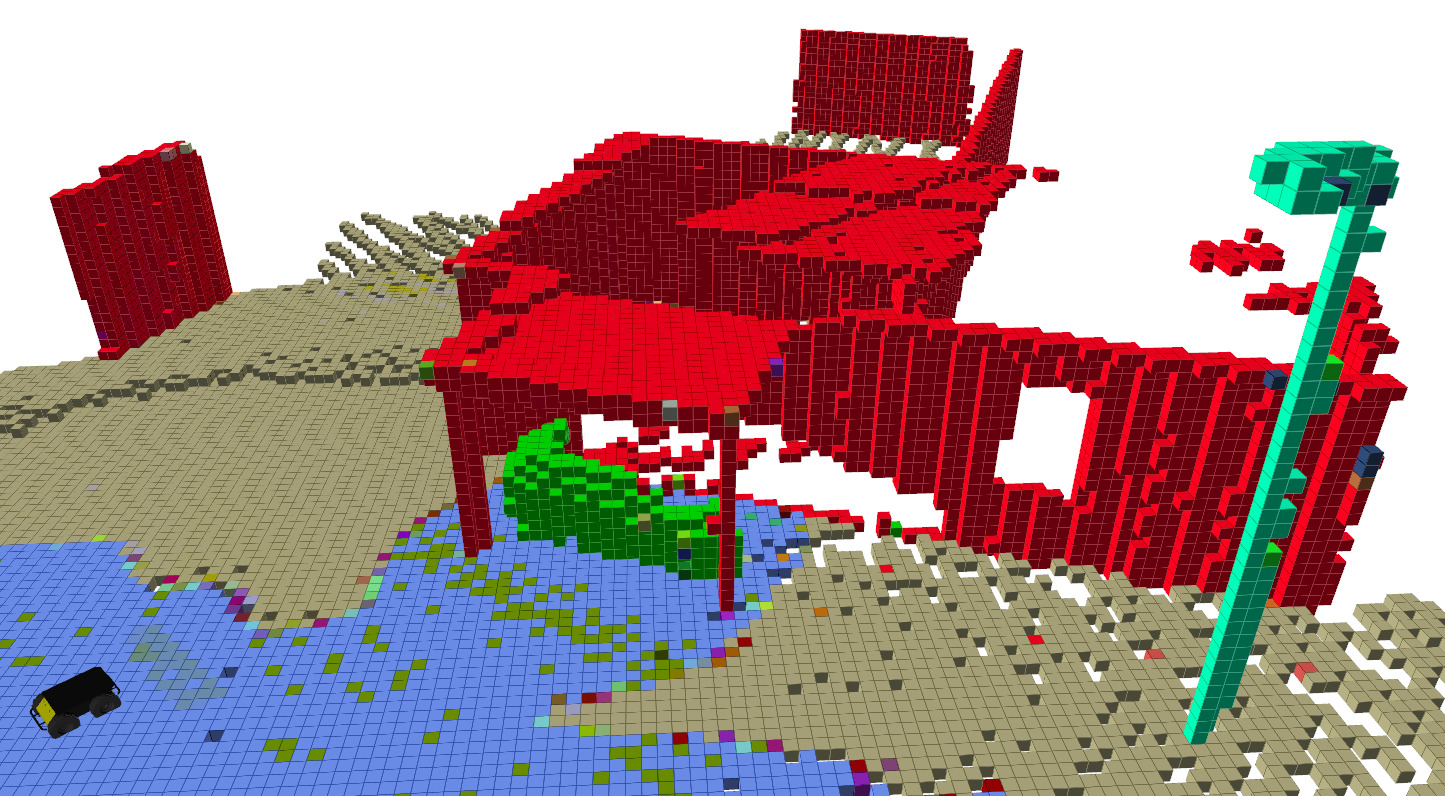}
    \caption{Semantically annotated point cloud (top) obtained from an RGBD sensor, where each object category is shown with a unique color and an octree map (bottom) obtained from the semantic point cloud.}
    \label{fig:mapping_setup}
\end{figure}

We represent the map $\bfm$ as a 3-D grid of $N$ independent cells, where each individual cell $m$ is labeled with a category in $\calC$. To model measurement noise, we use a probability density function (PDF) $q^i(\bfz^i_t \mid \bfm)$ as the observation model of each robot. The observation model $q^i(\bfz^i_t \mid \bfm)$ depends on the sensor pose as well but we assume that accurate sensor poses are available from localization and calibration between the robot body frame and the sensor frame. We intend to perform probabilistic mapping, which requires maintaining a PDF of the map, and updating it based on sensor observations. For this aim, we maximize the sum of expected log-likelihood of the measurements up to time $t$, i.e. local observations $\bfz_{1:t}^i$ collected from each robot $i \in \calV$\footnote{It can be shown that maximizing the sum of expected log-likelihood of the data is equivalent to minimizing the KL-divergence between the true and the evaluated observation models. See \cite{Paritosh_dist_mapping} for more details.}:
\begin{equation}
    \max_{p \in \calP} \sum_{i \in \calV} \bbE_{\bfm \sim p} \big[\log{q^i(\bfz_{1:t}^i|\bfm)}\big],
\label{eq:log_likelihood_obj}
\end{equation}
where $\calP$ is the space of all probability mass functions (PMF) over the set of possible maps:
\begin{equation}\label{eq:prob_manifold}
    \calP = \{p(\cdot) | \sum_{\bfm} p(\bfm) = 1, \; p(\bfm) \geq 0 \;\; \forall\bfm \in \calC^N\}
\end{equation}
The map cell independence assumption allows for decomposing the measurement log-likelihood as a sum over individual map cells $m$, as indicated in the following lemma.

\begin{lemma}\label{lemma:log_likelihood_obj_decomp}
The objective function in \eqref{eq:log_likelihood_obj} can be expressed as a sum over all map cells and all observations:
\begin{equation}
    \sum_{i \in \calV} \sum_{\tau = 1}^{t} \Big( \log{q^i(\bfz_{\tau}^i)} + \sum_{n =1}^N \bbE_{m \sim p_n} \big[\log{\frac{q^i(m | \bfz_{\tau}^i)}{p_n(m)}}\big] \Big),
\label{eq:log_likelihood_obj_decomp}
\end{equation}
where $q^i(\bfz_{\tau}^i)$ is the marginal density of the observation, $p_n(\cdot)$ denotes the PMF of the $n$-th map cell, and $q^i(m | \bfz_{\tau}^i)$ is an \emph{inverse observation model} that represents the sensor noise properties (see (10) in \cite{vaskasi_TRO}).
\end{lemma}

\begin{proof}
See Appendix~\ref{app:log_likelihood_obj_decomp}.
\end{proof}

The log-density term $\log{q^i(\bfz_{\tau}^i)}$ in \eqref{eq:log_likelihood_obj_decomp} does not depend on any of the map probabilities $p_n(\cdot)$, $n \in \crl{1, \ldots, N}$; hence, it can be removed from the objective without affecting the solution. Moreover, each term in the innermost summation in \eqref{eq:log_likelihood_obj_decomp} only depends on a single map cell probability $p_n(\cdot)$. Therefore, the maximization of the objective can be carried out separately for each cell $m$:
\begin{equation}\label{eq:log_likelihood_obj_decomp_single}
    \max_{p_n \in \calP_{\calC}} \sum_{i \in \calV} \bbE_{m \sim p_n} \big[\log{\frac{q^i_t(m)}{p_n(m)}}\big],
\end{equation}
where $\calP_{\calC}$ is the space of categorical distributions over $\calC$ and
\begin{equation}\label{eq:qit}
    q^i_t(m) = \prod_{\tau=1}^t q^i(m | \bfz_{\tau}^i)^{1/t}.    
\end{equation}
In order to remove the constraint $p_n \in \calP_{\calC}$, we utilize a multi-class log-odds ratio representation of the categorical distribution \cite{vaskasi_TRO}:
\begin{equation}\label{eq:logodds}
    \bfh_n := \begin{bmatrix} \log \frac{p_n(m = 0)}{p_n(m = 0)} & \cdots & \log \frac{p_n(m = C)}{p_n(m = 0)} \end{bmatrix}^\top \in \bbR^{C+1}.
\end{equation}
A PMF and its log-odds representation have a one-to-one correspondence through the softmax function $\sigma:\mathbb{R}^{C+1} \rightarrow \mathbb{R}^{C+1}$:
\begin{equation}
    p_n(m = c) = \sigma_{c+1}(\bfh_n) := \frac{\bfe_{c+1}^\top \exp(\bfh_n)}{\mathbf{1}^\top \exp(\bfh_n)},\nonumber
\end{equation}
where $\bfe_c$ is the standard basis vector with $c$-th element equal to $1$ and $0$ elsewhere, $\mathbf{1}$ is the vector with all elements equal to $1$, and $\exp(\cdot)$ is applied element-wise to the vector $\bfh_n$. In order to enable distributed optimization of the objective \eqref{eq:log_likelihood_obj_decomp_single} via the framework of Sec.~\ref{sec:dist_Riemannian_opt}, we introduce a constraint that requires the robots to agree on a common map estimate using only one-hop communication.

\begin{problem}
    Let $\calG(\calV, \calE)$ be a network of robots, where each robot $i \in \calV$ collects semantic point cloud observations $\bfz_t^i$.
    Construct local estimates of the map log-odds $\bfh^i$ at each robot $i$ that are consistent among the robots via the following optimization:
    \begin{equation}
    \begin{aligned}  
        &\max_{\bfh^{1:|\calV|} \in \bbR^{(C+1) \times |\calV|}} \sum_{i \in \calV} f^i(\bfh^i),\\
        \text{s.t.} \quad &\phi(\bfh^{1:|\calV|}) = \sum_{\{i,j\} \in \calE} A_{ij} \|\bfh^{j} - \bfh^{i}\|_2^2 = 0,
    \label{eq:dist_mapping_prob}
    \end{aligned}
    \end{equation}
    where $\displaystyle{f^i(\bfh^i) = \sum_{c \in \calC} \sigma_{c+1}(\bfh^i) \log{\frac{q^i_t(c)}{\sigma_{c+1}(\bfh^i)}}}$ and $q^i_t(c)$ is defined in \eqref{eq:qit}.
\end{problem}

The multi-robot mapping problem in \eqref{eq:dist_mapping_prob} has the same structure as the general distributed optimization in \eqref{eq:dist_opt}. Therefore, the distributed Riemannian optimization algorithm (Alg.~\ref{alg:dist_opt}) can be employed to perform multi-robot semantic mapping. Note that $\phi(\cdot)$ is globally convex because of the flatness of Euclidean space. Thus, Theorem~\ref{thm:dist_opt_cons_opt} guarantees that Alg.~\ref{alg:dist_opt} can achieves consensus in the map estimates of all robots. The application of Alg.~\ref{alg:dist_opt} to solve \eqref{eq:dist_mapping_prob} in a distributed manner is presented in Alg.~\ref{alg:dist_mapping}.
The update step in line~\ref{alg_line:consensus_map} guides the local log-odds towards satisfaction of the consensus constraint, which only requires single-hop communication between neighboring robots $j \in \calN_i$. Line~\ref{alg_line:grad_comp_map} incorporates the local observations via $\bfgamma^i$ and $\bfbeta^i$, where $\odot$ is element-wise multiplication. This step is local to each robot $i$ and does not require communication. Note that lines \ref{alg_line:consensus_map} and \ref{alg_line:grad_apply_map} resemble the log-odds equivalent of Bayes rule for updating multi-class probabilities (see (8) in \cite{vaskasi_TRO}).

\begin{algorithm}[t]
\caption{Distributed Semantic Mapping}
\begin{algorithmic}[1]
\renewcommand{\algorithmicrequire}{\textbf{Input:}}
\renewcommand{\algorithmicensure}{\textbf{Output:}}
\Require Local observations $\bfz_{1:t}^i$ and initial multi-class map estimate ${\bfh^i}^{(0)}$
\Ensure Globally consistent semantic map
\For{$k \in \bbZ_{\geq 0}$}
    \For{each cell in $\bfm$}
        \LineComment{Promote consensus with step size $\epsilon_m$:}
        \State ${\Tilde{\bfh}^i}^{(k)} = {\bfh^i}^{(k)} + \epsilon_m \sum_{j \in \calN_i} A_{ij} ({\bfh^j}^{(k)} - {\bfh^i}^{(k)})$ \label{alg_line:consensus_map}
        \LineComment{Local gradient computation:}
        \State $\boldsymbol{\Delta}^i = {\Tilde{\bfh}^i}^{(k)} - \log{\bfq^i_t}$ \Comment{$\log{\bfq^i_t} = [\log{q^i_t(c)}]_{c=0}^{C}$}\label{alg_line:update_obs_start}
        \State $\bfgamma^i = (\exp({\Tilde{\bfh}^i}^{(k)})^{\top} \boldsymbol{\Delta}^i) \mathbf{1}$
        \State $\bfbeta^i = (\exp({\Tilde{\bfh}^i}^{(k)})^{\top} \mathbf{1}) \boldsymbol{\Delta}^i$
        \State $\bfg^i = (\bfgamma^i - \bfbeta^i) \odot \frac{\exp({\Tilde{\bfh}^i}^{(k)})}{(\exp({\Tilde{\bfh}^i}^{(k)})^{\top} \mathbf{1})^2}$ \label{alg_line:grad_comp_map}
        \LineComment{Apply gradient with step size $\alpha_m^{(k)}$:}
        \State ${\bfh^i}^{(k+1)} = {\Tilde{\bfh}^i}^{(k)} + \alpha_m^{(k)} \bfg^i$\label{alg_line:grad_apply_map}
        \State ${h^i_1}^{(k+1)} = 0$ \Comment{${h^i_1}^{(k+1)} = \log\frac{p(m=0)}{p(m=0)} = 0$}\label{alg_line:update_obs_end}
    \EndFor
\EndFor
\State \Return ${\bfh^i}^{(k)}$
\end{algorithmic}
\label{alg:dist_mapping}
\end{algorithm}

The distributed semantic mapping algorithm we developed assumes a regular grid representation of the environment. To reduce the storage and communication requirements, we may utilize a semantic octree data structure which provides a lossless compression of the original 3-D multi-class map. In this case, the update rules in Alg.~\ref{alg:dist_mapping} should be applied to all leaf nodes in the semantic octree map of each robot $i$. Refer to Alg.~3 in \cite{vaskasi_TRO} for the semantic octree equivalents of the update steps in lines \ref{alg_line:consensus_map} and \ref{alg_line:grad_apply_map}.

In this section, we presented the mapping component of ROAM as distributed construction of semantic octree maps given local semantic point cloud observations at each robot. In the next section, we introduce the multi-robot planning component of ROAM, where robots cooperatively find trajectories along which their observations are maximally informative. Employing ROAM for simultaneous distributed mapping and planning closes the loop for autonomously exploring an unknown environment with a team of robots.

\section{Multi-Robot Planning}
\label{sec:dist_planning}

We discussed the case where observations are collected passively along the robot trajectories and used for distributed mapping. In this section, we consider planning the motion of the robots to collect observations that reduce map uncertainty and uncover an unknown environment. This active mapping process prevents redundant observations that may not improve the map accuracy or increase the overall covered area.

Let $\bfX^i_t \in \SE$ be the pose of robot $i \in \calV$, at time $t$:
\begin{equation*}
    \bfX^i_t = \begin{bmatrix} \bfR^i_t & \bfp^i_t \\ \mathbf{0}^\top & 1 \end{bmatrix},
\end{equation*}
where $\bfR^i_t \in \textit{SO}(3)$ and $\bfp^i_t \in \bbR^3$ are the robot's orientation and position, respectively. The Lie algebra $\se$ corresponding to the Lie group $\SE$ is defined as follows:
\begin{equation}
    \se = \Big\{\hat{\bfxi} := \begin{bmatrix} \hat{\bftheta} & \bfrho \\ \mathbf{0}^\top & 0 \end{bmatrix} \in \bbR^{4 \times 4} \Big|\; \bfxi = \begin{bmatrix} \bfrho\\\bftheta \end{bmatrix} \in \bbR^6\Big\},\nonumber
\end{equation}
with $\hat{(\cdot)}$ used to denote the mapping from a vector $\bfxi \in \mathbb{R}^6$ to a $4 \times 4$ twist matrix in $\se$. The matrix exponential $\exp(\cdot): \se \rightarrow \SE$ relates a twist in $\se$ to a pose in $\SE$ via the Rodrigues' formula:
\begin{equation}
    \exp(\hat{\bfxi}) = I + \hat{\bfxi} + \frac{(1 - \cos{\|\theta\|})}{\|\theta\|^2} {\hat{\bfxi}}^2 + \frac{(\|\theta\| - \sin{\|\theta\|})}{\|\theta\|^3} {\hat{\bfxi}}^3.\nonumber
\end{equation}
The exponential mapping at an arbitrary pose $\bfX \in \SE$ with perturbation $\bfxi \in \bbR^6$ (in the robot frame) can be expressed as:
\begin{equation}
\Exp_{\bfX}{(\bfxi)} = \bfX \exp(\hat{\bfxi}).\nonumber
\end{equation}
The distance between two poses $\bfX^i_t$ and $\bfX^j_{t'}$ is defined as:
\begin{equation} 
    d^2(\bfX^i_t, \bfX^j_{t'}) = \bfxi^{\top}_{\bfX^i_t, \bfX^j_{t'}} \Gamma \bfxi_{\bfX^i_t, \bfX^j_{t'}}, \;\; \bfxi_{\bfX^i_t, \bfX^j_{t'}}\! =\! \log(\bfX^{i^{-1}}_t \bfX^j_{t'})^{\!\vee},\nonumber
\end{equation}
where the functions $\log(\cdot): \SE \rightarrow \se$ and $(\cdot)^\vee: \se \rightarrow \bbR^6$ denote the inverse mappings associated with $\exp(\cdot)$ and $\hat{(\cdot)}$, respectively. Also, $\Gamma \in \bbR^{6 \times 6}$ is a diagonal matrix with positive diagonal entries that account for the difference in scale between the linear and angular elements of $\bfxi_{\bfX^i_t, \bfX^j_{t'}}$. For more details, please refer to \cite[Ch.7]{BarfootBook}.

To enable gradient-based pose trajectory optimization, we introduce differentiable cost functions to quantify the safety and the informativeness of a pose trajectory. We use a distance field $D(\bfX^i_t, p^i_t(\bfm))$ as a measure of path safety derived from the map $p^i_t(\bfm)$ of robot $i$ given observations up to time $t$. To obtain the distance field, we extract a maximum likelihood occupancy map from $p^i_t(\bfm)$ and compute the distance transform. Furthermore, we use the Shannon mutual information $I(\bfm; \bfz | \bfX^i_t, p^i_t(\bfm))$~\cite[Eq.(4)]{vaskasi_TRO} to quantify the informativeness of an observation $\bfz$ made from pose $\bfX^i_t$ with respect to the current map $p^i_t(\bfm)$ of robot $i$.
In the case of semantic octree mapping with a range sensor, mutual information is not differentiable with respect to the pose $\bfX^i_t$. As a solution, we use the approach in \cite{vaskasi_iros22} to obtain a differentiable approximation of mutual information by interpolating its values at several nearby poses $\bfV \in \SE$. Specifically, the collision and informativeness score of a pose $\bfX_t^i$ is expressed as a convex combination of poses $\bfV$ on a grid $\calX(\bfX_t^i)$ inside a geodesic ball centered around $\bfX_t^i$ with radius $\xi_{\text{max}}$:
\begin{equation}
\begin{aligned}
    \mathfrak{f}(\bfX_t^i, p^i_t(\bfm)) &= \sum_{\bfV \in \calX(\bfX_t^i)} \lambda_{\bfV}(\bfX_t^i) s^i(\bfV),\\
    s^i(\bfV) &= I(\bfm; \bfz | \bfV, p^i_t(\bfm)) + \gamma_c \log{D(\bfV, p^i_t(\bfm))},\nonumber
\end{aligned}
\end{equation}
%
where the safety constant $\gamma_c > 0$ trades off informativeness with collision avoidance and the convex combination coefficients $\lambda_{\bfV}(\bfX_t^i)$ adjust the influence of the terms corresponding to $\bfV$ based on distance to $\bfX_t^i$:
\begin{equation}
\begin{aligned}
    \lambda_{\bfV}(\bfX_t^i) &= \frac{1 + \cos(\Bar{d}(\bfX_t^i, \bfV))}{\sum_{\bfU \in \calX(\bfX_t^i)} (1 + \cos(\Bar{d}(\bfX_t^i, \bfU)))},\\
    \Bar{d}(\bfX_t^i, \bfV) &= \frac{\pi}{\xi_{\text{max}}} d(\bfX_t^i, \bfV).\nonumber
\end{aligned}
\end{equation}
Fig.~\ref{fig:diff_info} illustrates the collision and informativeness score $\mathfrak{f}$. 

\begin{figure}[t]
    \centering
    \includegraphics[width=0.9\linewidth]{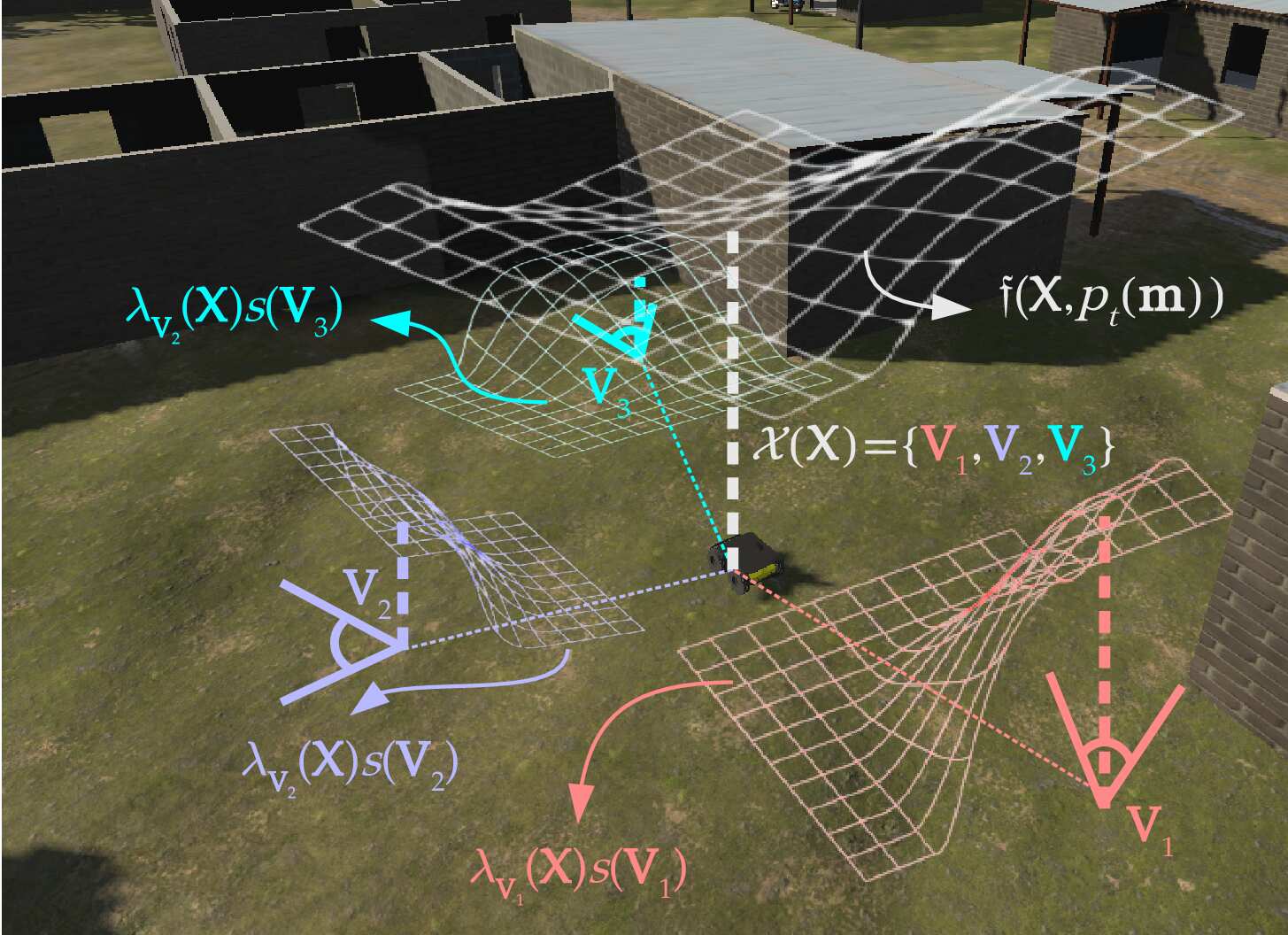}
    \caption{Collision and informativeness score for robot pose $\bfX$. Each sampled viewpoint $\bfV_l \in \calX(\bfX)$
    is colored differently. For each viewpoint, the field of view and the distance from the nearest obstacle determine the Shannon mutual information $I(\bfm; \bfz \mid \bfV_l, p_t(\bfm))$ and the log-distance $\log{D(\bfV_l, p_t(\bfm))}$, respectively. The weight $\lambda_{\bfV_l}(\bfX)$ dictates the contribution of $\bfV_l$ to the total score function $\mathfrak{f}(\bfX, p_t(\bfm))$, colored white.}
    \label{fig:diff_info}
\end{figure}

Cooperative planning requires the robots to take into account the plans of their peers in order to avoid actions that provide redundant information. Let $\mathfrak{X} = [\bfX^1_{t+1:t+T}, \ldots, \bfX^{|\calV|}_{t+1:t+T}]^{\top} \in \SE^{|\calV| \times T}$ be the concatenated $T$-length trajectories of all robots in $\calV$, where $T$ is the planning horizon. In the remainder of this section, we use $\mathfrak{X}_{i, \tau}$ as an alternative notation for $\bfX_{t+\tau}^i$, namely the $\SE$ pose of robot $i$ at time $t+\tau$. The function $\mathfrak{q}(\mathfrak{X}_{i,\tau}, \mathfrak{X}_{j,\tau'})$ quantifies the observation redundancy as the overlap between sensor field of views (FoVs) for two poses $\mathfrak{X}_{i,\tau}$ and $\mathfrak{X}_{j,\tau'}$:
\begin{equation}
    \mathfrak{q}(\mathfrak{X}_{i,\tau}, \mathfrak{X}_{j,\tau'}) = \max\crl{0, 2 d_{\mathfrak{q}} - \|\bfQ(\mathfrak{X}_{i,\tau} - \mathfrak{X}_{j,\tau'})\bfe\|_2}^2,\nonumber
\end{equation}
where:
\begin{equation}
    d_{\mathfrak{q}} = |\calF| + \xi_{\text{max}},\quad \bfQ = \begin{bmatrix} \bfI_{3 \times 3} & \boldsymbol{0}_{3 \times 1} \\ \boldsymbol{0}_{1 \times 3} & 0 \end{bmatrix},\quad \bfe = \begin{bmatrix} \boldsymbol{0}_{3 \times 1} \\ 1 \end{bmatrix},\nonumber
\end{equation}
and $|\calF|$ is the diameter of the sensor FoV. 

The local objective function for robot $i$ is defined using the collision and informativeness score $\mathfrak{f}$ and the FoV overlap $\mathfrak{q}$:
\begin{equation}\label{eq:traj_score}
\begin{aligned}
    f^i(&\mathfrak{X}, p^i_t(\bfm)) = \sum_{\tau = 1}^T \Big[\mathfrak{f}(\mathfrak{X}_{i,\tau}, p^i_t(\bfm))\\
    &- \gamma_{\mathfrak{q}} \sum_{j \in \calV} \sum_{\tau'=1}^T [1 - \delta_{ij}\delta_{\tau\tau'}][1 - \frac{\delta_{ij}}{2}] \mathfrak{q}(\mathfrak{X}_{i,\tau}, \mathfrak{X}_{j,\tau'})\Big],
\end{aligned}
\end{equation}
where $\delta_{ij}$ is the Kronecker delta which takes value $1$ if and only if $i = j$, and $0$ otherwise. Also, the constant $\gamma_{\mathfrak{q}} > 0$ trades off trajectory collision avoidance and informativeness with sensor FoV overlap.

The goal of multi-robot planning is to maximize the sum of local objective functions $f^i$ over $\calV$. Since we intend to perform the maximization in a distributed manner, we consider \emph{local plans} $\mathfrak{X}^i \in \SE^{|\calV| \times T}$ for each robot $i \in \calV$, representing an individual robot's plan for the collective trajectories of the team. Eventually, these local plans should reach consensus so that the team members act in agreement. To quantify the disagreement among the robot plans, we define an aggregate distance function $\phi(\cdot): \SE^{|\calV| \times T \times |\calV|} \rightarrow \bbR_{\geq 0}$ that accumulates the pairwise distances between all local plans $\mathfrak{X}^i$, $i \in \calV$:
\begin{equation}\label{eq:traj_constraint}
    \phi(\mathfrak{X}^{1:|\calV|}) = \sum_{\{i,j\} \in \calE} A_{ij} d^2(\mathfrak{X}^i, \mathfrak{X}^j),
\end{equation}
where $d: \SE^{|\calV| \times T \times 2} \rightarrow \bbR_{\geq 0}$ is defined via extension of the distance in $\SE$ to the product manifold $\SE^{|\calV| \times T}$.

\begin{problem}\label{problem:planning}
    Let $\calG(\calV, \calE)$ be a network of robots, where each robot $i \in \calV$ maintains a local map $p^i_t(\bfm)$ obtained by solving \eqref{eq:dist_mapping_prob}. Determine $\SE$ pose trajectories for all robots that maximize the cost function in \eqref{eq:traj_score} subject to the consensus constraint in \eqref{eq:traj_constraint}:
    \begin{equation}
    \begin{aligned}  
        \max_{\mathfrak{X}^{1:|\calV|} \in \SE^{|\calV| \times T \times |\calV|}} &\sum_{i \in \calV} f^i(\mathfrak{X}^i, p^i_t(\bfm)),\\
        \text{s.t.} \quad\qquad &\phi(\mathfrak{X}^{1:|\calV|}) = 0.
    \label{eq:dist_plan_prob}
    \end{aligned}
    \end{equation}
\end{problem}

The structure of the planning problem in \eqref{eq:dist_plan_prob} is compatible with the distributed Riemannian optimization method of Sec.~\ref{sec:dist_Riemannian_opt}. We formulate a version of Alg.~\ref{alg:dist_opt} specialized for the $\SE$ manifold. Due to the positive curvature of the $\SE$ manifold, the aggregate distance function $\phi(\cdot)$ has local minima (see Appendix A.3 in \cite{tedrake}). Thus, if the initial trajectories ${\mathfrak{X}^i}^{(0)}$, $i \in \calV$, are not similar, the algorithm may converge to a local optimum of the consensus constraint \eqref{eq:traj_constraint}. Furthermore, the local objective functions $f^i$, $i \in \calV$, are only locally concave \cite{julian}. Hence, Theorem~\ref{thm:dist_opt_cons_opt} guarantees only a locally optimal consensus solution.

\begin{algorithm}[t]
\caption{Distributed Planning for Exploration}
\begin{algorithmic}[1]
\renewcommand{\algorithmicrequire}{\textbf{Input:}}
\renewcommand{\algorithmicensure}{\textbf{Output:}}
\Require Local map $p_t^i(\bfm)$ of robot $i$
\Ensure Collaborative robot team plan for exploration
\State $\mathfrak{X}^{i^{(0)}} = \Call{frontier}{p_t^i(\bfm)} \quad \forall i \in \calV$ \Comment{Initialization}
\For{$k \in \bbZ_{\geq 0}$}
    \LineComment{Promote consensus with step size $\epsilon_p$:}
    \For{every $l \in \calV$ and $\tau \in \crl{1, \ldots, T}$}
        \State\vspace*{-1.4\baselineskip}
        \begin{fleqn}[\dimexpr\leftmargini-1.0\labelsep]
        \begin{equation}
        \begin{aligned}
            \Tilde{\mathfrak{X}}^{i^{(k)}}_{l, \tau} = \mathfrak{X}^{i^{(k)}}_{l, \tau} &\exp\Big(\epsilon_p \sum_{j \in \calN_i} A_{ij}\\
            &\times \big(J_L^{-\top}\!(\bfxi_{\mathfrak{X}_{l,\tau}^{i^{(k)}}, \mathfrak{X}_{l,\tau}^{j^{(k)}}}) \Gamma \bfxi_{\mathfrak{X}_{l,\tau}^{i^{(k)}}, \mathfrak{X}_{l,\tau}^{j^{(k)}}} \big)^{\wedge} \Big)\nonumber
        \end{aligned}
        \end{equation}
        \end{fleqn} \label{alg_line:consensus_plan}%
    \EndFor
    \LineComment{Local gradient computation:}
    \State $\bfg_{l, \tau} = 0$ \Comment{Initialize for all $l \in \calV$, $\tau \in \crl{1, \ldots, T}$}\label{alg_line:grad_init_plan}
    \For{every $\tau' \in \crl{1, \ldots, T}$}
        \State $c_{\text{set}} = |\calX(\Tilde{\mathfrak{X}}^{i^{(k)}}_{i,\tau'})| + \sum_{\bfV \in \calX(\Tilde{\mathfrak{X}}^{i^{(k)}}_{i,\tau'})} \cos{(\Bar{d}(\Tilde{\mathfrak{X}}^{i^{(k)}}_{i, \tau'}, \bfV))}$\label{alg_line:begin_info_grad_plan}
        \State\vspace*{-1.3\baselineskip}
        \begin{fleqn}[\dimexpr\leftmargini-1.0\labelsep]
        \begin{equation}
        \begin{aligned}
            \bfg_{i, \tau'} &\mathrel{+}= \sum_{\bfV \in \calX(\Tilde{\mathfrak{X}}^{i^{(k)}}_{i,\tau'})} [(s(\bfX) - \mathfrak{f}(\Tilde{\mathfrak{X}}^{i^{(k)}}_{i, \tau'}, p^i_t(\bfm)))\\
            &\times \frac{\sin{(\Bar{d}(\Tilde{\mathfrak{X}}^{i^{(k)}}_{i, \tau'}, \bfV))}}{c_{\text{set}} \Bar{d}(\Tilde{\mathfrak{X}}^{i^{(k)}}_{i, \tau'}, \bfV)} J_L^{-\top}\!(\bfxi_{\Tilde{\mathfrak{X}}^{i^{(k)}}_{i, \tau'}, \bfV}) \Gamma \bfxi_{\Tilde{\mathfrak{X}}^{i^{(k)}}_{i, \tau'}, \bfV} ]\nonumber
        \end{aligned}
        \end{equation}
        \end{fleqn}\label{alg_lin:info_grad_plan}%
        \For{every $l \in \calV$ and $\tau \in \crl{1, \ldots, T}$}
            \State $\bfp_{i, \tau'} = \bfQ \Tilde{\mathfrak{X}}^{i^{(k)}}_{i, \tau'} \bfe$,\quad $\bfp_{l, \tau} = \bfQ \Tilde{\mathfrak{X}}^{i^{(k)}}_{l, \tau} \bfe$\label{alg_line:begin_overlap_grad_plan}
            \State $\bfR_{i, \tau'} = \bfQ \Tilde{\mathfrak{X}}^{i^{(k)}}_{i, \tau'} \bfE$,\quad $\bfR_{l, \tau} = \bfQ \Tilde{\mathfrak{X}}^{i^{(k)}}_{l, \tau} \bfE$
            \State $\bfc_{\text{disp}} = [1 - \delta_{il}\delta_{\tau\tau'}][1 - \frac{\delta_{il}}{2}] (\bfp_{i, \tau'} - \bfp_{l, \tau})$
            \State $\bfc_{\text{tot}} = \gamma_{\mathfrak{q}} \bfc_{\text{disp}} \max\crl{0, \frac{2 d_{\mathfrak{q}}}{\|\bfp_{i, \tau'} - \bfp_{l, \tau}\|_2} - 1}$
            \State $\bfg_{i, \tau'} \mathrel{+}= \begin{bmatrix} {\bfR^{\top}_{i, \tau'}} \bfc_{\text{tot}}  \\ \boldsymbol{0}_{3 \times 1} \end{bmatrix}$,\quad $\bfg_{l, \tau} \mathrel{-}= \begin{bmatrix} {\bfR^{\top}_{l, \tau}} \bfc_{\text{tot}} \\ \boldsymbol{0}_{3 \times 1} \end{bmatrix}$\label{alg_lin:overlap_grad_plan}
        \EndFor
    \EndFor
    \LineComment{Apply gradient with step size $\alpha_p^{(k)}$}
    \For{every $l \in \calV$ and $\tau \in \crl{1, \ldots, T}$}
        \State $\mathfrak{X}^{i^{(k+1)}}_{l, \tau} = \Tilde{\mathfrak{X}}^{i^{(k)}}_{l, \tau} \exp( \alpha_p^{(k)} \Hat{\bfg}_{l, \tau})$\label{alg_line:grad_apply_plan}
    \EndFor
\EndFor
\State \Return $\mathfrak{X}^{i^{(k)}}$
\end{algorithmic}
\label{alg:dist_planning}
\end{algorithm}

Our distributed planning algorithm for solving \eqref{eq:dist_plan_prob} is presented in Alg.~\ref{alg:dist_planning}.
Given its current local map $p_t^i(\bfm)$, each robot $i \in \calV$ computes an initial plan $\mathfrak{X}^{i^{(0)}}$ for the whole team using frontier-based exploration \cite{frontier}.
In line~\ref{alg_line:consensus_plan}, each pose in the local plan $\mathfrak{X}^{i^{(k)}}$ is guided towards consensus with the plans of neighboring robots $j \in \calN_i$. The update in this line is carried out via a perturbation in the robot frame, where $J_L(\cdot)$ denotes the left Jacobian of $\SE$~\cite[Ch.7]{BarfootBook}, and only involves communication between neighbors. To compute the local objective function gradients with respect to each pose in the local plan, we first initialize the gradients with zero in line~\ref{alg_line:grad_init_plan}, and then populate them with proper values in lines~\ref{alg_lin:info_grad_plan} and \ref{alg_lin:overlap_grad_plan}. The gradient of the collision and informativeness score $\mathfrak{f}(\Tilde{\mathfrak{X}}^{i^{(k)}}_{i, \tau'}, p^i_t(\bfm))$ is computed in lines~\ref{alg_line:begin_info_grad_plan}-\ref{alg_lin:info_grad_plan}, while lines~\ref{alg_line:begin_overlap_grad_plan}-\ref{alg_lin:overlap_grad_plan} derive the gradient of the sensor overlap $\mathfrak{q}(\Tilde{\mathfrak{X}}^{i^{(k)}}_{i, \tau'}, \Tilde{\mathfrak{X}}^{i^{(k)}}_{l, \tau})$ with respect to both inputs, where $\bfE = [\bfI_{3 \times 3} \; \boldsymbol{0}_{3 \times 1}]^{\top}$. Note that, for the $\mathfrak{f}$ terms, we only need to compute the gradient with respect to robot $i$'s own trajectory $\Tilde{\mathfrak{X}}^{i^{(k)}}_{i, \tau'}$, $\tau' \in \crl{1, \ldots, T}$, whereas for the $\mathfrak{q}$ terms, robot $i$ should locally obtain gradients with respect to both its own trajectory as well as the trajectories of all other robots in $\calV$. Since each robot stores the trajectory of the whole team, the computation for the gradients of the $\mathfrak{q}$ terms does not require any communication among the robots. Lastly, in line~\ref{alg_line:grad_apply_plan}, we apply the computed gradients to each pose in the local plan, using a right perturbation in the robot frame.

Solving \eqref{eq:dist_plan_prob} via Alg.~\ref{alg:dist_planning} leads to two types of behaviors.
\begin{enumerate}
    \item Locally, the robots attempt to maximize information and distance from obstacles along their trajectories. This encourages each robot to visit unvisited parts of the environment, and corresponds to the $\mathfrak{f}$ terms of \eqref{eq:traj_score}.
    
    \item Within each neighborhood, the robots \emph{negotiate} with their peers to minimize redundant observations. This prevents the trajectories to amass at certain regions of the map, and corresponds to the $\mathfrak{q}$ terms of \eqref{eq:traj_score}.
\end{enumerate}
We emphasize that each local plan $\mathfrak{X}^i$ stores the paths for all robots in $\calV$, instead of only robot $i$'s and its immediate neighbors. This is because storing all $|\calV|$ paths in each robot allows propagation of the mentioned behaviors on a global scale, due to the consensus constraint of \ref{eq:dist_plan_prob}. Therefore, the global solution of \eqref{eq:dist_plan_prob} corresponds to a Pareto optimum where agents find an optimal trade-off between their own information and safety maximization on one hand and avoiding observation overlap with their peers on the other hand.

In this section, we developed the distributed planning component of ROAM. The robot trajectories are chosen to maximize information and safety for cooperative estimation of a semantic octree map. Combined with the distributed mapping method of Sec.~\ref{sec:dist_mapping}, the overall system can be utilized for efficient multi-robot exploration of an unknown environment. In the next section, we demonstrate the performance of ROAM in a variety of simulation and real-world experiments.

\section{Experiments}
\label{sec:exp}

\begin{figure}[t]
    \centering
    \includegraphics[width=\linewidth]{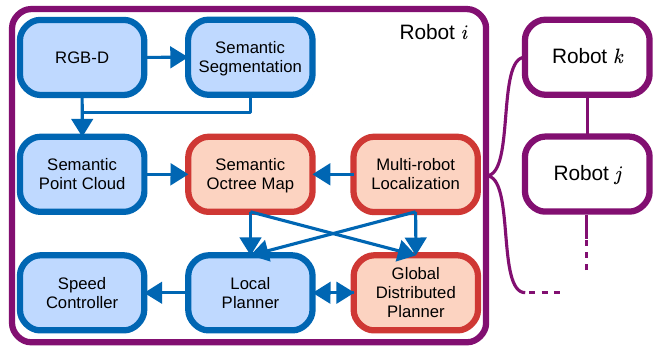}
    \caption{Software stack for multi-robot distributed active mapping. The blue blocks are local to each robot, whereas the red blocks require communication with neighboring robots. The communication links between pairs of robots are represented by violet lines.}
    \label{fig:system_arch}
\end{figure}

This section describes the implementation of ROAM on multi-robot systems. Next, we evaluate the performance of ROAM using several measures that quantify optimality, convergence to consensus, and communication overhead. The evaluations are done in both simulation and real-world.

\subsection{Implementation of ROAM for Distributed Active Mapping}
\label{subsec:exp_impl}


We deploy our approach on a team of ground wheeled robots, each equipped with an RGB-D sensor. Fig.~\ref{fig:system_arch} shows an overview of the software stack, implemented using the \textit{Robot Operating System} (ROS) \cite{ros}. The RGB-D sensor provides synchronized RGB and depth images. The RGB image is processed with a semantic segmentation algorithm to label each pixel with an object category. The segmented image is fused with the depth image to obtain a 3-D semantically annotated point cloud in the sensor frame of robot $i$.

\paragraph*{Multi-robot localization}
    It is required to perform multi-robot localization in order to find the transformation \begin{enumerate*}[label=\itshape\alph*\upshape)] \item from robot $i$'s sensor frame to a static \textit{world} frame $\calW_i$ for point cloud registration, and \item from $\calW_i$ to $\calW_j$ for distributed multi-robot mapping and planning\end{enumerate*}. Our implementation of multi-robot localization in the simulation and real-world experiments is explained in Sec.~\ref{subsec:exp_sim} and Sec.~\ref{subsec:exp_real}, respectively.

\paragraph*{Multi-robot mapping}
    The semantic point cloud is used to build and update a semantic octree map for each robot $i \in \calV$ via lines \ref{alg_line:update_obs_start}-\ref{alg_line:update_obs_end} of Alg.~\ref{alg:dist_mapping}. The semantic map of each robot $i$ is broadcasted to its neighboring robots $\calN_i$ once every $t^{\textit{pub}}_m$ seconds. Moreover, each robot pushes any newly received map to a local buffer memory, and performs line~\ref{alg_line:consensus_map} of Alg.~\ref{alg:dist_mapping} every $t^{\textit{int}}_m$ seconds to integrate neighbors' maps into its local map. The buffer is cleared after each successful iteration of Alg.~\ref{alg:dist_mapping}.

\paragraph*{Multi-robot viewpoint planning}
    To decouple low-frequency informative planning from high-frequency planning for collision-avoidance, we perform two separate planning stages, namely on global viewpoint level and on local trajectory level. On the viewpoint level, the distributed collaborative planning in Alg.~\ref{alg:dist_planning} is employed to find informative viewpoints for each robot in $\calV$. To coordinate viewpoint planning across all robots, every robot $i \in \calV$ maintains a ledger $\calL$ composed of $|\calV|$ binary values each indicating whether the corresponding robot in the team is ready for planning. Due to the decentralized nature of our method, each robot sends its own copy of the ledger $\calL_{i}$ to its neighbors every $t^{\textit{pub}}_{p}$ second, and updates $\calL_{i}$ using the incoming ledgers, as well as its status with respect to the current plan. Alg.~\ref{alg:dist_ledger} details the process of decentralized ledger synchronization for each robot $i$. In line~\ref{alg_line:copy_ledger} robot $i$ makes a copy of the incoming ledger $\calL_{\textit{inc}}$. Then, in line~\ref{alg_line:check_ready}, the function $\Call{checkReady()}{}$ determines whether or not the robot is ready to compute a new plan. A robot would declare ready to plan only when it has finished its previous plan and also it is currently not planning. The global distributed planning of Alg.~\ref{alg:dist_planning} would start only after a minimum fraction of robots, denoted by $\textit{thresh}_p$, are ready to plan. Line~\ref{alg_line:stable} is used to stabilize the ledger synchronization process. During the global distributed planning of Alg.~\ref{alg:dist_planning}, each robot $i$ broadcasts its local plan $\mathfrak{X}^i$ after each optimization iteration. Incoming local plans $\mathfrak{X}^j$, $j \in \calN_i$, are pushed to a local buffer memory to be used during the consensus step (line~\ref{alg_line:consensus_plan} of Alg.~\ref{alg:dist_planning}). The buffer is cleared after each optimization iteration. Lastly, the planning terminates after reaching $k_p$ iterations.

\begin{algorithm}[t]
\caption{Distributed Ledger Synchronization}
\begin{algorithmic}[1]
\renewcommand{\algorithmicrequire}{\textbf{Input:}}
\renewcommand{\algorithmicensure}{\textbf{Output:}}
\Require Incoming ledger $\calL_{\textit{inc}}$
\Ensure Synchronized ledger
\State $\calL_{i} = \calL_{\textit{inc}}$\label{alg_line:copy_ledger}
\If{$\Call{checkReady()}{}$}\label{alg_line:check_ready}
    \State $\calL_{i}[i] = 1$
    \If{$\Call{mean}{\calL_{i}} \geq \textit{thresh}_p$}
        \State $\Call{startPlanning()}{}$ \Comment{Viewpoint planning via Alg.~\ref{alg:dist_planning}}
    \EndIf
\Else
    \State $\calL_{i}[i] = \Call{isPlanning()}{}$\label{alg_line:stable}
\EndIf
\State \Return $\calL_{i}$
\end{algorithmic}
\label{alg:dist_ledger}
\end{algorithm}

\paragraph*{Local trajectory optimization}
    After computing a sequence of viewpoints $\mathfrak{X}^{1:|\calV|}$, each robot $i$ locally computes a trajectory to visit its portion of the viewpoints $\mathfrak{X}^i_{i, 1:T}$. The separation of the viewpoint planning from the trajectory optimization allows the robots to rapidly react to environment changes or mapping errors via local path re-planning, without the need to coordinate with their peers in viewpoint planning via Alg.~\ref{alg:dist_planning}. Furthermore, the two stage planning allows accounting for dynamical constraints of each robot in heterogeneous robot teams, such that the low-level trajectory optimizer takes the viewpoint set $\mathfrak{X}^i_{i, 1:T}$ and computes a dynamically feasible path. In our experiments, each robot $i$ projects its own semantic octree map onto a 2-D plane to obtain an occupancy grid map of the environment. Given its viewpoint set $\mathfrak{X}^i_{i, 1:T}$, the robot computes a sequence of collision-free positions and orientations that connect its current pose to $\mathfrak{X}^i_{i, 1}$, and each $\mathfrak{X}^i_{i, \tau}$ to $\mathfrak{X}^i_{i, \tau + 1}$ for $\tau \in \{1, \ldots, T - 1\}$. For this purpose, the trajectory optimizer uses $A^*$ graph search over the 2-D occupancy map. If a collision is detected during execution of the path, the corresponding path segment is re-planned using another $A^*$ call. The local trajectory is then used by a low-level speed controller to generate velocity commands.

An open-source implementation of ROAM is available on GitHub\footnote{\url{https://github.com/ExistentialRobotics/ROAM}.}. The rest of this section describes the simulation and real-world experiments. Table~\ref{tbl:exp_params} summarizes the parameters used across all experiments.

\bgroup
\def\arraystretch{1.25}
\begin{table}[t]
    \centering
    \caption{Parameter set for multi-robot exploration.}
    \small
    \label{tbl:exp_params}
    \begin{tabular}{|cc|cc|cc|}
        \hline
        \multicolumn{4}{|c|}{Planning} & \multicolumn{2}{c|}{Mapping}\\
        \hline
        $\epsilon_p$ & $0.1$ & $\alpha_p^{(k)}$ & $\frac{0.1}{k+1}$ & $\epsilon_m$ & $0.1$\\
        $d_{\mathfrak{q}}$ & $20$m & $\xi_{\text{max}}$ & $16$m & $\alpha_m^{(k)}$ & $\frac{1}{k+1}$\\
        $\gamma_c$ & $10^{-3}$ & $\gamma_{\mathfrak{q}}$ & $10^{-2}$ & $t^{\textit{pub}}_m$ & $5$\\
        $T$ & $5$ & $k_p$ & $20$ & $t^{\textit{int}}_m$ & $5$\\
        $t^{\textit{pub}}_{p}$ & $1$sec & $\textit{thresh}_p$ & $0.4$ &  Voxel  & \multirow{2}{*}{$0.2$m}\\\cline{1-4}
        \multicolumn{2}{|c}{$\Gamma$} & \multicolumn{2}{c|}{$\textit{diag}(1,1,1,0.1,0.1,0.1)$} & side length &\\
        \hline
    \end{tabular}
\end{table}
\egroup

\subsection{Simulation Experiments}
\label{subsec:exp_sim}

\begin{figure}[t]
    \centering
    \includegraphics[width=0.9\linewidth]{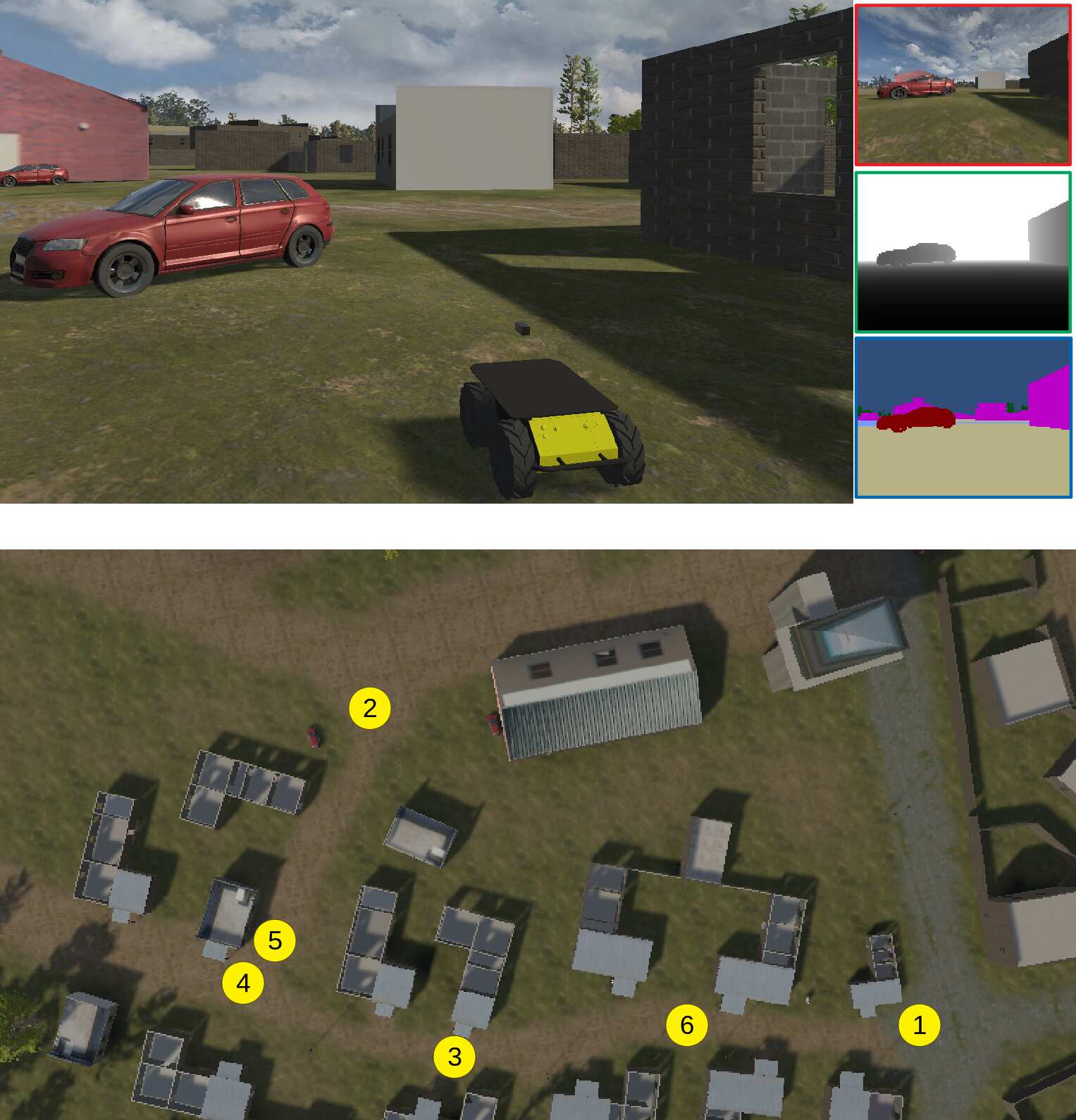}
    \caption{Simulation environment for multi-robot distributed active mapping. Top: A Husky robot receiving RGB, depth, semantic segmentation images. Bottom: A top-down view of the simulated environment, where the numbered circles show the starting positions of six robots.}
    \label{fig:sim_env}
\end{figure}

We carry out experiments in a photo-realistic 3-D simulation powered by the Unity engine~\cite{unity}. The environment resembles an outdoor village area with various types of terrain (e.g., grass, dirt road, asphalt, etc.) and object classes, such as buildings, cars, and street lighting. Our experiments utilize $|\calV| = 6$ ClearPath Husky wheeled robots, each equipped with an RGB-D sensor. We assume known robot poses and perfect semantic segmentation over the RGB input in the simulation experiments. Fig.~\ref{fig:sim_env} shows the simulation setup.

\begin{figure*}[t]
    \includegraphics[width=\linewidth]{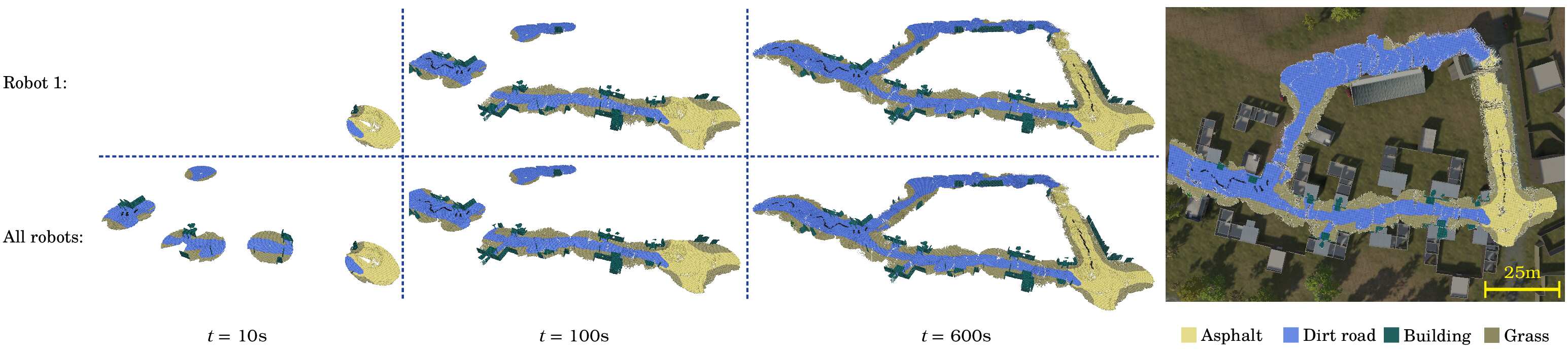}
    \caption{Time lapse of the multi-robot active mapping experiment. The local map of robot 1 (in Fig.~\ref{fig:sim_env}) is compared against the combined maps of all robots. The right sub-figure shows the estimated semantic octree map of robot 1 overlayed on the ground-truth simulation environment. The exploration is carried out using a fully-connected network of robots.}
    \label{fig:exp_timelapse}
\end{figure*}

Each robot uses its local semantic octree map to extract traversable regions, while other object and terrain classes are considered as obstacles. In particular, \textit{Asphalt} and \textit{Dirt road} classes are selected as traversable terrain classes. Fig.~\ref{fig:exp_timelapse} visualizes a time lapse of the distributed multi-robot active mapping experiment.
The consistency between the local map of robot 1 and the combined map of all robots can be seen as a qualitative example of the map consensus achieved by the distributed mapping method in Alg.~\ref{alg:dist_mapping}. Analogously, Fig.~\ref{fig:path_timelapse} illustrates consensus achieved by the distributed multi-robot planning in Alg.~\ref{alg:dist_planning}.
As described in Sec.~\ref{sec:dist_planning}, each robot computes its local plan based on its local map. Hence, differences in the local maps can cause variation across the local plans, as seen in Fig.~\ref{subfig:path_timelapse_a}. However, during each iteration of distributed planning, line~\ref{alg_line:consensus_plan} in Alg.\ref{alg:dist_planning} steers the local plans towards a consensus plan, as is evident in Fig.~\ref{subfig:path_timelapse_d}.

\begin{figure}[t]
    \begin{subfigure}[t]{0.5\linewidth}
    \centering
    \includegraphics[width=\linewidth]{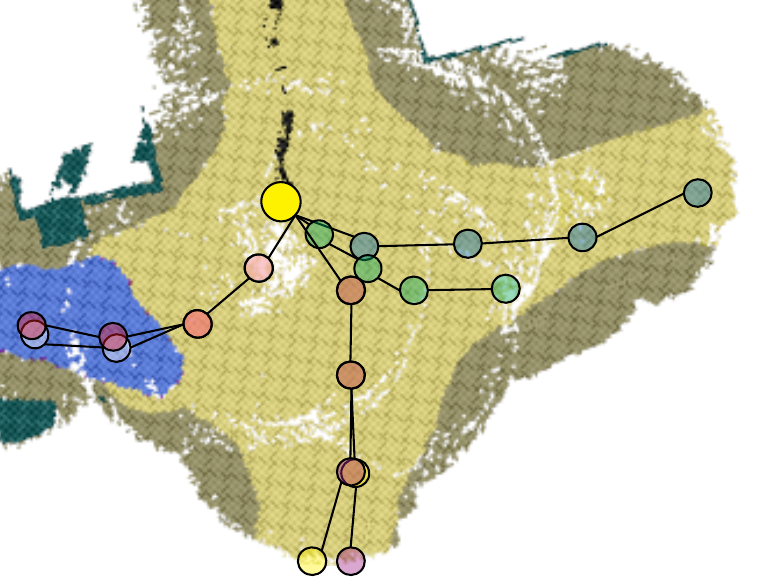}
    \captionsetup{justification=centering}
    \caption{Initial local plans}
    \label{subfig:path_timelapse_a}
    \end{subfigure}%
    \hfill%
    \begin{subfigure}[t]{0.5\linewidth}
    \centering
    \includegraphics[width=\linewidth]{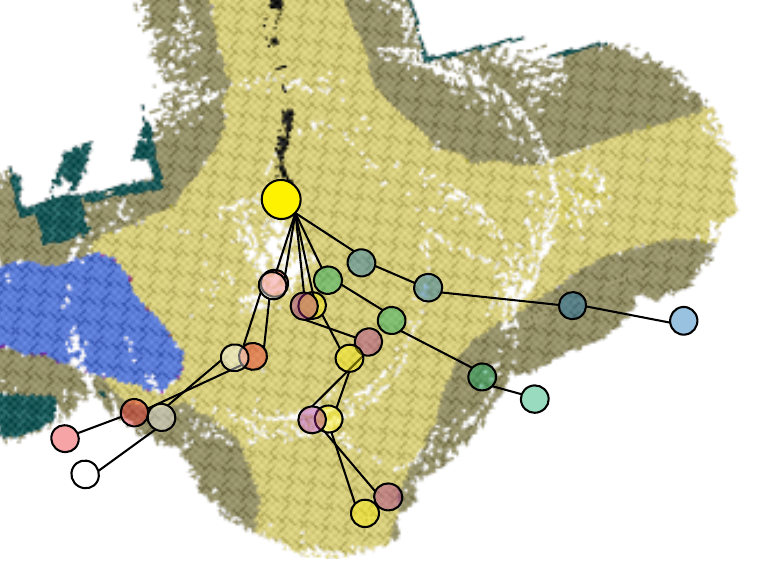}
    \captionsetup{justification=centering}
    \caption{$5^{\text{th}}$ iteration}
    \end{subfigure}\\
    \begin{subfigure}[t]{0.5\linewidth}
    \centering
    \includegraphics[width=\linewidth]{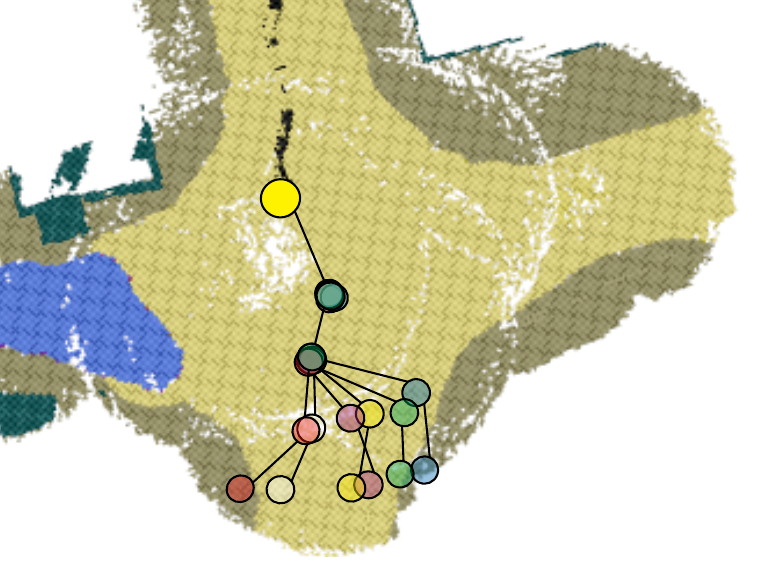}
    \captionsetup{justification=centering}
    \caption{$15^{\text{th}}$ iteration}
    \end{subfigure}%
    \hfill%
    \begin{subfigure}[t]{0.5\linewidth}
    \centering
    \includegraphics[width=\linewidth]{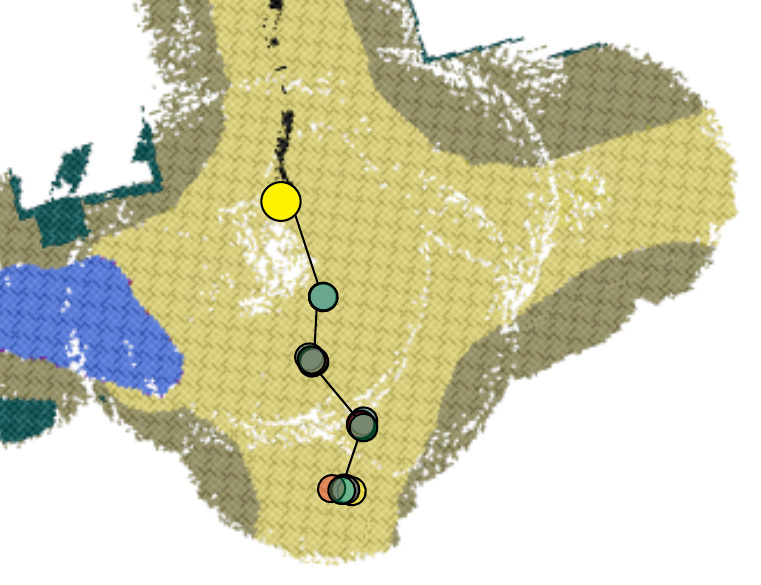}
    \captionsetup{justification=centering}
    \caption{$20^{\text{th}}$ iteration}
    \label{subfig:path_timelapse_d}
    \end{subfigure}
    \caption{Time lapse of viewpoint planning for robot 1 from Fig.~\ref{fig:sim_env}. Each color corresponds to a robot $j \in \calV$ computing a local plan $\mathfrak{X}^j_{1,1:T}$ for robot 1. The planned trajectories contain both position and orientation, however only the positions are visualized for clarity. The planning is carried out over a fully-connected network of robots.}
    \label{fig:path_timelapse}
\end{figure}


The performance of ROAM is evaluated quantitatively under various robot network configurations and planning parameters. We consider $3$ different network topologies: \begin{enumerate*} \item \textit{Full}, where all robots can communicate with each other in a fully-connected network, \item \textit{Hierarchical}, where robots can only communicate with their team leaders, and \item \textit{Ring}, where each robot has exactly $2$ neighbors\end{enumerate*}. Fig.~\ref{fig:network_confs} depicts the $3$ network configurations. For each network configuration, we perform exploration under $3$ variants of Alg.~\ref{alg:dist_planning}: \begin{enumerate*} \item \textit{Collaborative}, which is the original version of Alg.~\ref{alg:dist_planning}, \item \textit{Egocentric}, where each robot only maximizes its own path informativeness and safety by choosing $\epsilon_p = \gamma_{\mathfrak{q}} = 0$, and \item \textit{Frontier}, where robots perform frontier-based exploration by choosing $k_p = 0$\end{enumerate*}.

\begin{figure}[t]
    \centering
    \includegraphics[width=0.9\linewidth]{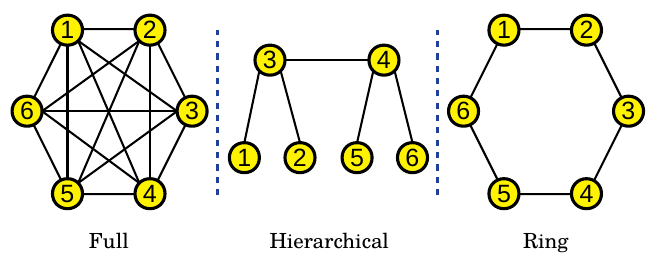}
    \caption{Network topologies used in the simulation experiments.}
    \label{fig:network_confs}
\end{figure}

Fig.~\ref{fig:cov_vs_time_dist} quantifies the coverage achieved by each network topology and planning parameter set. For \textit{Collaborative} and \textit{Egocentric} planning configurations, \textit{Full} network configuration leads to faster coverage while traveling less distance compared to \textit{Hierarchical} and \textit{Ring} topologies. This is expected since \textit{Full} is the only network topology that allows one-hop exchange of information between any pair of robots. On the other hand, the network configuration does not play a significant role for \textit{Frontier} exploration in terms of total covered area, since robots usually choose a frontier that is nearby their current position, and do not utilize information coming from their peers' local maps. The most interesting takeaway from Fig.~\ref{fig:cov_vs_time_dist} is the similar performance of \textit{Collaborative} planning with \textit{Hierarchical} and \textit{Ring} topologies, compared to \textit{Egocentric} planning with \textit{Full} topology. This observation suggests that effective coordination among agents via \textit{Collaborative} planning can alleviate the longer multi-hop communication routes caused by the sparse connectivity of \textit{Hierarchical} and \textit{Ring} topologies.

\begin{figure*}[t]
    \includegraphics[width=\linewidth]{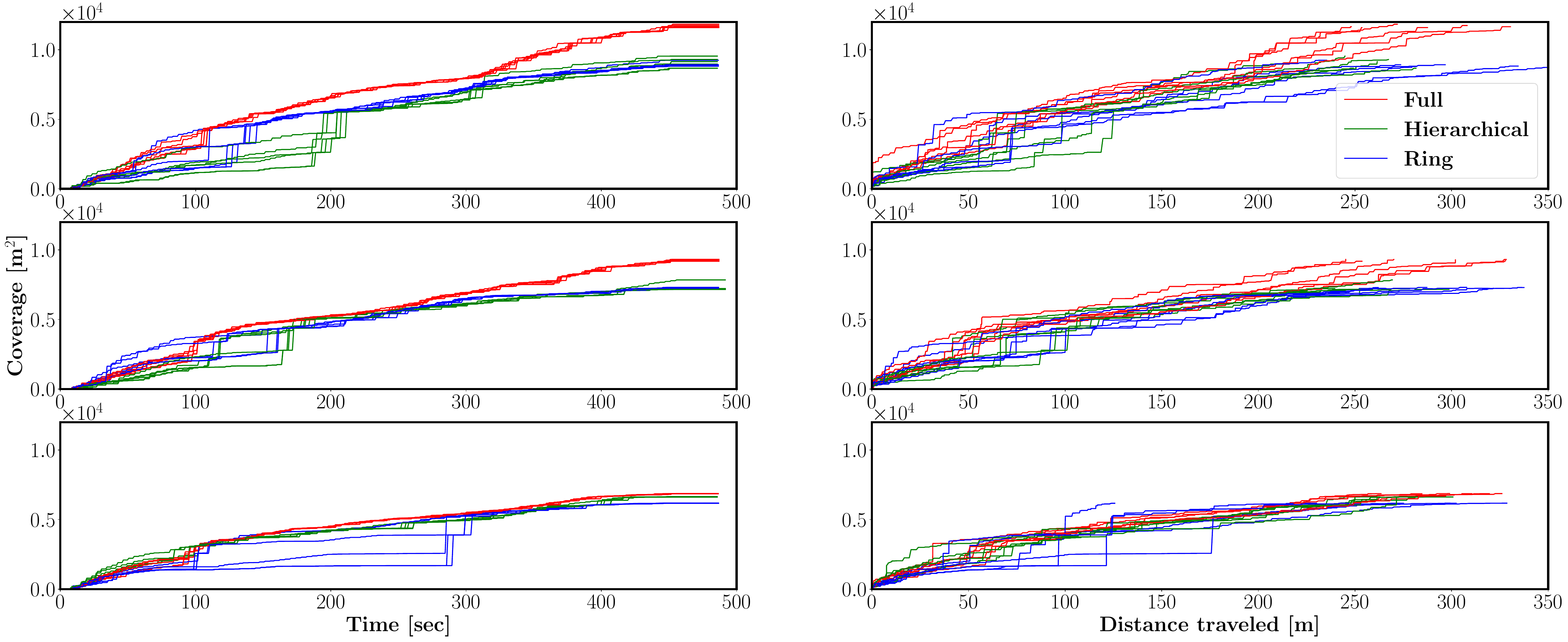}
    \caption{Coverage versus time (left column) and distance traveled (right column) for the simulation experiments. The top, middle, and bottom rows show the results for \textit{Collaborative}, \textit{Egocentric}, and \textit{Frontier} modes of planning, respectively. In each plot, lines with the same color correspond to robots participating in the same multi-robot exploration experiment, while the experiments are separated by the type of network topology in Fig.~\ref{fig:network_confs}.}
    \label{fig:cov_vs_time_dist}
\end{figure*}

\begin{figure*}[t]
    \includegraphics[width=\linewidth]{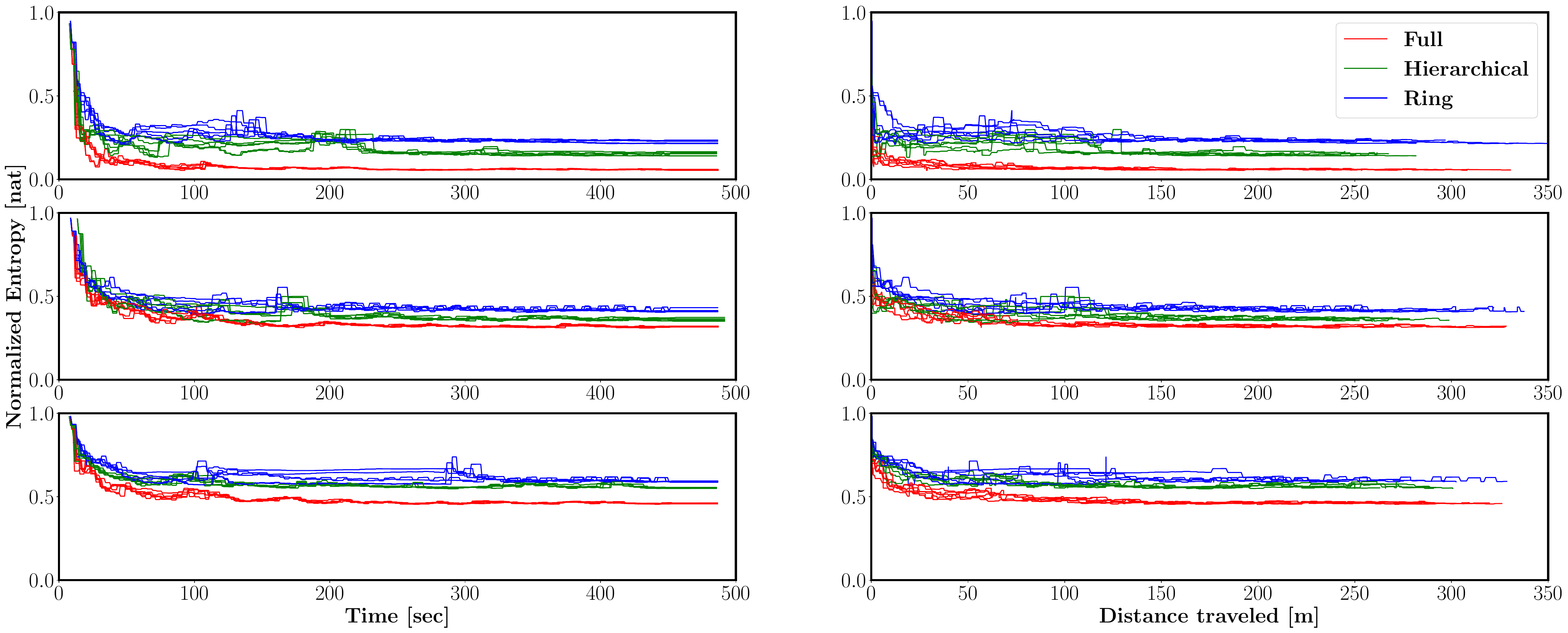}
    \caption{Normalized map entropy versus time (left column) and distance traveled (right column) for the simulation experiments. The top, middle, and bottom rows show the results for \textit{Collaborative}, \textit{Egocentric}, and \textit{Frontier} modes of planning, respectively. In each plot, lines with the same color correspond to robots participating in the same multi-robot exploration experiment, while the experiments are separated by the type of network topology in Fig.~\ref{fig:network_confs}.}
    \label{fig:ent_vs_time_dist}
\end{figure*}

Similar insights can be obtained from Fig.~\ref{fig:ent_vs_time_dist}, where normalized map entropy is measured against elapsed time and distance traveled, for each network topology and planning mode. Normalized map entropy for robot $i \in \calV$ is defined as the sum of Shannon entropies of all map voxels divided by the number of voxels:
\begin{equation}
    H^i_{\textit{norm}} = \frac{-1}{N^i} \sum_{n=1}^{N^i} \sum_{c \in \calC} p^i_n(m = c) \log{p^i_n(m = c)},\nonumber
\end{equation}
where $N^i$ denotes the number of voxels in the local map of robot $i$, and $\calC$ as well as $p^i_n(m)$ are defined in the previous sections. Note that, unlike total map entropy, normalized entropy can increase as the robots register unvisited voxels into their map. As Fig.~\ref{fig:ent_vs_time_dist} shows, for each planning mode, \textit{Full} network topology outperforms \textit{Hierarchical} and \textit{Ring} configurations. Also, \textit{Collaborative} planning with \textit{Hierarchical} and \textit{Ring} configurations have similar performance to \textit{Egocentric} planning with \textit{Full} network topology. The same reasoning used for Fig.~\ref{fig:cov_vs_time_dist} can be utilized to justify these observations. However, unlike coverage, network topology plays a more significant role in terms of normalized map entropy for \textit{Frontier} planning mode. This is due to the relatively more distributed mapping consensus steps for the \textit{Full} topology that lead to more certainty in the map estimation and, hence, smaller entropy compared to \textit{Hierarchical} and \textit{Ring}. Since coverage does not take map uncertainty into account, such behavior is only noticeable in the bottom row of Fig.~\ref{fig:ent_vs_time_dist} but not in Fig.~\ref{fig:cov_vs_time_dist}.

\begin{figure*}[t]
    \includegraphics[width=\linewidth]{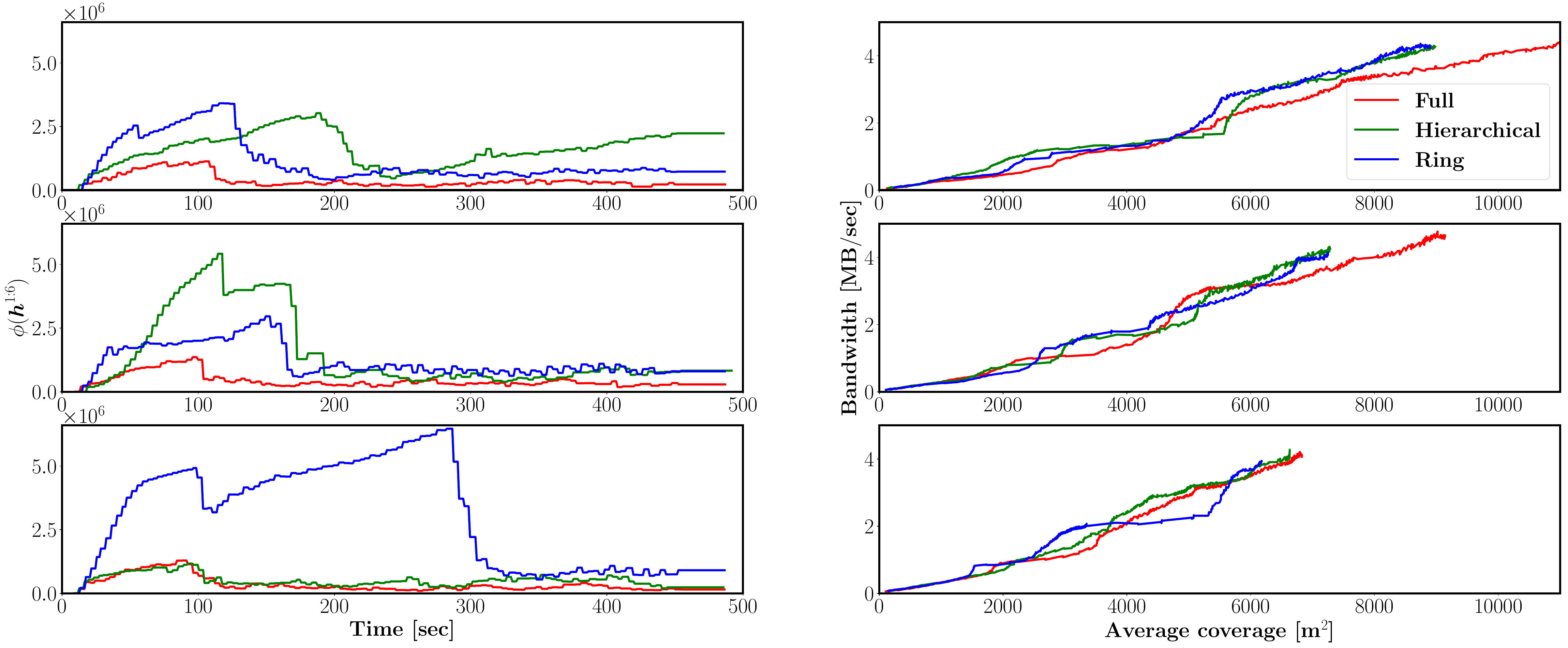}
    \caption{Multi-robot exploration performance metrics in the simulation experiments. The left column shows evolution of the map discrepancy $\phi(\bfh^{1:6})$ across the robot networks $\calG$ in Fig.~\ref{fig:network_confs} over time, and the right column displays bandwidth requirements for the distributed mapping with respect to average coverage. The top, middle, and bottom rows show the results for \textit{Collaborative}, \textit{Egocentric}, and \textit{Frontier} modes of planning, respectively. The average coverage is computed by averaging the area covered at each timestamp over all robots participating in an experiment.}
    \label{fig:map_phi_bw}
\end{figure*}

Additional quantitative metrics specific to multi-robot exploration are reported in Fig.~\ref{fig:map_phi_bw}. The left column of Fig.~\ref{fig:map_phi_bw} shows the aggregate distance $\phi(\bfh^{1:6})$, which represents the total discrepancy across all local maps. Despite robots discovering distinct unexplored regions during exploration, which can increase the difference among the local maps, it can be seen that the map discrepancy tends to decrease overall. The long-term value of the map discrepancy depends on the ratio of exploration rate and information exchange rate. Hence, the \textit{Full} topology yields the closest performance to map consensus due to its relatively faster rate of exchanging the local maps amongst the robots. The right column of Fig.~\ref{fig:map_phi_bw} displays the bandwidth required for communicating the local maps within the robot network. Since the simulation environment uses a centralized network scheme to register the broadcasted local maps, there is no significant variation in terms of bandwidth use across different network topologies and planning modes. Nevertheless, the results show scalability of the semantic octree mapping for multi-robot applications, where an average $97$ $\text{Bytes}/\text{sec}$ of bandwidth is needed for each $1$ $\text{m}^2$ of covered area for voxel dimensions of $0.2 \times 0.2 \times 0.2$ $\text{m}^3$.

The quantitative results of Fig.~\ref{fig:cov_vs_time_dist} and Fig.~\ref{fig:ent_vs_time_dist} demonstrate the effective performance of ROAM, while Fig.~\ref{fig:map_phi_bw} showcases the consensus and communication properties of our method. In the next subsection, we evaluate the distributed multi-robot active mapping in real-world experiments.

\subsection{Real-World Experiments}
\label{subsec:exp_real}

\begin{figure}[t]
    \begin{subfigure}[t]{0.32\linewidth}
    \centering
    \includegraphics[width=0.9\linewidth]{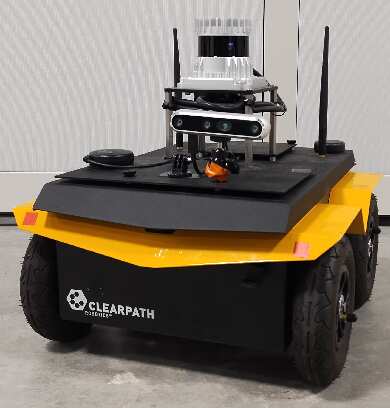}
    \captionsetup{justification=centering}
    \caption{Robot 1}
    \end{subfigure}%
    \hfill%
    \begin{subfigure}[t]{0.32\linewidth}
    \centering
    \includegraphics[width=0.9\linewidth]{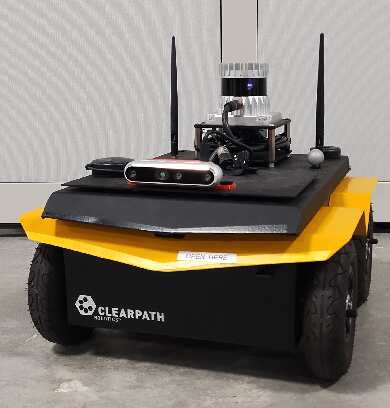}
    \captionsetup{justification=centering}
    \caption{Robot 2}
    \end{subfigure}%
    \hfill%
    \begin{subfigure}[t]{0.32\linewidth}
    \centering
    \includegraphics[width=0.9\linewidth]{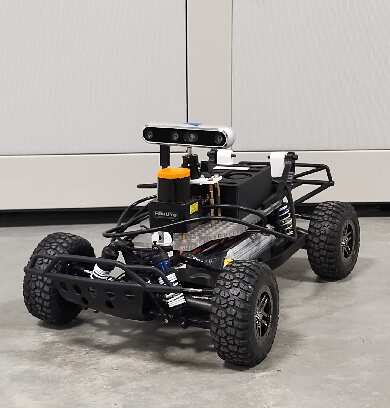}
    \captionsetup{justification=centering}
    \caption{Robot 3}
    \end{subfigure}
    \caption{Ground robot team used in our real-world multi-robot active mapping experiments.}
    \label{fig:real_exp_team}
\end{figure}

\begin{figure}[t]
    \begin{subfigure}[t]{0.40\linewidth}
    \centering
    \includegraphics[height=3.1cm]{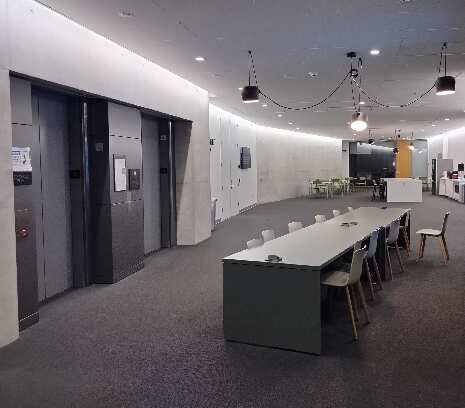}
    \captionsetup{justification=centering}
    \caption{Lobby}
    \end{subfigure}%
    \hspace{0.012\linewidth}%
    \begin{subfigure}[t]{0.29\linewidth}
    \centering
    \includegraphics[height=3.1cm]{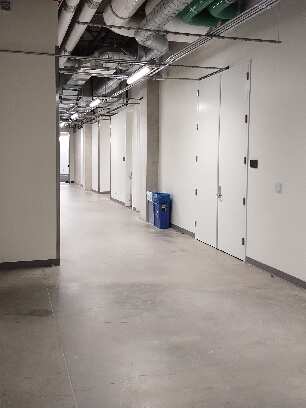}
    \captionsetup{justification=centering}
    \caption{Corridor}
    \end{subfigure}%
    \begin{subfigure}[t]{0.29\linewidth}
    \centering
    \includegraphics[height=3.1cm]{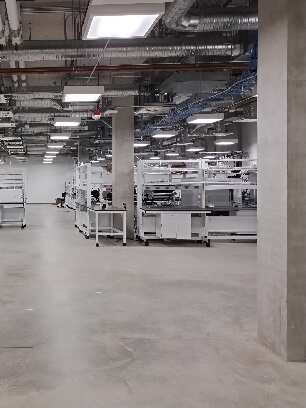}
    \captionsetup{justification=centering}
    \caption{Laboratory}
    \end{subfigure}%
    \caption{Indoor environment used in our real-world multi-robot active mapping experiments.}
    \label{fig:real_exp_env}
\end{figure}

We deployed ROAM on a team of ground robots to achieve autonomous exploration and mapping of an unknown indoor area. The robot team was comprised of two ClearPath Jackal robots (robot 1 and robot 2), and an F1/10 race car robot (robot 3). The Jackals were each equipped with an Ouster OS1-32 3D LiDAR, an Intel RealSense D455 RGB-D camera, and an NVIDIA GTX 1650 GPU. The F1/10 race car was equipped with a Hokuyo UST-10LX 2D LiDAR, an Intel RealSense D455 RGB-D camera, and an NVIDIA Xavier NX computer. Fig.~\ref{fig:real_exp_team} shows the three robots participating in the experiment. We utilize a ResNet18 \cite{resnet} neural network architecture pre-trained on the SUN RGB-D dataset \cite{sun_rgbd} for semantic segmentation. To achieve real-time segmentation, we employed the deep learning inference ROS nodes provided by NVIDIA~\cite{ros_deep_learning}, which are optimized for NVIDIA GPUs via TensorRT acceleration. The semantic segmentation module processes the RGB image stream from the D455 camera, and fuses the segmentation results with the depth image stream, to publish semantic point cloud ROS topics. For localization, the Jackal robots used the direct LiDAR odometry of \cite{dlo}, while the F1/10 race car used iterative closest point (ICP) scan matching \cite{icp}. In order to align the world frames $\calW_i$, $i \in \{1,2,3\}$, we used AprilTag detection \cite{apriltag}, where the estimated $\SE$ transformation between the RGB sensor frame of a robot and the detected tag is defined as the \textit{world-to-sensor} transformation. Communication was handled via a Wi-Fi network and multi-master ROS architecture, such that each robot $i$ runs its own ROS master, and shares its planning ledger $\calL_i$, local plan $\mathfrak{X}_i$, local semantic octree map $\bfm^i$, and estimated pose with respect to $\calW_i$. With all the mapping and planning computations carried out using the on-board robot computers, we obtain an average frame rate of $2.44$Hz for distributed semantic octree mapping, and an average distributed planning iteration time of $0.014$s.


\begin{figure*}[t]
    \begin{subfigure}[t]{0.24\linewidth}
    \centering
    \includegraphics[width=\linewidth]{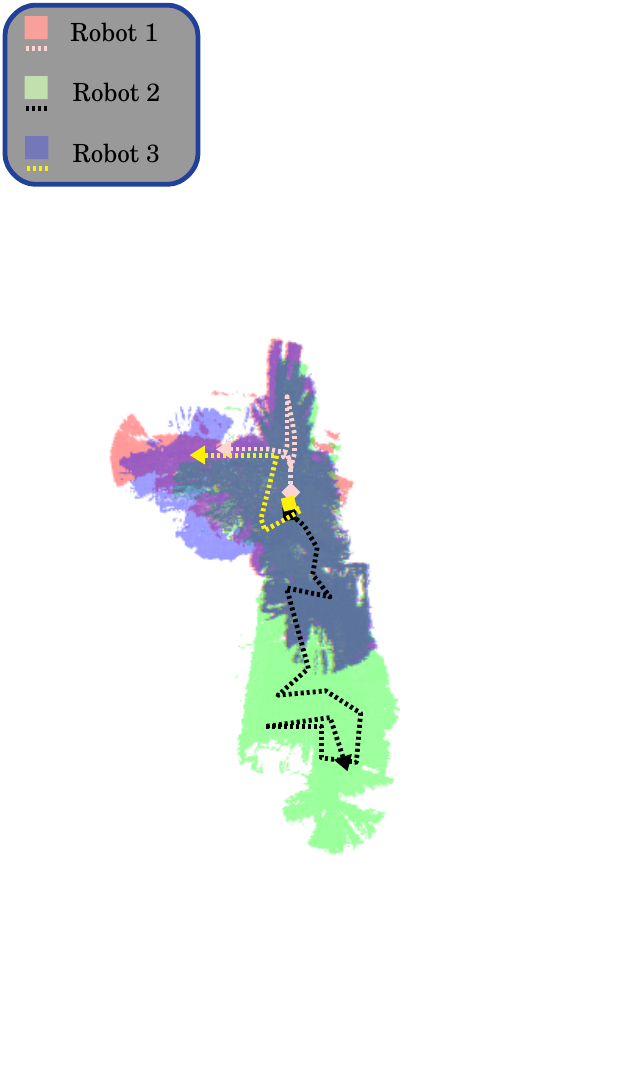}
    \caption{At $t=400$s, the robots effectively choose different sections of the environment to explore.}
    \label{subfig:real_exp_timelapse_a}
    \end{subfigure}%
    \hfill%
    \begin{subfigure}[t]{0.24\linewidth}
    \centering
    \includegraphics[width=\linewidth]{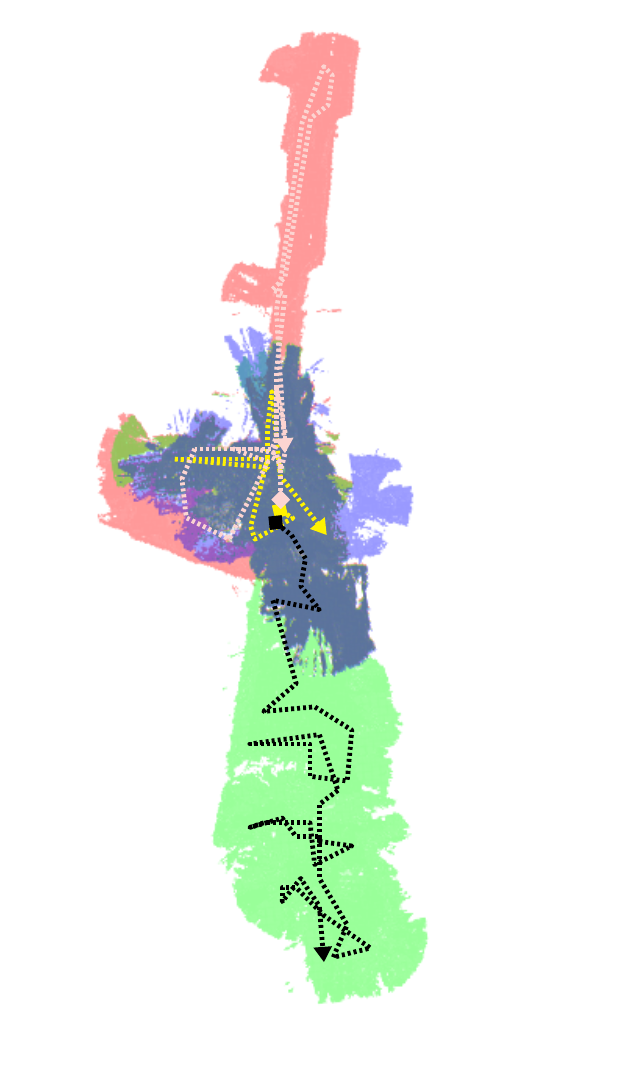}
    \caption{At $t=900$s, the robots continue to explore the environment.}
    \label{subfig:real_exp_timelapse_b}
    \end{subfigure}%
    \hfill%
    \begin{subfigure}[t]{0.24\linewidth}
    \centering
    \includegraphics[width=\linewidth]{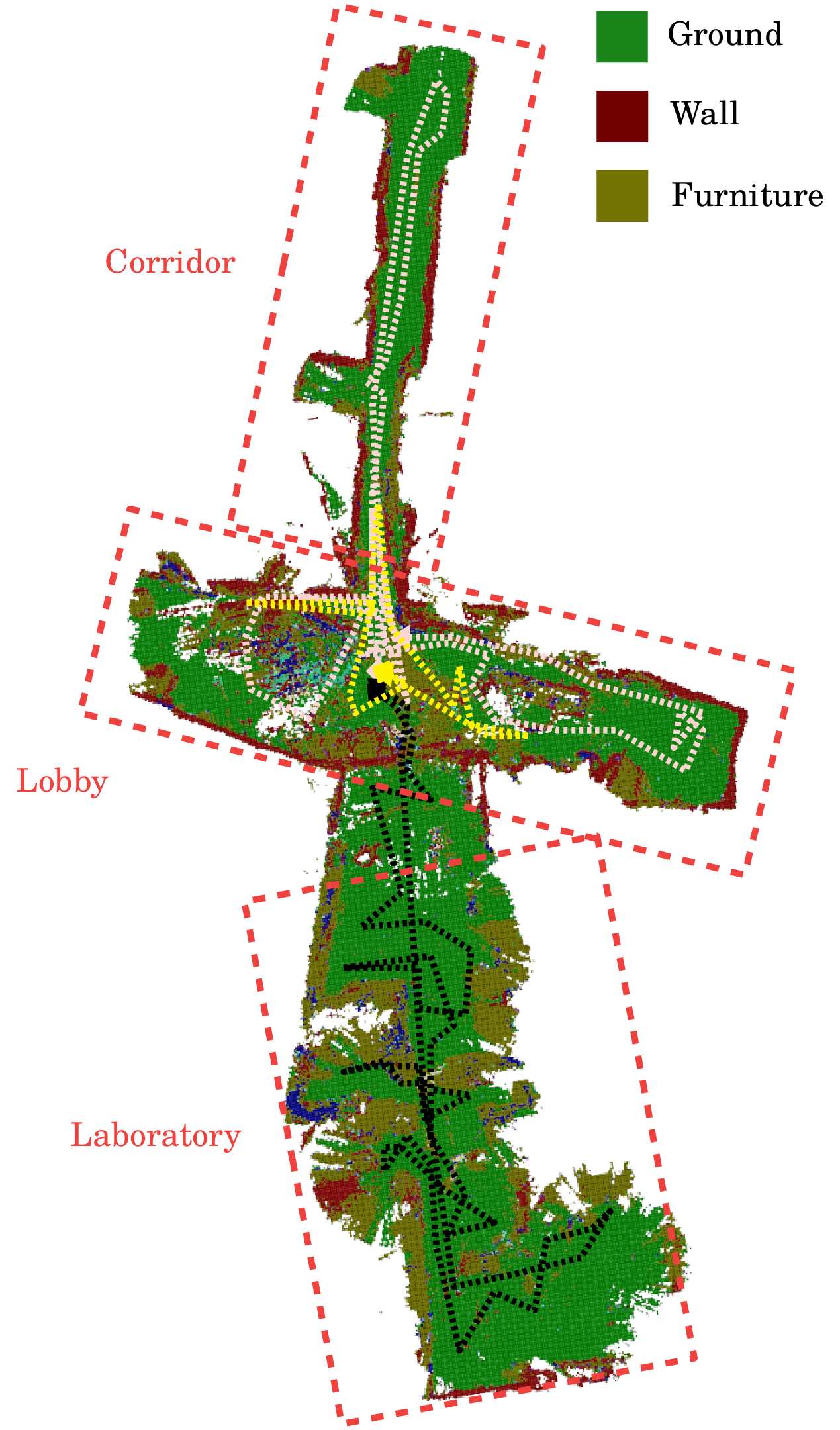}
    \caption{At $t=1450$s, the robots are back to their initial locations and perform a final map exchange to ensure consensus.}
    \label{subfig:real_exp_timelapse_c}
    \end{subfigure}%
    \hfill%
    \begin{subfigure}[t]{0.24\linewidth}
    \centering
    \includegraphics[width=\linewidth]{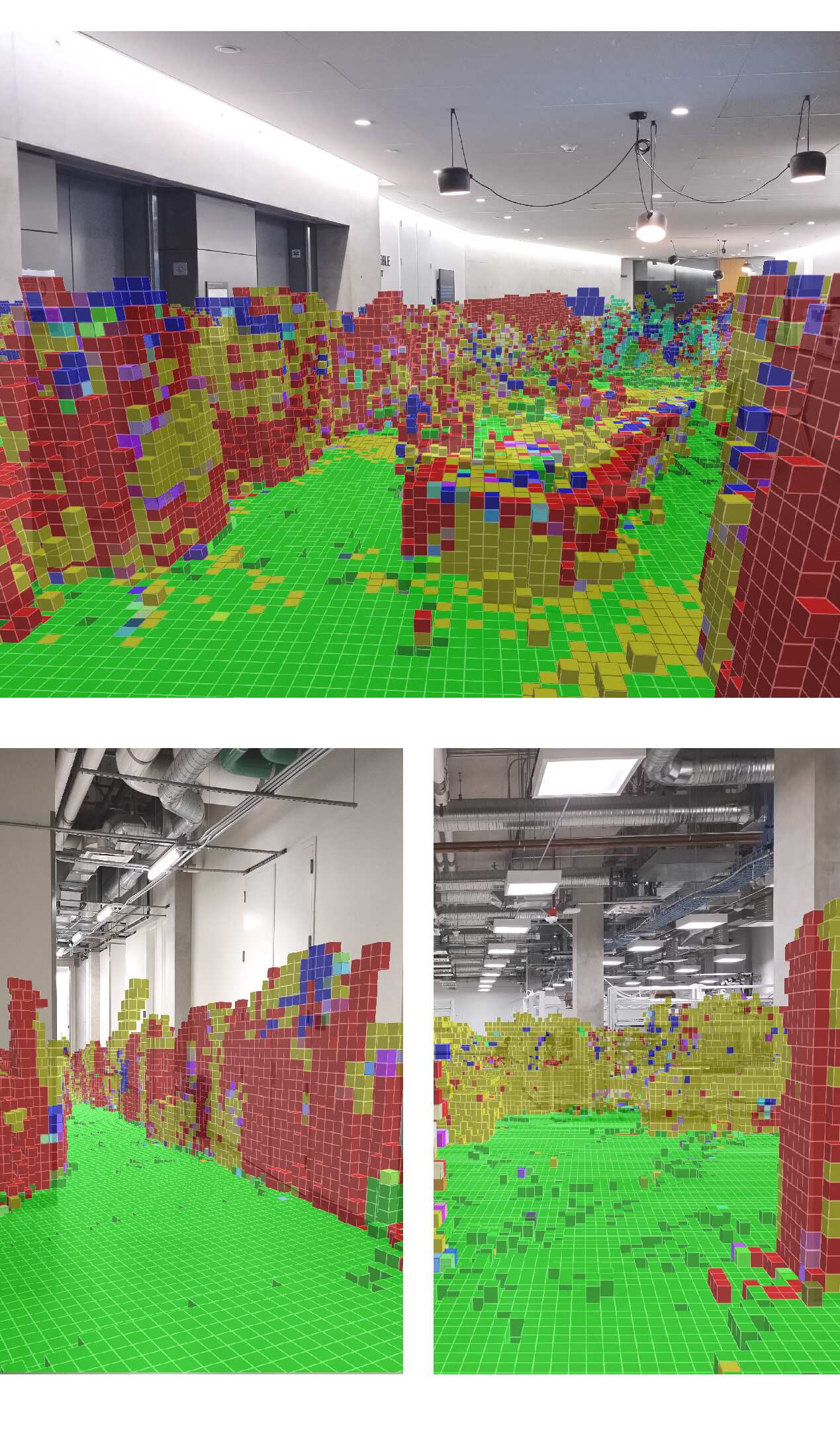}
    \caption{First-person view of the semantic octree map of robot 1.}
    \label{subfig:real_exp_timelapse_d}
    \end{subfigure}%
    \caption{Qualitative results from a real-world multi-robot active mapping experiment. (a)-(b): Snapshots of exploration at $t=400$s and $t=900$s. The explored region and the path of each robot are identified by a distinct color (legend at the top left). (c): The final semantic octree map of robot 1 at $t=1450$s. The three sections of the environment (i.e. lobby, laboratory, and corridor) have been marked on the map. (d) First-person views of the map overlaying the ground-truth environment.}
    \label{fig:real_exp_timelapse}
\end{figure*}

\begin{figure*}[t]
    \includegraphics[width=\linewidth]{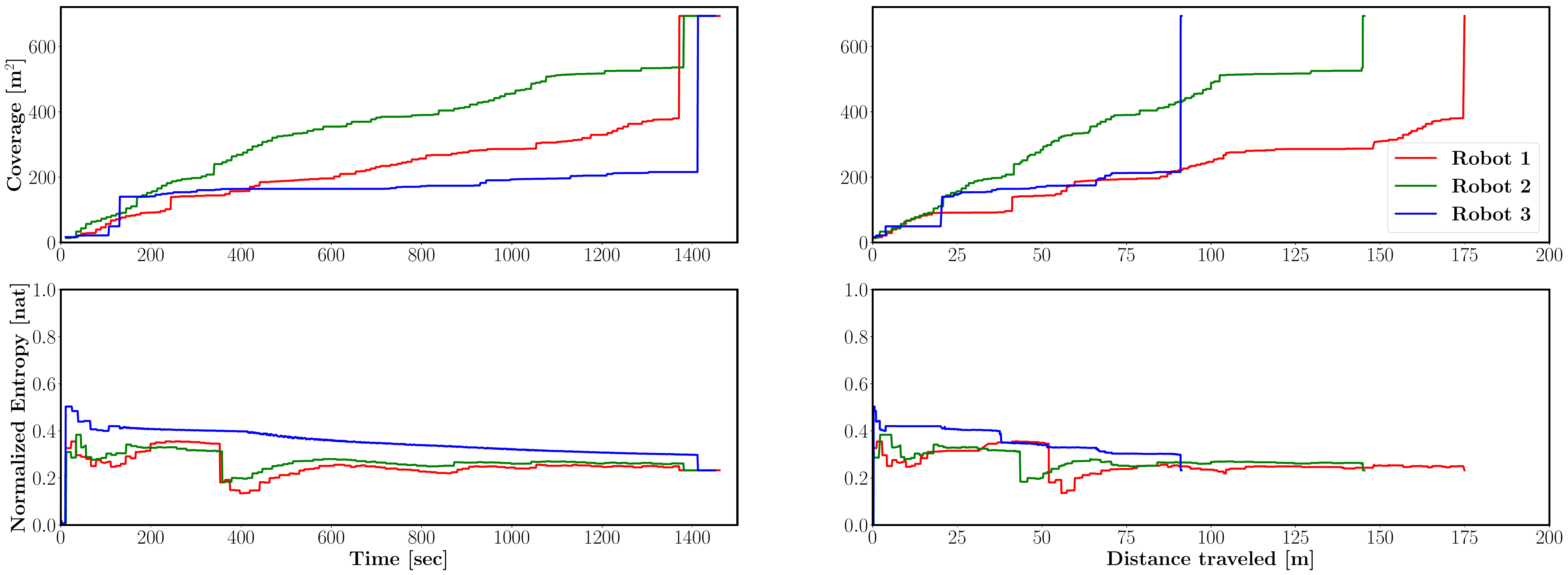}
    \caption{Coverage (top row) and normalized map entropy (bottom row) versus time (left column) and distance traveled (right column) for the real-world multi-robot active mapping experiment. Each color corresponds to one robot in the team.}
    \label{fig:real_quant_res_1}
\end{figure*}

\begin{figure}[t]
    \centering
    \includegraphics[width=\linewidth]{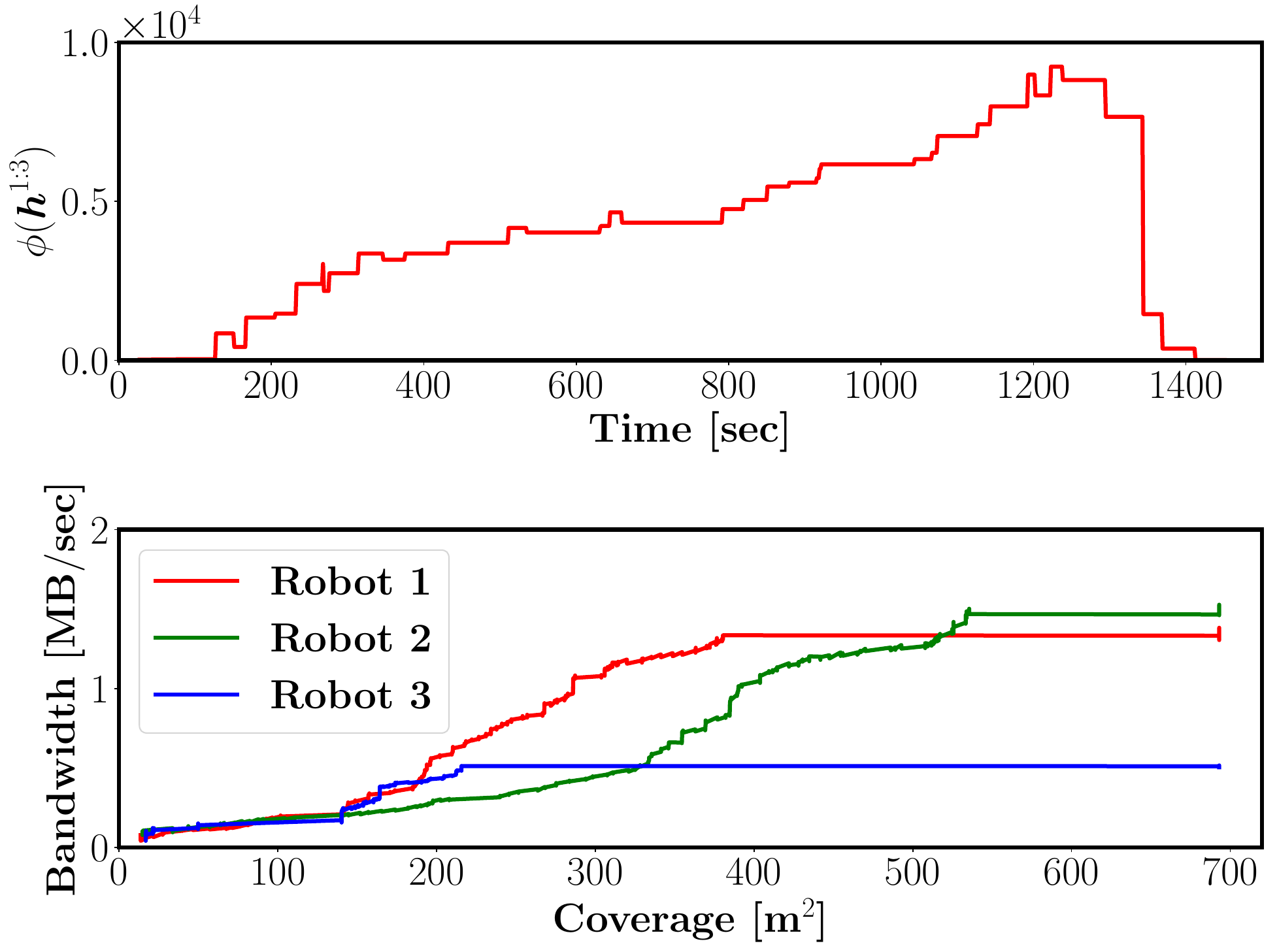}
    \caption{Multi-robot exploration performance metrics for the real-world experiment. Top: Evolution of the map discrepancy $\phi(\bfh^{1:3})$ across the robot network $\calG$ over time. Bottom: Bandwidth requirement of each robot for the distributed mapping with respect to coverage.}
    \label{fig:real_quant_res_2}
\end{figure}


Our experiments took place in a basement area, consisting of a lobby room connected to a corridor and a large laboratory, shown in Fig.~\ref{fig:real_exp_env}. Similar to the simulation experiments, the local semantic octree maps were utilized to analyze the terrain traversability, where in this case the \textit{Ground} object class was selected as the traversable region. Exploration was performed using the \textit{Collaborative} planning parameter set. The communication network topology was \textit{Full}, however, there were intermittent disconnections due to signal attenuation and occlusion by the walls. Fig.~\ref{fig:real_exp_timelapse} shows a time lapse of the real-world multi-robot active mapping experiment. The robots started at nearby positions, all facing the same AprilTag to align their world frames. During the first $400$s of exploration, the robots mostly explored the same areas in their immediate vicinity. Each robot gradually separated from the others after $t=400$s, and focused on a specific part of the environment, as it can be seen in Fig.~\ref{subfig:real_exp_timelapse_a}. In particular, robot 1 explored the corridor area, while robot 2 and robot 3 visited the laboratory and the lobby, respectively. Around $t=500$s, as the robots got farther away from each other, they temporarily lost communication. During the disconnection period, the robots could not plan collaboratively and relied only on their local maps for planning. Nonetheless, this did not deteriorate the exploration performance since the robots were so far away from each other that they would not revisit a region which had already been explored by a peer. Fig.~\ref{subfig:real_exp_timelapse_b} is a good example of such situations. At $t=1200$s, the team was ordered to return to base, where the robots individually planed paths from their current positions back to their initial positions. Communication was automatically re-established as soon as the robots arrived near the starting locations, and distributed mapping was resumed leading to agreement in the semantic octree maps, as shown in Fig.~\ref{subfig:real_exp_timelapse_c}. First-person views of the semantic octree map of robot 1 overlaid on the real-world environment is illustrated in Fig.~\ref{subfig:real_exp_timelapse_d}.

Fig.~\ref{fig:real_quant_res_1} and Fig.~\ref{fig:real_quant_res_2} show the quantitative results of the same real-world experiment. Note the sudden jump in the map coverage plot of Fig.~\ref{fig:real_quant_res_1}, corresponding to reaching map consensus amongst the robots when network connection was re-established. This can be clearly observed in Fig.~\ref{fig:real_quant_res_2} where the map discrepancy $\phi(\bfh^{1:3})$ shrinks to zero as the robots reach consensus. Unlike the simulation experiments, the bandwidth use of the robots can be measured separately by probing the communication packages sent to/from each robot, reported in Fig.~\ref{fig:real_quant_res_2}. The flat lines represent the jump in map coverage due to the eventual connection after a period of intermittent disconnections during the exploration. Overall, the real-world experiments show the practicality of ROAM for autonomous mapping of large unstructured areas using a team of robots and consumer-grade communication infrastructure.

\section{Conclusions}
\label{sec:conclusions}

We developed a distributed Riemannian optimization method that achieves consensus among the variables estimated by different nodes in a communication graph. We used this method to formulate distributed techniques for multi-robot semantic mapping and information-theoretic viewpoint planning. The resulting Riemannian Optimization for Active Mapping (ROAM) enables fully distributed collaborative active mapping of an unknown environment without the need for central estimation and control. Our experiments demonstrated scalability and efficient performance even with sparse communication, and corroborated the theoretical guarantees of ROAM to achieve consensus and convergence to an optimal solution. ROAM offers the possibility to generalize many single-robot non-Euclidean problems to distributed multi-robot applications, for example formulating multi-robot control for $\SE$ robot dynamics~\cite{duong2024port}. An important future research direction involves analyzing the distance to consensus as well as sub-optimality bounds for ROAM in the case of non-convex distance measures and non-concave objective functions. Additionally, further research effort is needed to study faster variants of ROAM using Nesterov accelerated \cite{assran2020convergence} and second-order \cite{pierrot2021first} gradient methods.



\appendices
\section{Proof of Theorem~\ref{thm:dist_opt_cons_opt}}
\label{app:dist_opt_cons_opt}

We organize the proof into three main steps. First, we show that the aggregate distance function $\phi(\cdot)$ is geodesically $L$-smooth. Second, we show that Alg.~\ref{alg:dist_opt} converges to a consensus configuration. Last, optimality properties of Alg.~\ref{alg:dist_opt} are derived.

\begin{step}\label{step:phi_smoothness}
    We begin the proof by showing that the aggregate distance function $\phi(\cdot)$ is geodesically $L$-smooth. Namely, for any pair of joint states $\bfx, \bfy \in \calM^{|\calV|}$ we prove:
    \begin{equation}
        \|g_{\phi}(\bfx) - \calT_{\bfy}^{\bfx}g_{\phi}(\bfy)\|_{\bfx} \leq L d(\bfx, \bfy)\nonumber,
    \end{equation}
    with $g_{\phi}(\bfx)$ as a shorthand notation for $\grad{\phi(\bfx)}$.
    
    For two joint states $\bfx,\bfy \in \calM^{|\calV|}$ we have:
    \begin{equation}
        \grad_{x^i}{\phi(\bfx)} - \calT_{y^i}^{x^i}\grad_{y^i}{\phi(\bfy)} = -2 \sum_{j \in \calV} A_{ij} (v_x^{ij} - \calT_{y^i}^{x^i}v_y^{ij}),\nonumber
    \end{equation}
    where the vectors $v_x^{ij}$ and $v_y^{ij}$ follow the notation in Assumption~\ref{assumption:net_curve_bound}. Using row-stochasticity of $A$ and the fact that $\calM^{|\calV|}$ is the product manifold of $\calM$ leads to:
    \begin{equation} \label{eq:g_phi_ineq}
        \|g_{\phi}(\bfx) - \calT_{\bfy}^{\bfx}g_{\phi}(\bfy)\|^2_{\bfx} \leq 4 \sum_{(i,j) \in \calV^2} A_{ij} \|v_x^{ij} - \calT_{y^i}^{x^i}v_y^{ij}\|^2_{x^i},
    \end{equation}
    Adding and subtracting tangent terms $v_{xy}^i$ and $v_{xy}^j$ as well as utilizing the Cauchy-Schwarz inequality in $T_{x^i}\calM$ results in the following decomposition:
    \begin{equation}\label{eq:v_x/v_y}
    \begin{aligned}
        \|v_x^{ij} - \calT_{y^i}^{x^i}v_y^{ij}\|^2_{x^i} \leq \|v_{xy}^i - \calT_{x^j}^{x^i}v_{xy}^{j}\|_{x^i}^2 + &\|v_{xy}^{ij}\|^2_{x^i} +\\
        2 \|v_{xy}^i - \calT_{x^j}^{x^i}v_{xy}^{j}\|_{x^i} &\|v_{xy}^{ij}\|_{x^i},
    \end{aligned}
    \end{equation}
    such that the vector $v_{xy}^{ij} \in T_{x^i}\calM$ is defined in \eqref{eq:net_vec}. The vector $v_{xy}^{ij}$ contains sum of $4$ vectors corresponding to a geodesic loop. Assumption~\ref{assumption:net_curve_bound}, in addition to applying the triangle inequality, allows finding an upper bound for \eqref{eq:v_x/v_y} that does not involve $v_{x}^{ij}$ and $v_{y}^{ij}$ terms:
    \begin{equation}
    \begin{aligned}
        \|v_x^{ij} - \calT_{y^i}^{x^i}v_y^{ij}\|^2_{x^i} \leq (\rho + 1)^2 (\|v_{xy}^i\|_{x^i} + \|v_{xy}^{j}\|_{x^j})^2.\nonumber
    \end{aligned}
    \end{equation}
    Plugging into \eqref{eq:g_phi_ineq} and summing over $(i, j) \in \calV^2$ yields:
    \begin{equation}
    \begin{aligned}
        \|g_{\phi}(\bfx) - &\calT_{\bfy}^{\bfx}g_{\phi}(\bfy)\|^2_{\bfx} \leq \\
        &8 (\rho + 1)^2 \Big(d^2(\bfx, \bfy) + \sum_{(i,j) \in \calV^2} A_{ij} \|v_{xy}^{i}\|_{x^i} \|v_{xy}^{j}\|_{x^j}\Big),\nonumber
    \end{aligned}
    \end{equation}
    The summation term in the above inequality is upper-bounded by the induced norm of $A$, i.e. its largest eigenvalue. Using the fact that the largest eigenvalue of row-stochastic matrices is $1$, we derive the following:
    \begin{equation}\label{eq:phi_smoothness}
        \frac{\|g_{\phi}(\bfx) - \calT_{\bfy}^{\bfx}g_{\phi}(\bfy)\|_{\bfx}}{d(\bfx, \bfy)} \leq 4 (1 + \rho)  := L.
    \end{equation}
    Therefore, the aggregate distance function $\phi(\bfx)$ is geodesically $L$-smooth, with $L = 4 (1 + \rho)$.
\end{step}

\begin{step}\label{step:consensus}
    This step proves convergence of Alg.~\ref{alg:dist_opt} to a consensus configuration. We use the $L$-smoothness of $\phi(\cdot)$ to find a bound for the values of $\phi(\bfx)$. Consider two points $\bfx$ and $\bfy$ in $\calM^{|\calV|}$, and the geodesic $\bfs(\cdot): [0, 1] \rightarrow \calM^{|\calV|}$ connecting $\bfx$ to $\bfy$, i.e. $\bfs(0) = \bfx$ and $\bfs(1) = \bfy$. Using the fundamental theorem of calculus for line integrals we have:
    \begin{equation}
    \begin{aligned}
        \phi(\bfy) - \phi(\bfx) - &\langle g_{\phi}(\bfx), \Exp^{-1}_{\bfx}{(\bfy)} \rangle_{\bfx} =\nonumber\\
        \int_0^1 &\langle \calT_{\bfs(t)}^{\bfx}g_{\phi}(\bfs(t)) - g_{\phi}(\bfx), \Exp^{-1}_{\bfx}{(\bfy)} \rangle_{\bfx}\,dt.\nonumber
    \end{aligned}
    \end{equation}
    Applying the Cauchy-Schwarz inequality and using the $L$-smoothness of $\phi(\bfx)$ results in:
    \begin{equation}\label{eq:first_order_ineq}
    \begin{aligned}
        \phi(\bfy) - \phi(\bfx) - \langle &g_{\phi}(\bfx), \Exp^{-1}_{\bfx}{(\bfy)} \rangle_{\bfx} \leq\\
        &L d^2(\bfx,\bfy)\int_0^1 t\,dt = \frac{L}{2} d^2(\bfx,\bfy).
    \end{aligned}
    \end{equation}
    The above bound helps to analyze the dynamics of the joint state $\bfx$ over the iterations of Alg.~\ref{alg:dist_opt}. Consider line~\ref{alg_line:consensus} of Alg.~\ref{alg:dist_opt}. Using \eqref{eq:first_order_ineq} leads to:
    \begin{equation}\label{eq:consensus_ineq}
    \begin{aligned}
        \phi(\Tilde{\bfx}^{(k)}) \leq \phi(\bfx^{(k)}) &+ \langle g_{\phi}(\bfx^{(k)}), -\epsilon g_{\phi}(\bfx^{(k)}) \rangle_{\bfx^{(k)}}\\
        &+ \frac{L \epsilon^2}{2} \|g_{\phi}(\bfx^{(k)})\|^2_{\bfx^{(k)}}.
    \end{aligned}
    \end{equation}
    Similarly, for line~\ref{alg_line:optimize_local} of Alg.~\ref{alg:dist_opt} we have:
    \begin{equation}
    \begin{aligned}
        \phi(\bfx^{(k+1)}) \leq \phi(\Tilde{\bfx}^{(k)}) &+ \langle g_{\phi}(\Tilde{\bfx}^{(k)}), \alpha^{(k)} g_{F}(\Tilde{\bfx}^{(k)}) \rangle_{\Tilde{\bfx}^{(k)}}\nonumber\\
        &+ \frac{L {\alpha^{(k)}}^2}{2} \|g_{F}(\Tilde{\bfx}^{(k)})\|^2_{\Tilde{\bfx}^{(k)}},\nonumber
    \end{aligned}
    \end{equation}
    where, analogous to $g_{\phi}(\bfx^{(k)})$, $g_{F}(\Tilde{\bfx}^{(k)})$ is shorthand notation for $\grad{F(\bfx)}|_{\bfx = \Tilde{\bfx}^{(k)}}$. Using the positive definiteness of the Riemannian metric $\langle v,u \rangle_{\bfx}$, we have $2\langle v,u \rangle_{\bfx} \leq \eta \|v\|^2_{\bfx} + \|u\|^2_{\bfx}/{\eta}$ for any $v, u \in T_{\bfx}\calM^{|\calV|}$ and $\eta > 0$. Hence:
    \begin{equation}
    \begin{aligned}
        \phi(\bfx^{(k+1)}) \leq \phi(\Tilde{\bfx}^{(k)}) &+ \frac{\eta}{2} \|g_{\phi}(\Tilde{\bfx}^{(k)})\|^2_{\Tilde{\bfx}^{(k)}}\\
        &+ \frac{{\alpha^{(k)}}^2}{2} (L + \frac{1}{\eta}) \|g_{F}(\Tilde{\bfx}^{(k)})\|_{\Tilde{\bfx}^{(k)}}^2.\nonumber
    \label{eq:inequality_k+1}
    \end{aligned}
    \end{equation}
    By adding and subtracting $\calT_{\bfx^{(k)}}^{\Tilde{\bfx}^{(k)}}g_{\phi}(\bfx^{(k)})$ from $g_{\phi}(\Tilde{\bfx}^{(k)})$, and using the fact that $\|v + u\|^2_{\bfx}/2 \leq \|v\|^2_{\bfx} + \|u\|^2_{\bfx}$, we have:
    \begin{equation}
    \begin{aligned}
        \phi(&\bfx^{(k+1)}) \leq \phi(\Tilde{\bfx}^{(k)}) + \frac{{\alpha^{(k)}}^2}{2} (L + \frac{1}{\eta}) \|g_{F}(\Tilde{\bfx}^{(k)})\|_{\Tilde{\bfx}^{(k)}}^2\\
        &+ \eta \|g_{\phi}(\Tilde{\bfx}^{(k)}) - \calT_{\bfx^{(k)}}^{\Tilde{\bfx}^{(k)}}g_{\phi}(\bfx^{(k)})\|^2_{\Tilde{\bfx}^{(k)}} + \eta \|g_{\phi}(\bfx^{(k)})\|^2_{\bfx^{(k)}}.\nonumber
    \end{aligned}
    \end{equation}
    Using the bound for gradients of the local objective functions $f^i(x^i)$ alongside utilizing the $L$-smoothness of $\phi(\bfx)$ and plugging \eqref{eq:consensus_ineq} into the above inequality, yields:
    \begin{equation}
    \begin{aligned}
        \phi(\bfx^{(k+1)}) &\leq  \phi(\bfx^{(k)}) + \frac{{\alpha^{(k)}}^2}{2 |\calV|} (L + \frac{1}{\eta}) C^2\\
        &+ [\epsilon (\frac{L \epsilon}{2} - 1) + \eta (1 + L^2 \epsilon^2)] \|g_{\phi}(\bfx^{(k)})\|^2_{\bfx^{(k)}}.\nonumber
    \end{aligned}
    \end{equation}
    Choosing $\eta = \frac{\epsilon (2 - L \epsilon)}{4 (1 + L^2 \epsilon^2)}$ we have:
    \begin{equation}\label{eq:per_step_inequality}
    \begin{aligned}
        \frac{\epsilon}{2}(1 - \frac{L \epsilon}{2}) \|g_{\phi}(\bfx^{(k)})&\|^2_{\bfx^{(k)}} \leq \phi(\bfx^{(k)}) - \phi(\bfx^{(k+1)})\\
        &+ {\alpha^{(k)}}^2 (\frac{1 + \epsilon^2 L^2}{\epsilon (1 - \frac{L \epsilon}{2})} + \frac{L}{2}) \frac{C^2}{|\calV|}.
    \end{aligned}
    \end{equation}
    Summing over $k_{\text{max}}$ first iterations of Alg.~\ref{alg:dist_opt} yields:
    \begin{equation}
    \begin{aligned}
        \frac{\epsilon}{2}(1 - \frac{L \epsilon}{2}) \sum_{k=0}^{k_{\text{max}}} \|g_{\phi}(\bfx^{(k)})&\|^2_{\bfx^{(k)}} \leq \phi(\bfx^{(0)}) - \phi(\bfx^{(k_{\text{max}}+1)})\\
        &+ \frac{C^2}{|\calV|} (\frac{1 + \epsilon^2 L^2}{\epsilon (1 - \frac{L \epsilon}{2})} + \frac{L}{2}) \sum_{k=0}^{k_{\text{max}}} {\alpha^{(k)}}^2.\nonumber
    \end{aligned}
    \end{equation}
    Because $\phi(\bfx^{(k_{\text{max}}+1)})$ is always non-negative, we have:
    \begin{equation}
    \begin{aligned}
        \frac{\epsilon}{2}(1 - \frac{L \epsilon}{2}) \sum_{k=0}^{k_{\text{max}}} \|g_{\phi}(\bfx^{(k)})&\|^2_{\bfx^{(k)}} \leq \phi(\bfx^{(0)})\\
        &+ \frac{C^2}{|\calV|} (\frac{1 + \epsilon^2 L^2}{\epsilon (1 - \frac{L \epsilon}{2})} + \frac{L}{2}) \sum_{k=0}^{k_{\text{max}}} {\alpha^{(k)}}^2.\nonumber
    \end{aligned}
    \end{equation}
    As a consequence of the compactness of $\calM$, $\phi(\bfx^{(0)})$ will be bounded for any choice of $\bfx^{(0)}$. Moreover, setting $k_{\text{max}} \rightarrow \infty$, due to the convergence property for the sum of the squared step sizes ${\alpha^{(k)}}^2$, we have:
    \begin{equation}\label{eq:phi_convergence}
    \begin{aligned}
        \frac{\epsilon}{2}&\biggl(1 - \frac{L \epsilon}{2}\biggr) \sum_{k=0}^{\infty} \|g_{\phi}(\bfx^{(k)})\|^2_{\bfx^{(k)}} \leq\\
        &\phi(\bfx^{(0)}) + \frac{C^2}{|\calV|} \biggl(\frac{1 + \epsilon^2 L^2}{\epsilon (1 - \frac{L \epsilon}{2})} + \frac{L}{2}\biggr) \sum_{k=0}^{\infty} {\alpha^{(k)}}^2 < \infty.
    \end{aligned}
    \end{equation}
    Therefore, for $\epsilon \in (0, 2/L)$, $g_{\phi}(\bfx^{(k)})$ shrinks to zero as $k$ goes to infinity. Since $\alpha^{(k)}$ is a decaying sequence and $\|g_{F}(\Tilde{\bfx}^{(k)})\|_{\Tilde{\bfx}^{(k)}} \leq \frac{C}{\sqrt{|\calV|}}$ for all $k \geq 0$, Alg.~\ref{alg:dist_opt} converges to a first-order critical point of $\phi(\bfx)$.

    Let $\bfx^{(\infty)}$ be the joint state that Alg.~\ref{alg:dist_opt} converges to. Also, let $\bfx_c$ be a consensus state, which can be constructed by setting all $x^i$, $i \in \calV$, to an arbitrary state $x_c \in \calM$. Since the squared distance $d^2(\cdot)$ is geodesically convex, and the adjacency matrix $A$ is symmetric and row stochastic, it is possible to show that $\phi(\cdot)$ is also geodesically convex. Thus we have:
    \begin{equation}
        \phi(\bfx_c) \geq \phi(\bfx^{(\infty)}) + \langle g_{\phi}(\bfx^{(\infty)}), \Exp^{-1}_{\bfx^{(\infty)}}{(\bfx_c)} \rangle_{\bfx^{(\infty)}}.\nonumber
    \end{equation}
    Since $\phi(\bfx_c)$ is zero, $\phi(\bfx^{(\infty)})$ is non-negative, and \eqref{eq:phi_convergence} indicates that $\bfx^{(\infty)}$ is a first-order critical point of $\phi(\cdot)$, we arrive at the following for Alg.~\ref{alg:dist_opt}:
    \begin{equation}\label{eq:phi_consensus}
        \phi(\bfx^{(\infty)}) = 0
    \end{equation}
    Note that \eqref{eq:phi_convergence} and \eqref{eq:phi_consensus} hold even for a disconnected graph $\calG$. However, for the global asymptotic consensus, $\calG$ should be a connected graph; otherwise, consensus occurs separately for each connected component of $\calG$.
\end{step}

\begin{step}\label{step:optimality}
    In the last step, we show the optimality properties of Alg.~\ref{alg:dist_opt}. We utilize the Riemannian manifold version of the law of cosines, which can be expressed for a geodesic triangle with side lengths $a$, $b$, and $c$ as follows (Lemma 5, \cite{zhang_sra}):
    \begin{equation}
        a^2 \leq \frac{c\sqrt{|\kappa_{\text{min}}|}}{\tanh{(c\sqrt{|\kappa_{\text{min}}|})}}b^2 + c^2 - 2 b c \cos{(\angle{bc})},
    \label{eq:law_of_cosines}
    \end{equation}
    where $\kappa_{\text{min}}$ is a lower bound on the sectional curvature of the manifold. Now, consider a geodesic triangle specified by ${x^i}^{(k+1)}$, ${\Tilde{x}^i}^{(k)}$, and $x^*$, where $x^* \in \calM$ is an element of the centralized optimal solution of \eqref{eq:dist_opt}, denoted by $\bfx^* \in \calM^{|\calV|}$:
    \begin{equation}
    \begin{aligned}
        d^2({x^i}^{(k+1)}, x^*) \leq d^2&({\Tilde{x}^i}^{(k)}, x^*) + \zeta {\alpha^{(k)}}^2 \|g_{f^i}({\Tilde{x}^i}^{(k)})\|^2_{{\Tilde{x}^i}^{(k)}}\\
        &- 2 \alpha^{(k)} \langle g_{f^i}({\Tilde{x}^i}^{(k)}), \Exp^{-1}_{{\Tilde{x}^i}^{(k)}}{(x^*)} \rangle_{{\Tilde{x}^i}^{(k)}},\nonumber
    \end{aligned}
    \end{equation}
    where $\zeta = \frac{d_{\text{max}}\sqrt{|\kappa_{\text{min}}|}}{\tanh{(d_{\text{max}}\sqrt{|\kappa_{\text{min}}|})}}$ with $d_{\text{max}} \geq \max_{x,y \in \calM} d(x, y)$. Since $\calM$ is compact, $d_{\text{max}}$ is well-defined. Also, the law of cosines still holds using $d_{\text{max}}$ instead of side length $d^2({\Tilde{x}^i}^{(k)}, x^*)$ due to the strict monotonicity of the function $\frac{y}{\tanh{(y)}}$ for $y \in \bbR_{\geq 0}$. As a result of the local objective function concavity and gradient boundedness, we have:
    \begin{equation}
    \begin{aligned}
        2 \alpha^{(k)} (f^i(x^*) - f^i({\Tilde{x}^i}^{(k)})) \leq d^2({\Tilde{x}^i}^{(k)}, x^*) &- d^2({x^i}^{(k+1)}, x^*)\\
        &+ \zeta C^2 {\alpha^{(k)}}^2.\nonumber
    \end{aligned}
    \end{equation}
    Summing over all agents in $\calV$ yields:
    \begin{equation}
    \begin{aligned}
        2 \alpha^{(k)} |\calV| &(F(\bfx^*) - F(\Tilde{\bfx}^{(k)})) \leq\\
        &d^2(\Tilde{\bfx}^{(k)}, \bfx^*) - d^2(\bfx^{(k+1)}, \bfx^*) + \zeta |\calV| C^2 {\alpha^{(k)}}^2.
    \label{eq:F_inequality_alone}
    \end{aligned}
    \end{equation}
    Now, we repeat the same steps for the geodesic triangle specified by $\Tilde{\bfx}^{(k)}$, $\bfx^{(k)}$, and $\bfx^*$:
    \begin{equation}
    \begin{aligned}
        d^2(\Tilde{\bfx}^{(k)}, \bfx^*) \leq d^2(\bfx^{(k)}, \bfx^*) &+ \zeta \epsilon^2 \|g_{\phi}(\bfx^{(k)})\|^2_{\bfx^{(k)}}\\
        &+ 2 \epsilon \langle g_{\phi}(\bfx^{(k)}), \Exp^{-1}_{\bfx^{(k)}}{(\bfx^*)} \rangle_{\bfx^{(k)}}.\nonumber
    \end{aligned}
    \end{equation}
    Using the convexity of $\phi(\cdot)$, the fact that $\phi(\bfx^*)$ is zero by definition, and $\phi(\bfx^{(k)})$ is always non-negative, we have:
    \begin{equation}
        d^2(\Tilde{\bfx}^{(k)}, \bfx^*) \leq d^2(\bfx^{(k)}, \bfx^*) + \zeta \epsilon^2 \|g_{\phi}(\bfx^{(k)})\|^2_{\bfx^{(k)}}.\nonumber
    \end{equation}
    Plugging \eqref{eq:per_step_inequality} into the above inequality yields:
    \begin{equation}
    \begin{aligned}
        d^2(\Tilde{\bfx}^{(k)}, \bfx^*) \leq d^2(\bfx^{(k)}, \bfx^*) &+ \frac{2 \zeta \epsilon}{1 - \frac{L \epsilon}{2}} \Big[\phi(\bfx^{(k)}) - \phi(\bfx^{(k+1)})\\
        &+ \frac{{\alpha^{(k)}}^2 C^2}{|\calV|} [\frac{1 + \epsilon^2 L^2}{\epsilon (1 - \frac{L \epsilon}{2})} + \frac{L}{2}] \Big].\nonumber
    \end{aligned}
    \end{equation}
    Plugging once more into \eqref{eq:F_inequality_alone} leads to:
    \begin{equation}
    \begin{aligned}
        2 \alpha^{(k)} |\calV| (F(\bfx^*) - F(\Tilde{\bfx}^{(k)})) \leq d^2(\bfx^{(k)}, \bfx^*) - d^2(\bfx^{(k+1)}, \bfx^*)&\nonumber\\
        + \zeta |\calV| C^2 {\alpha^{(k)}}^2 + \frac{2 \zeta \epsilon}{1 - \frac{L \epsilon}{2}} \Big[\phi(\bfx^{(k)}) - \phi(\bfx^{(k+1)})&\nonumber\\
        + \frac{{\alpha^{(k)}}^2 C^2}{|\calV|} [\frac{1 + \epsilon^2 L^2}{\epsilon (1 - \frac{L \epsilon}{2})} + \frac{L}{2}] \Big]&.\nonumber
    \end{aligned}
    \end{equation}
    Summing over $k_{\text{max}}$ first iterations of Alg.~\ref{alg:dist_opt} yields:
    \begin{equation}
    \begin{aligned}
        2 |\calV| &\sum_{k=0}^{k_{\text{max}}} \alpha^{(k)} (F(\bfx^*) - F(\Tilde{\bfx}^{(k)})) \leq d^2(\bfx^{(0)}, \bfx^*) + \frac{4 \zeta \phi(\bfx^{(0)})}{2/\epsilon - L}\\
        &+ \zeta C^2 \Big[\frac{2 \epsilon}{|\calV|(1 - \frac{L \epsilon}{2})} [\frac{1 + \epsilon^2 L^2}{\epsilon (1 - \frac{L \epsilon}{2})} + \frac{L}{2}] + |\calV|\Big] \sum_{k=0}^{k_{\text{max}}} {\alpha^{(k)}}^2.\nonumber
    \end{aligned}
    \end{equation}
    It is straightforward to show:
    \begin{equation}
    \begin{aligned}
        \min_{0 \leq k' \leq k_{\text{max}}} \{F(\bfx^*) - F(\Tilde{\bfx}^{(k')})\} &\sum_{k=0}^{k_{\text{max}}} \alpha^{(k)} \leq\\
        &\sum_{k=0}^{k_{\text{max}}} \alpha^{(k)} (F(\bfx^*) - F(\Tilde{\bfx}^{(k)})).\nonumber
    \end{aligned}
    \end{equation}
    Therefore, we have:
    \begin{equation}
    \begin{aligned}
        F(&\bfx^*) \leq \max_{0 \leq k' \leq k_{\text{max}}} \{F(\Tilde{\bfx}^{(k')})\}\\
        &+ \frac{1}{2 |\calV| \sum_{k=0}^{k_{\text{max}}} \alpha^{(k)}} \Big[d^2(\bfx^{(0)}, \bfx^*) + \frac{4 \zeta \phi(\bfx^{(0)})}{2/\epsilon - L}\\
        &+ \zeta C^2 \Big(\frac{2 \epsilon}{|\calV|(1 - \frac{L \epsilon}{2})} (\frac{1 + \epsilon^2 L^2}{\epsilon (1 - \frac{L \epsilon}{2})} + \frac{L}{2}) + |\calV|\Big) \sum_{k=0}^{k_{\text{max}}} {\alpha^{(k)}}^2\Big].\nonumber
    \end{aligned}
    \end{equation}
    Since $d^2(\bfx^{(0)}, \bfx^*)$, $\phi(\bfx^{(0)})$, and $\sum_{k=0}^{\infty} {\alpha^{(k)}}^2$ are bounded and $\sum_{k=0}^{\infty} \alpha^{(k)} = \infty$, the term inside the brackets in the right hand side of the above inequality vanishes as $k_{\text{max}} \rightarrow \infty$. Therefore, $\max_{0 \leq k' \leq k_{\text{max}}} \{F(\Tilde{\bfx}^{(k')})\}$ will be asymptotically lower-bounded by $F(\bfx^*)$. Moreover, since $\|g_{\phi}(\bfx^{(k)})\|_{\bfx^{(k)}}$ and $\alpha^{(k)}$ both shrink to zero as $k \rightarrow \infty$, we derive the same asymptotic lower bound for $F(\bfx^{(k)})$:
    \begin{equation}\label{eq:optimality_prop}
        F(\bfx^*) \leq \lim_{k_{\text{max}} \rightarrow \infty} \; \max_{0 \leq k \leq k_{\text{max}}} \{F(\bfx^{(k)})\}.
    \end{equation}
\end{step}

In summary, the relations in \eqref{eq:phi_convergence} and \eqref{eq:phi_consensus} show convergence to a consensus configuration for step size $\epsilon \in (0, 2/L)$, where $L = 4 (1 + \rho)$. Furthermore, \eqref{eq:optimality_prop} expresses the optimality of Alg.~\ref{alg:dist_opt} with respect to a centralized solution.\qed

\section{Proof of Lemma~\ref{lemma:log_likelihood_obj_decomp}}
\label{app:log_likelihood_obj_decomp}

The conditional independence of observations given the map $\bfm$ allows decomposing the observation model $q^i(\bfz_{1:t}^i|\bfm)$ into a product of single observation models $q^i(\bfz_{\tau}^i|\bfm)$, $\tau \in \crl{1, \ldots, t}$. By applying Bayes rule to the decomposed observation model, the objective function in \eqref{eq:log_likelihood_obj} can be written as:
\begin{equation}
    \sum_{i \in \calV} \sum_{\tau = 1}^{t} \Big( \log{q^i(\bfz_{\tau}^i)} + \bbE_{\bfm \sim p} \big[\log{\frac{q^i(\bfm | \bfz_{\tau}^i)}{p(\bfm)}}\big] \Big).\nonumber
\end{equation}
Using the map independence assumption, the log term inside the expectation can be expressed as the sum of log terms with respect to single cells. The additivity of expected value yields:
\begin{equation}
    \sum_{i \in \calV} \sum_{\tau = 1}^{t} \Big( \log{q^i(\bfz_{\tau}^i)} + \sum_{n = 1}^N \bbE_{\bfm \sim p} \big[\log{\frac{q^i(m_n | \bfz_{\tau}^i)}{p_n(m_n)}}\big] \Big).\nonumber
\end{equation}
Since every term inside the expectation only depends on a single cell $m_n \sim p_n$, $n \in \crl{1, \ldots, N}$, the expectation can thus be simplified to only incorporate one cell instead of the joint distribution $\bfm \sim p$. This leads to the expression in \eqref{eq:log_likelihood_obj_decomp}.\qed



\ifCLASSOPTIONcaptionsoff
  \newpage
\fi


{\small
\bibliographystyle{cls/IEEEtran}
\bibliography{bib/IEEEabrv.bib, bib/IEEEexample.bib}
}

%








\begin{IEEEbiography}
[{\includegraphics[width=1in,height=1.25in,clip,keepaspectratio]{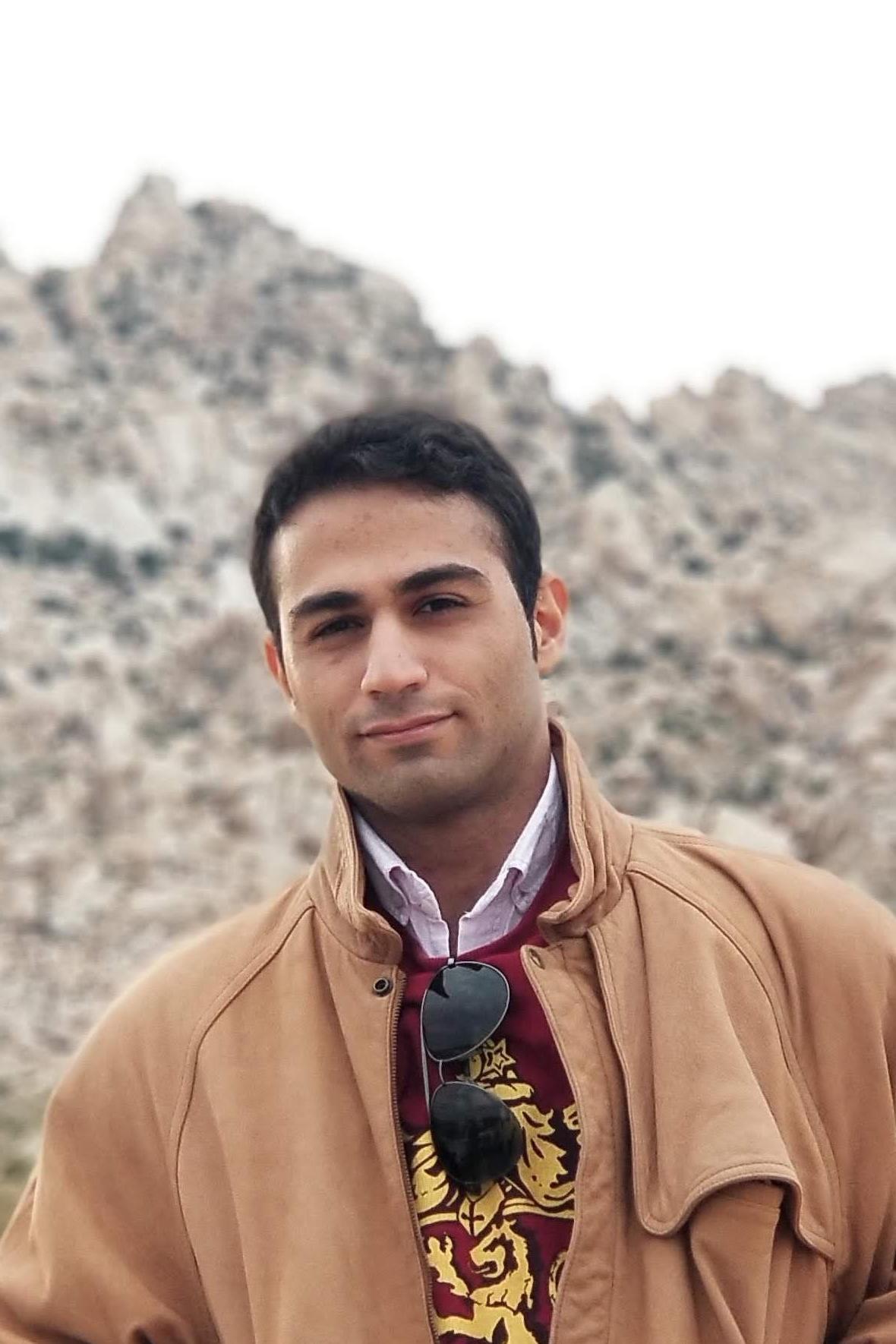}}]{Arash Asgharivaskasi}
(S'22) is a Ph.D. student of Electrical and Computer Engineering at the University of California San Diego, La Jolla, CA. He obtained a B.S. degree in Electrical Engineering from Sharif University of Technology, Tehran, Iran, and an M.S. degree in Electrical and Computer Engineering from the University of California San Diego, La Jolla, CA, USA. His research focuses on active information acquisition using mobile robots with applications to mapping, security, and environmental monitoring.
\end{IEEEbiography}
\begin{IEEEbiography}
[{\includegraphics[width=1in,height=1.25in,clip,keepaspectratio]{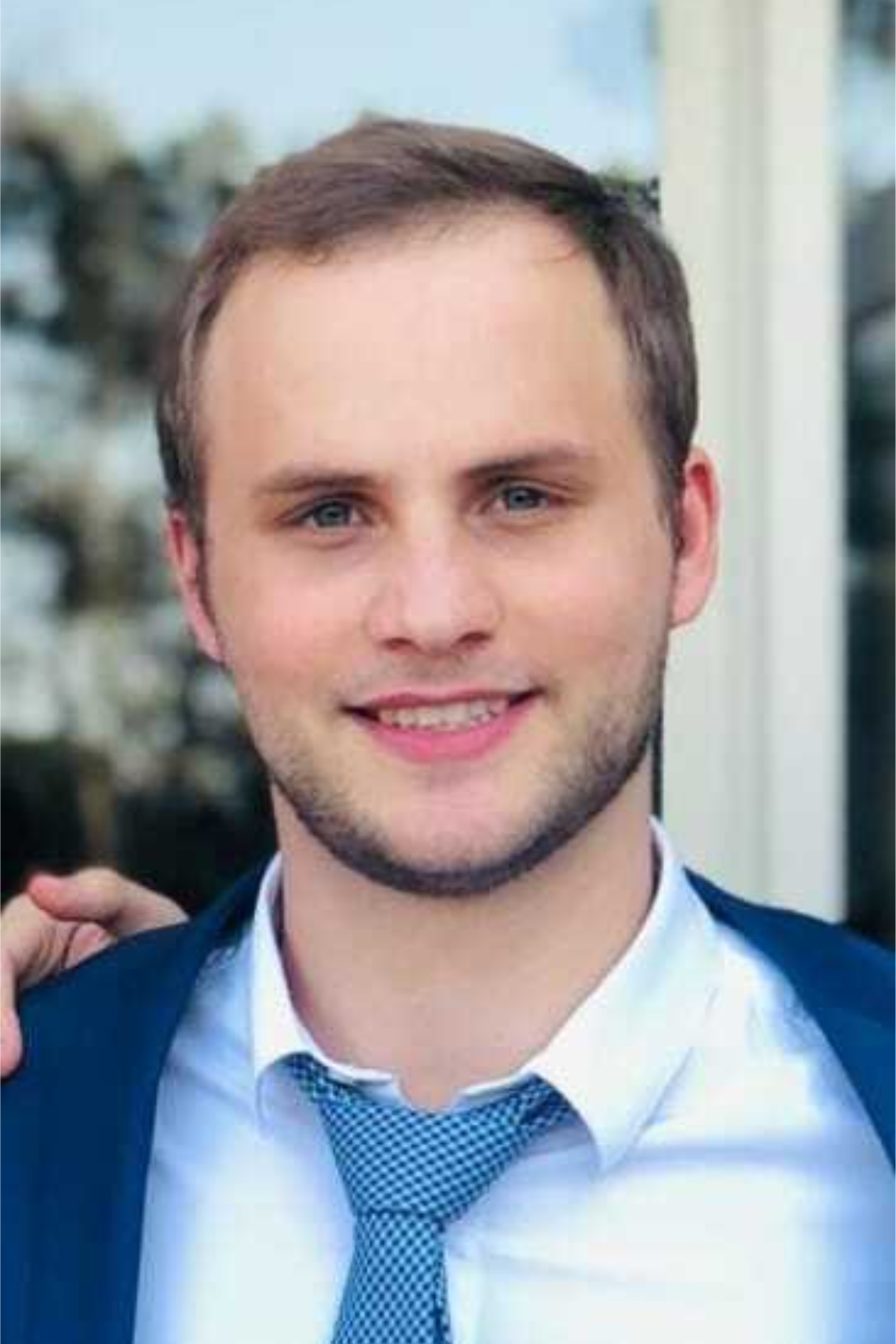}}]{Fritz Girke}
is pursuing an M.S. degree in Electrical Engineering and Information Technology at the Technical University of Munich (TUM), Munich, Germany. They commenced their B.S. degree in Electrical Engineering and Information Technology at TUM in 2017 and completed it in 2021. Their research interests lie in the realm of robotics, specifically focusing on point cloud registration, semantic segmentation, odometry, and simultaneous localization and mapping (SLAM) in mobile robots.
\end{IEEEbiography}
\begin{IEEEbiography}
[{\includegraphics[width=1in,height=1.25in,clip,keepaspectratio]{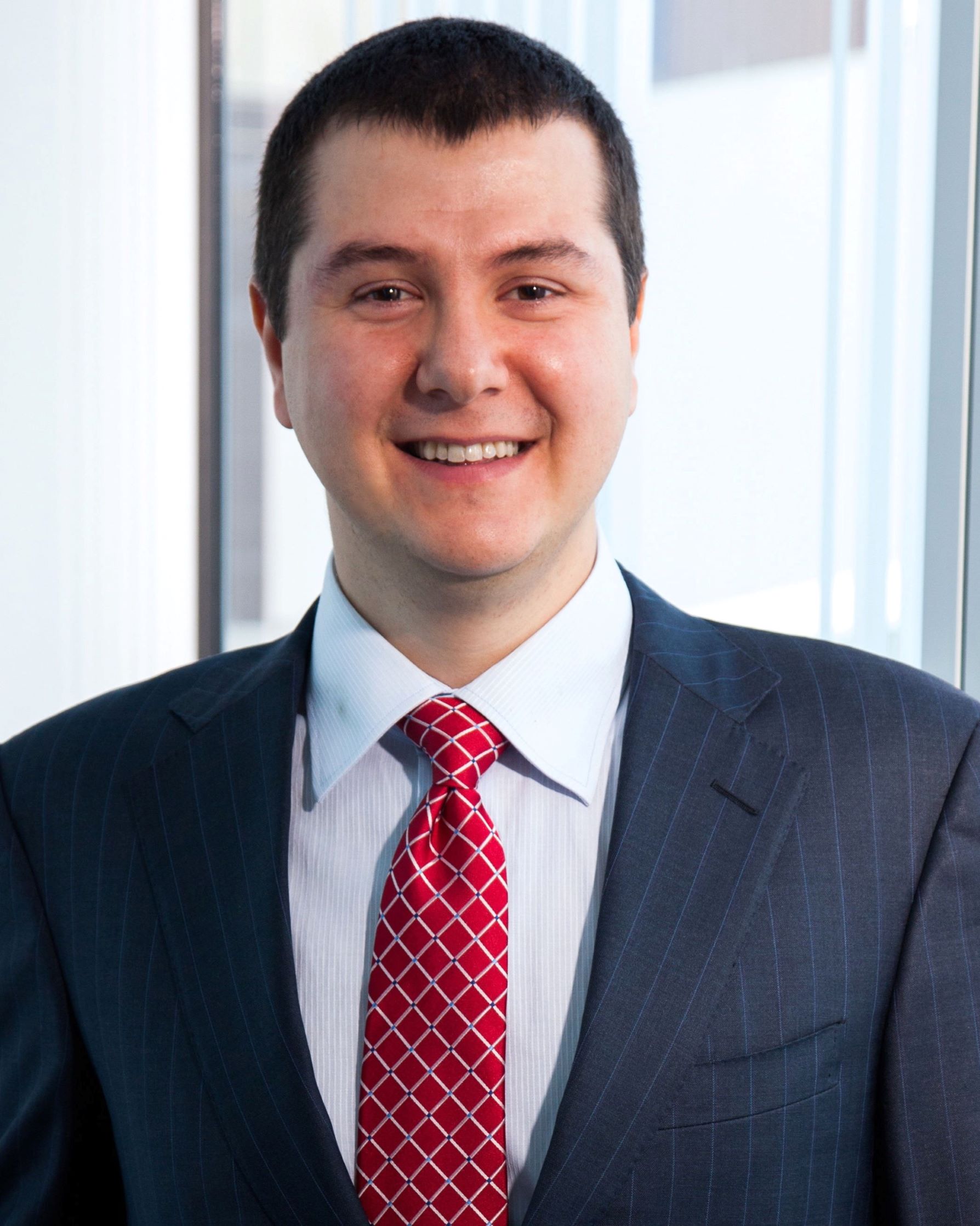}}]{Nikolay Atanasov}
(S'07-M'16-SM'23) is an Assistant Professor of Electrical and Computer Engineering at the University of California San Diego, La Jolla, CA, USA. He obtained a B.S. degree in Electrical Engineering from Trinity College, Hartford, CT, USA in 2008, and M.S. and Ph.D. degrees in Electrical and Systems Engineering from University of Pennsylvania, Philadelphia, PA, USA in 2012 and 2015, respectively. Dr. Atanasov's research focuses on robotics, control theory, and machine learning with emphasis on active perception problems for autonomous mobile robots. He works on probabilistic models that unify geometric and semantic information in simultaneous localization and mapping (SLAM) and on optimal control and reinforcement learning algorithms for minimizing probabilistic model uncertainty. Dr. Atanasov's work has been recognized by the Joseph and Rosaline Wolf award for the best Ph.D. dissertation in Electrical and Systems Engineering at the University of Pennsylvania in 2015, the Best Conference Paper Award at the IEEE International Conference on Robotics and Automation (ICRA) in 2017, the NSF CAREER Award in 2021, and the IEEE RAS Early Academic Career Award in Robotics and Automation in 2023.
\end{IEEEbiography}
\end{document}